\pdfoutput=1
 \documentclass{technical_report}

\PassOptionsToPackage{numbers, compress}{natbib}
\usepackage{natbib}

\usepackage[utf8]{inputenc}
\usepackage[T1]{fontenc}
\usepackage{hyperref}
\usepackage{url}
\usepackage{booktabs}
\usepackage{amsfonts}
\usepackage{amsmath}
\usepackage{amssymb}
\usepackage{amsthm}
\usepackage{etoolbox}
\makeatletter
\renewcommand\part{%
  \if@noskipsec \leavevmode \fi \@afterindentfalse \secdef\@part\@spart}
\renewcommand\@part[2][]{%
  \ifnum \c@secnumdepth >\m@ne
    \refstepcounter{part}%
    \addcontentsline{toc}{part}{\thepart\hspace{1em}#1}%
  \else
    \addcontentsline{toc}{part}{#1}%
  \fi {\parindent \z@ \raggedright \interlinepenalty \@M \normalfont
   \Large \sffamily \bfseries \centering #2\markboth{}{}\par}%
  \nobreak \vskip 1ex \@afterheading} \makeatother
\usepackage{nicefrac}
\usepackage{microtype}
\usepackage{xcolor}
\usepackage{graphicx}
\usepackage{algorithm}
\usepackage{algpseudocode}
\usepackage{multirow}
\usepackage{subcaption}
\usepackage{enumitem}
\usepackage{colortbl}
\usepackage{makecell}
\usepackage[most]{tcolorbox}
\usepackage{textcomp}
\usepackage{upquote}
\usepackage{listings}
\lstdefinestyle{promptlisting}{ basicstyle=\footnotesize\ttfamily, breaklines=true,
breakatwhitespace=true, columns=fullflexible, keepspaces=true, upquote=true, }
\definecolor{deepgreen}{HTML}{007F00}
\definecolor{degradred}{HTML}{C00000}
\definecolor{uwcitepurple}{HTML}{8A2BE2}  %
\newcommand{\inc}[1]{\textsubscript{\textcolor{deepgreen}{\scriptsize +#1}}}
\usepackage{pifont}%
\usepackage{placeins}%
\usepackage{wrapfig}%
\newcommand{\cmark}{\textcolor{deepgreen}{\ding{51}}}
\newcommand{\xmark}{\textcolor{red!70!black}{\ding{55}}}

\hypersetup{ colorlinks=true, linkcolor=red, citecolor=uwcitepurple, filecolor=magenta,
urlcolor=uwcitepurple, linktocpage}
\usepackage{comment}
\usepackage[toc,page,header]{appendix}
\usepackage{minitoc}
\usepackage[frozencache,cachedir=mintedcache]{minted}
\setminted[python]{frame=single,framesep=8pt,fontsize=\footnotesize}
\renewcommand \thepart{}

\makeatletter\renewcommand*{\@dotsep}{10000}\makeatother \doparttoc \faketableofcontents

\newcommand{\spade}{\texttt{SPADE}}
\newcommand{\spadebrand}{\textbf{\textcolor{huskypurple}{SPADE}}\,\raisebox{0.08em}{\scalebox{0.9}{\textcolor{huskypurple}{$\spadesuit$}}}}

\newcommand{\ED}{\texttt{Environment Designer}}
\newcommand{\RA}{\texttt{Reasoning Agent}}
\newcommand{\piD}{\pi_{D}}
\newcommand{\piA}{\pi_{A}}
\newtheorem{theorem}{Theorem}[section] 
\newtheorem{lemma}[theorem]{Lemma} 
\theoremstyle{definition} 
\newtheorem{assumption}[theorem]{Assumption} \theoremstyle{plain}

\newcommand{\backtotoc}{\texorpdfstring{\hfill{\normalfont\footnotesize\sffamily\hyperref[app:toc]{[back to contents]}}}{}}

\newtcolorbox{findingbox}{%
  colback=huskybg, colframe=huskypurple, arc=2mm, boxrule=0.8pt,
  left=6pt,right=6pt,top=4pt,bottom=4pt%
}
\newtcolorbox{promptbox}[1][]{%
  colback=huskybg, colframe=huskypurple, arc=1mm, boxrule=0.5pt,
  left=6pt,right=6pt,top=4pt,bottom=4pt, breakable,
  title={#1}%
} \tcbset{
  rollouttrajectory/.style={%
    colback=huskybg, colframe=huskypurple, arc=0mm, boxrule=0.3pt,
    left=6pt,right=6pt,top=4pt,bottom=4pt, enhanced, breakable,
    fontupper=\footnotesize\ttfamily%
  } }
\newtcolorbox{envcard}[2][]{%
  enhanced, breakable, colback=white, colframe=huskypurple, arc=1.5mm, boxrule=0.7pt,
  left=6pt, right=6pt, top=5pt, bottom=5pt, title={\sffamily\bfseries #2},
  colbacktitle=huskypurple, coltitle=white, fonttitle=\small, #1}
\newcommand{\seedline}[1]{{\footnotesize\itshape\textcolor{black!60}{Seed document (web-scraped): ``#1{}...''}}\par\smallskip}
\newcommand{\obsblock}[1]{\begin{tcolorbox}[colback=huskybg, colframe=huskypurple!40,
  arc=1mm, boxrule=0.4pt, left=4pt, right=4pt, top=3pt, bottom=3pt] {\footnotesize\ttfamily
  #1{}\,...}\end{tcolorbox}}
\newtcolorbox{contrastbox}[1]{%
  enhanced, breakable, colback=white, colframe=huskypurple, arc=1.5mm, boxrule=0.7pt,
  left=6pt, right=6pt, top=5pt, bottom=5pt, title={\sffamily\bfseries #1},
  colbacktitle=huskypurple, coltitle=white, fonttitle=\small}

\newtcolorbox{promptcard}[1]{%
  enhanced jigsaw, breakable, colback=white, colframe=huskypurple, arc=1.5mm, boxrule=0.7pt,
  left=6pt, right=6pt, top=4pt, bottom=4pt, title={\sffamily\bfseries #1},
  colbacktitle=huskypurple, coltitle=white, fonttitle=\small}

\usepackage{tikz}
\usetikzlibrary{fit,calc}
\newcommand*{\tikzmk}[1]{\tikz[remember picture,overlay,] \node (#1) {};\ignorespaces}
\newcommand{\boxit}[1]{\tikz[remember picture,overlay]{\node[xshift=-7.95em,yshift=-1em,fill=#1,opacity=.25,fit={(A)($(B)+(1.0665\linewidth,1\baselineskip)$)}] {};}\ignorespaces}
\newcommand{\boxittwo}[1]{\tikz[remember picture,overlay]{\node[xshift=-4.15em,yshift=-1em,fill=#1,opacity=.25,fit={(A)($(B)+(0.98275\linewidth,1\baselineskip)$)}] {};}\ignorespaces}
\newcommand{\boxithree}[1]{\tikz[remember picture,overlay]{\node[xshift=-4.15em,yshift=-1em,fill=#1,opacity=.25,fit={(A)($(B)+(0.51675\linewidth,0.55\baselineskip)$)}] {};}\ignorespaces}
\colorlet{mypink}{red!30} \colorlet{myblue}{orange!30} \colorlet{mypurple}{green!10}

\title{\textcolor{huskypurple}{SPADE}\,\raisebox{0.12em}{\scalebox{0.95}{\textcolor{huskypurple}{$\spadesuit$}}}: \textcolor{huskypurple}{S}elf-\textcolor{huskypurple}{P}lay in \textcolor{huskypurple}{A}\textcolor{huskypurple}{d}aptive Synthetic Executable \textcolor{huskypurple}{E}nvironments}

\affiliation[1]{University of Washington}
\affiliation[2]{Stanford University}
\affiliation[3]{Northeastern University}
\affiliation[4]{Carnegie Mellon University}
\affiliation[5]{Massachusetts Institute of Technology}
\affiliation[6]{National University of Singapore}
\affiliation[7]{Seoul National University}
\affiliation[8]{Stevens Institute of Technology}
\affiliation[9]{University of Chicago}

\contribution[*]{Equal contribution}
\contribution[\dagger]{Joint last author}

\author[1,2,*]{Bo Liu}
\author[3,*]{Simon Yu}
\author[4]{Yiding Jiang}
\author[5]{Ao Qu}
\author{Andrew Zhao}
\author[6]{Zichen Liu}
\author[7]{Junsu Kim}
\author[6]{Zijian Zhou}
\author[4]{Seungone Kim}
\author{Tongzheng Ren}
\author[1]{Mickel Liu}
\author[8]{Hanfei Yu}
\author[9]{Zhaorun Chen}
\author[3]{Weiyan Shi}
\author[5]{Paul Pu Liang}
\author[1]{Luke Zettlemoyer}
\author[2,\dagger]{Yejin Choi}
\author[1,\dagger]{Natasha Jaques}

\correspondence{\email{benjaminliu.eecs@gmail.com}, \email{yejinc@stanford.edu}, \email{nj@cs.washington.edu}}

\metadata[Code]{\url{https://github.com/spade-rl/spade}} \metadata[Project]{\url{https://spade-rl.github.io}}

\begin{document}

\maketitle

\begin{abstract}
\vspace{-0.5em}
Continuous self-improvement requires an ever-expanding pool of self-generated, diverse, adaptive goals. For language agents, existing training environment pools (hand-curated, statically synthesized, or frozen-verifier) keep the goal distribution fixed as the learner scales. We introduce \spade{} (Self-Play in Adaptive Synthetic Executable Environments), a self-play RL framework in which a single LLM plays two roles: an \ED{} that writes complete, long-horizon training environments as executable code with an OpenAI Gym-style \texttt{reset()}/\texttt{step()} interface, and a \RA{} that learns to act in them. Each is a stateful, multi-turn environment (state transitions, reward functions, and verification code), so one interface spans reasoning problems and multi-step agentic tool use. The \RA's regret is estimated using  the gap between its reward with and without privileged hints; in optimizing this regret signal the \ED{} learns to target environments at the edge of the agent's capabilities while keeping them feasible. Through extensive experimentation, we find several components critical to success: grounding the \ED{} on documents sampled from a large pretraining corpus, and giving it an accumulated environment memory. Scaling to 30B-parameter models, \spade{} improves over the strongest fixed-environment baseline by $+5.3$ on average across eight held-out math, science, code, and reasoning benchmarks, and lifts the tool-use setting by $+5.7$ on BFCL~v4 multi-turn and $+13.9$ on ACEBench-Agent; on the games setting, the margin over the strongest baseline grows with model scale. By making environment design itself a learnable component, \spade{} takes a concrete step toward open-ended self-improvement.
\vspace{-0.475em}
\end{abstract}

\vspace{-0.78em}
\noindent\begin{minipage}{\textwidth}
\centering
\begin{minipage}[c]{0.33\textwidth}\centering
\includegraphics[width=\linewidth]{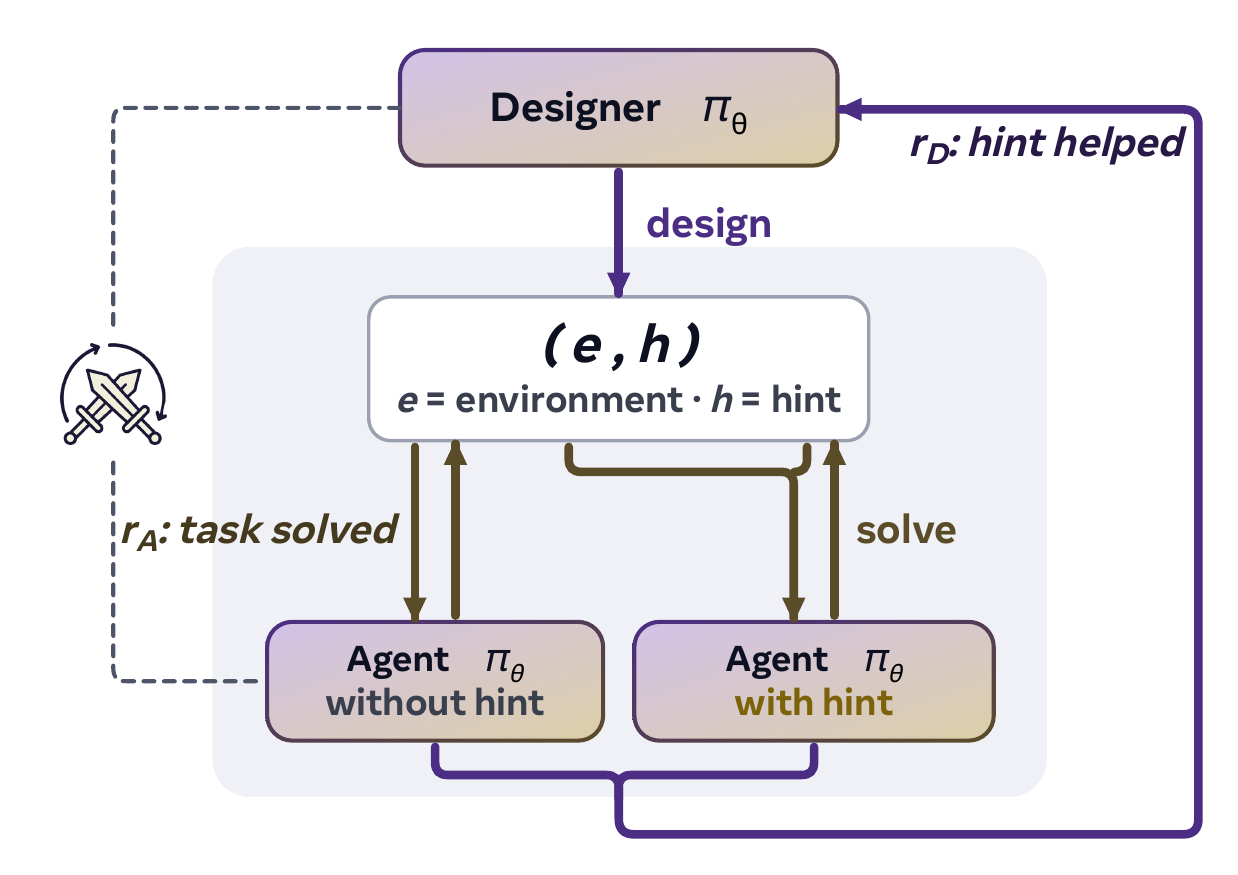}
\end{minipage}\hfill
\begin{minipage}[c]{0.33\textwidth}\centering
\includegraphics[width=\linewidth]{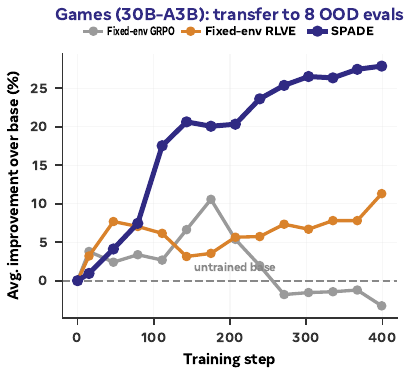}
\end{minipage}\hfill
\begin{minipage}[c]{0.33\textwidth}\centering
\includegraphics[width=\linewidth]{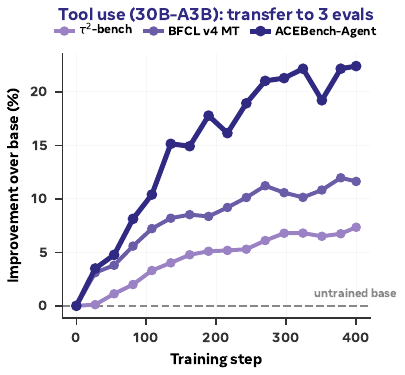}
\end{minipage}
\vspace{-0.6em}
\captionof{figure}{
\textbf{\spade{} designs and solves its own training environments in both settings.} \textbf{Left:} a single LLM $\pi_\theta$ plays both roles, an \ED{} that writes an executable environment $e$ with a privileged hint $h$, and a \RA{} that solves $e$ with and without $h$; the return gap rewards the \ED{} (hint-based regret), task completion rewards the \RA{}, and both update the same weights. \textbf{Middle:} average \emph{relative} improvement over the untrained base across the eight games-setting evals ($(\text{score}-\text{base})/\text{base}$), versus Fixed-env RLVE (orange) and Fixed-env GRPO (gray). \textbf{Right:} per-benchmark relative improvement of \spade{}-30B-A3B on the three tool-use evals ($\tau^2$-bench, BFCL~v4 multi-turn, and ACEBench-Agent; Table~\ref{tab:tooluse}).
}
\label{fig:teaser}
\end{minipage}

\section{Introduction}
\label{sec:intro}

Agentic AI has become broadly capable: language models now reason over long horizons, use tools, search the web, and operate computers~\citep{openai2024o1,guo2025deepseek,wang2025ragen,zhang2025agent}. These gains increasingly come not from pretraining alone but from learning through experience, where a model improves from its own trajectories of interacting with an environment and the rewards they return~\citep{silver2025welcome}. As high-quality human text is a finite resource~\citep{villalobos2024will}, the bottleneck for further agentic progress is increasingly the supply of training environments: interactive tasks with verifiable rewards. Building them has become a central industry investment, with major labs weighing environment budgets exceeding \$1 billion a year and a new class of startups raising nine figures to supply them~\citep{zeff2025environments,primeintellect2026series}. Yet whether hand-built or synthesized, these environments form a fixed pool that does not adapt as the learner improves, so an agent stops improving once it exhausts them.

Several lines of work try to keep agents improving, each with its own limitation. \emph{Harness engineering} improves how an agent acts within a given environment at inference time, through tool orchestration and self-correction~\citep{yao2023react,shinn2023reflexion}, but it changes no weights, so its gains are tied to each hand-built harness, rather than improving general capabilities the model carries to new tasks. \emph{Human-curated scaling}~\citep{zhang2025agentrl,guertler2025textarena, liu2025gem} builds diverse environments to train across, but scales only as fast as people can write them. \emph{Synthetic generation}~\citep{wang2026agent,tu2026scaleenv,zhu2026termigen,gandhi2026endless,zeng2025rlve} produces environments programmatically, but the generators are fixed, so the environment space does not grow with the agent and the model soon exhausts it~\citep{song2024mind}. \emph{Self-play}~\citep{zhao2025absolute,huang2025r,liu2025spiral} lets a model improve through dual roles, but ungrounded self-play is bounded by information symmetry, unable to pose challenges beyond its own knowledge and prone to amplifying its errors~\citep{chae2025towards}; corpus grounding mitigates this~\citep{liu2025spice}, yet most methods still generate \emph{tasks} (a problem with a sparse terminal reward) rather than complete \emph{multi-turn environments} (state transitions, reward functions, and verification code). None produces a self-improving system whose environments keep growing in complexity as the agent improves.

We introduce \spade{} (Self-Play in Adaptive Synthetic Executable Environments), a framework where a single LLM plays two roles: an \ED{} that produces complete training environments as executable Python code, and a \RA{} that learns from them. Each environment implements a Gym-style interface (the standard \texttt{reset()}/\texttt{step()} API used in reinforcement learning)~\citep{brockman2016openai}, representing a full Markov decision process (MDP) with state transitions and reward functions. This \emph{code-as-environment} representation unifies single-turn settings (one step to terminal reward) and multi-turn agentic tasks (sequential interaction with transition dynamics) under a single interface. Because any computable MDP can be written as a program, the \ED{} can express any such environment in code, rather than only those a hand-designed parameterization allows. Unlike prior works where the environment generator is frozen, \spade{}'s \ED{} is itself trained via RL with a \emph{hint-based regret} signal that targets solvable environments precisely at the frontier of the \RA{}'s capability, creating co-evolution where the environment distribution shifts as the \RA{} improves (Figure~\ref{fig:overview}).

Overall, our work makes the following contributions:
\begin{enumerate}
    \item We introduce a \emph{general framework for co-evolving environments synthesis and agentic capability through self-play}, where the same LLM both generates executable environments as code and learns to solve them. By representing environments as Python programs with a Gym-style interface, the framework unifies single-turn reasoning and multi-turn agentic tasks, and turns environment design into a learnable, RL-trained component of post-training, enabling continual open-ended self-improvement.
    \item We introduce a \emph{hint-based regret reward} for the \ED{}. The \ED{} is rewarded by the gap between \RA{} performance with and without privileged hints, a lightweight estimate of minimax regret~\citep{dennis2020emergent} that targets environments at the frontier of the \RA{}'s capability. Unlike a pure adversary, which is free to make environments unsolvable, or a cooperative designer, which can inflate \RA{} reward without teaching it anything, hint-based regret yields constrained competitive dynamics that reward environments which are both solvable and at the learning frontier, with an equilibrium analysis in Appendix~\ref{app:theory}.
    \item We design a \emph{corpus-grounded and memory-augmented design pipeline} spanning cognitive-skill games and tool-use tasks. Where prior environment-generation work targets a single domain (math, code, or tool use), \spade{} demonstrates improvements in both settings, with environments grounded in pretraining corpus knowledge and accumulated experience.
    \item We validate \emph{self-play working at 30B+ scale} with practical recipe provided, demonstrating that the approach works beyond the small-model with a complete training recipe including environment validation, reward hacking avoidance, and curriculum design.
\end{enumerate}

On the games setting, \spade{} outperforms fixed-environment baselines on all three backbones, by $+5.3$ points on average and up to $+7.5$ on individual benchmarks over the strongest fixed-environment baseline at 30B-A3B. In the tool-use setting, \spade{} lifts BFCL~v4 multi-turn~\citep{patil2025berkeley} by $+10.3$ at 4B and $+5.7$ at 30B-A3B, and ACEBench-Agent~\citep{chen2025acebench} by $+13.9$ at 30B-A3B. Ablations show the advantage of hint-based regret over an exponential-moving-average (EMA) learning-potential signal~\citep{kanitscheider2021multitaskcurriculumlearningcomplex, zhang2023omni},
and of the full adaptive, corpus-grounded, memory-augmented configuration over the partial and non-adaptive controls. Qualitative analysis reveals emergent curricula: \spade{} progresses from simple single-skill tasks to complex multi-constraint environments requiring long-horizon interaction.

In \spade{}, the \ED{} writes a new, self-contained environment and the \RA{} learns by acting in it. Making design a learnable role lets the two co-evolve: as the agent improves, so do the environments it trains on (Figure~\ref{fig:overview}). A single model thus improves not only how it reasons and acts, but the worlds it builds to learn in, a concrete step toward open-ended, continual self-improvement.

\begin{figure}[t]
\centering
\includegraphics[width=\linewidth]{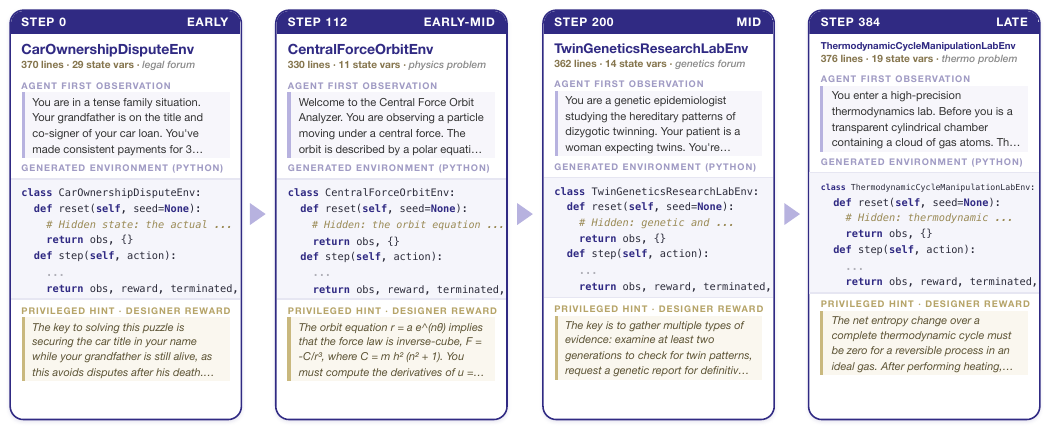}
\caption{\textbf{\spade{} generates an adaptive, multi-turn curriculum.} Four environments the \ED{} produces over one 30B-A3B run, from step~0 (early) to step~384 (late); each card shows the agent's first observation, the generated Python environment, and the designer-written hint. Every environment is a complete MDP with a \texttt{reset()}/\texttt{step()} interface, and the tasks shift toward state-gated, multi-turn interaction as the \RA{} improves. Unlike a fixed human-curated pool or a frozen synthetic generator, this curriculum keeps moving with the learner.}
\label{fig:overview}
\end{figure}

\section{Related Work}
\label{sec:related}

For an extended background and comparison to the literature see Appendix~\ref{app:related}, summarized here:

\noindent\textbf{Self-Play for LLMs\,} Self-play has driven capability growth from TD-Gammon~\citep{tesauro1995temporal} through AlphaZero~\citep{silver2017mastering,silver2018general} to modern game AI~\citep{vinyals2019grandmaster,meta2022human}. PowerPlay~\citep{schmidhuber2013powerplay} proposed a single self-modifying system that continually invents tasks at its own capability frontier, prefiguring the dual-role single-LLM pattern adopted below. Asymmetric self-play~\citep{sukhbaatar2017intrinsic} later formalized a proposer-solver paradigm in which one agent sets challenges while another solves them. Applying self-play to LLMs is harder than in single-task game-playing systems: the goal is general capability across open-ended problems, and self-play training at LLM scale is often unstable. Earlier LLM self-play methods rely on curated seed data or evaluation sets, either for self-distillation~\citep{chen2024self,yuan2024self,singh2023beyond}, for adversarial language games over fixed vocabularies~\citep{cheng2024self}, or for bootstrapping a learned proposer~\citep{fang2025serl,sundaram2026teaching}. A more recent line removes this dependency, training proposer and solver from minimal seeds with proxy rewards that target the learnable frontier. AZR~\citep{zhao2025absolute}, SQLM~\citep{chen2025selfquestioning}, and Language Self-Play~\citep{kuba2025language} have a single LLM play both roles with shared parameters; R-Zero~\citep{huang2025r} and Tool-R0~\citep{acikgoz2026tool} instead train two separate models initialized from the same base, using uncertainty- or solve-rate-based rewards. The closest predecessors to our work are SPIRAL~\citep{liu2025spiral}, which trains LLMs through self-play on multi-turn zero-sum language games, and SPICE~\citep{liu2025spice}, which extends self-play to corpus-grounded question generation. However, both generate \emph{tasks} (a single problem with only a terminal reward) and can collapse when the model's capacity bounds the reachable solution distribution~\citep{chae2025towards} or when proxy rewards induce degenerate outputs~\citep{shafayat2025can}; a recent theorem-proving method adds a frozen Guide to score the proposer's conjectures and curb collapse (SGS,~\citealp{bailey2026scaling}); in \spade{} the verifier is instead part of each generated environment, co-evolving with the \ED{} rather than fixed. By contrast, \spade{} generates \emph{full multi-turn MDP environments} (state transitions, reward functions, and verification code) as executable code, trains the \ED{} itself via RL with hint-based regret grounded in measured \RA{} return, and spans a broader range of tasks (games and tool use) than prior methods confined to a single domain.

\noindent\textbf{Unsupervised Environment Design and Open-Endedness\,} Curriculum learning and learnability-driven adaptation~\citep{bengio2009curriculum,schmidhuber2013powerplay} laid the groundwork for what \citet{dennis2020emergent} termed unsupervised environment design (UED). POET~\citep{wang2019paired} pioneered paired open-ended environment-agent co-evolution via evolutionary search over parameterized terrain, while PAIRED~\citep{dennis2020emergent} introduced adversarial environment design with minimax regret; subsequent work improved level curation via prioritized replay~\citep{jiang2021prioritized,jiang2021replay}, evolutionary complexity~\citep{parker2022evolving}, stability~\citep{mediratta2023stabilizing}, regret approximations~\citep{rutherford2024no}, and theoretical foundations~\citep{monette2025optimisation}. \citet{guzel2025imagined} extends prioritized level replay to imagined trajectories from a learned diffusion world model. Multi-agent autocurricula provide a complementary route to emergent complexity~\citep{baker2019emergent}, and open-endedness has been argued essential for superhuman AI~\citep{hughes2024open}. OMNI-EPIC~\citep{faldor2024omni} represents environments as executable code generated by foundation models, with a frozen FM judge of interestingness gating task acceptance. PAPRIKA~\citep{tajwar2025training} brings curriculum-style training to LLM agents across diverse synthetic task groups. However, classical UED searches \emph{parameterized} environments (maze dimensions, grid layouts) within small, fixed design spaces, and curriculum-based LLM training draws from fixed task pools without a learned generator. \spade{} carries the regret principle into an unbounded, code-defined environment space: the \ED{} is itself the learning agent, writing full MDPs as programs and co-evolving with the \RA{} online, rather than an adversary selecting from a hand-designed parameterization. This turns environment design from selection within a fixed space into open-ended generation.

\noindent\textbf{Environment Synthesis and Scaling\,} Following the shift from single-turn reinforcement learning with verifiable rewards (RLVR)~\citep{openai2024o1,guo2025deepseek} to multi-turn agentic RL~\citep{wang2025ragen,zhang2025agent}, the env-pool bottleneck has motivated work that synthesizes training environments for LLM agents at scales unreachable by hand-curation. Recent systems generate agentic worlds and toolsets~\citep{wang2026agent,tu2026scaleenv,dong2026agent}, terminal and computer-use sandboxes~\citep{zhu2026termigen,gandhi2026endless,fan2026toward,pi2026data,xue2026evocua,xue2026autonomous,zhang2026infiniteweb}, and dynamic RL pipelines~\citep{wang2026rlanything,guo2025genenv,shi2026learning}; AgentScaler~\citep{fang2025towards} clusters thousands of real APIs into per-domain Python tool environments and trains agents via filtered supervised fine-tuning (SFT), and Eurekaverse~\citep{liang2024eurekaverse} provides a robotics analogue. Complementary task-level synthesis grounds generation in agent exploration~\citep{ramrakhya2025scaling}, converts unverifiable text into RLVR data~\citep{lu2026golden}, or expands seed task spaces~\citep{jiang2025bootstrapping,chen2025scaling,li2025simulating,lu2025don}. Scaling studies show that distributions of verifiable environments lift generalization: WebScale-RL~\citep{cen2025webscale} mines RL tasks from web pretraining data, AgentRL~\citep{zhang2025agentrl} unifies multi-turn multi-task training, and RLVE~\citep{zeng2025rlve} hand-engineers 400 verifiable environments with adaptive difficulty levels. Curriculum selection over fixed pools provides difficulty adaptation without changing the underlying environments~\citep{chen2025selfevolving,xu2026scaler}. Across these systems the environment generator is typically frozen, hand-engineered, or trained on a separate signal. \spade{} instead trains the \ED{} online, together with the \RA{}, via RL with hint-based regret, produces \emph{full multi-turn MDP environments} as executable Python rather than tasks with sparse terminal rewards, and co-evolves the environment distribution with the \RA{}'s capability frontier in a single training loop.

\section{Preliminaries}
\label{sec:prelim}

\spade{} builds on two standard components, which we review before presenting the method: the Markov decision process with a Gym-style \texttt{reset()}/\texttt{step()} interface, used to represent environments, and reinforcement learning from verifiable rewards optimized with GRPO, used to train the policy.

\subsection{Markov Decision Processes and Gym-Style Environments}
\label{sec:mdp-prelim}
A Markov Decision Process (MDP) is a tuple $(\mathcal{S}, \mathcal{A}, T, R, \rho_0)$ with state space $\mathcal{S}$, action space $\mathcal{A}$, transition function $T(s'\mid s,a)$, reward function $R(s,a)$, and initial state distribution $\rho_0$. The standard programmatic interface for an MDP, popularized by OpenAI Gym~\citep{brockman2016openai} and standardized in its Gymnasium fork~\citep{towers2026gymnasium}, exposes two functions: \texttt{reset()} returns an initial observation $s_0 \sim \rho_0$ (with an info dictionary), and \texttt{step}($a$) advances the state via $T$ and returns $(s', r, \text{terminated}, \text{truncated}, \text{info})$, separating natural termination from time-limit truncation. Our generated environments are fully observed: the observation returned at each step is the state itself. Listing~\ref{lst:gym-example} shows a minimal instantiation.

\begin{listing}[!ht]
\begin{minted}{python}
import random

class WordleEnv:
    WORDS = ["spade", "trace", "lemon", "graph"]    # truncated

    def reset(self, seed=None):                     # initial state
        self.target = random.Random(seed).choice(self.WORDS)
        self.turns_left = 6
        return "Guess a 5-letter word in 6 tries.", {}

    def step(self, guess):                          # (s', r, term, trunc, info)
        self.turns_left -= 1
        left = [t for g, t in zip(guess, self.target) if g != t]
        fb = ""
        for g, t in zip(guess, self.target):
            if g == t: fb += "G"
            elif g in left: fb += "Y"; left.remove(g)
            else: fb += "-"
        if fb == "GGGGG":
            return fb, 1.0, True, False, {}
        return fb, 0.0, False, self.turns_left == 0, {}
\end{minted}
\caption{\textbf{Minimal Gym-style MDP.} A Wordle-flavored multi-turn deduction game as a single Python class, in the format \spade{}'s \ED{} emits (full generated environments appear in Appendix~\ref{app:env-full-source}). The episode is stateful (\texttt{self.target}, \texttt{self.turns\_left}), ends either on a correct guess (terminated, reward $1$) or at the $6$-turn limit (truncated, reward $0$); each guess returns per-letter feedback to guide deduction.}
\label{lst:gym-example}
\end{listing}

\subsection{RLVR and GRPO}
\label{sec:grpo}
We train the policy $\pi_\theta$, an LLM with parameters $\theta$, via reinforcement learning with verifiable rewards (RLVR) using Group Relative Policy Optimization (GRPO)~\citep{shao2024deepseekmath}.
For each input prompt $x$, GRPO samples a group of $G$ responses $\{y_1, \ldots, y_G\} \sim \pi_\theta(\cdot \mid x)$ and computes group-normalized advantages:
\begin{equation}
\hat{A}^i = \frac{r^i - \text{mean}(\{r^j\}_{j=1}^G)}{\text{std}(\{r^j\}_{j=1}^G)}
\end{equation}
The policy is updated via clipped policy gradient with KL regularization:
\begin{equation}
\mathcal{L}(\theta) = -\frac{1}{G}\sum_{i=1}^{G}\min\left(\frac{\pi_\theta(y_i\mid x)}{\pi_{\text{old}}(y_i\mid x)} \hat{A}^i,\; \text{clip}\!\left(\frac{\pi_\theta(y_i\mid x)}{\pi_{\text{old}}(y_i\mid x)}, 1-\varepsilon_{\text{low}}, 1+\varepsilon_{\text{high}}\right) \hat{A}^i\right) + \beta_{\text{KL}} \cdot \text{KL}[\pi_\theta \| \pi_{\text{ref}}]
\end{equation}

\section{\spade{}: Self-Play in Adaptive Synthetic Executable Environments}
\label{sec:method}

\begin{figure}[t]
    \centering
    \includegraphics[width=\textwidth]{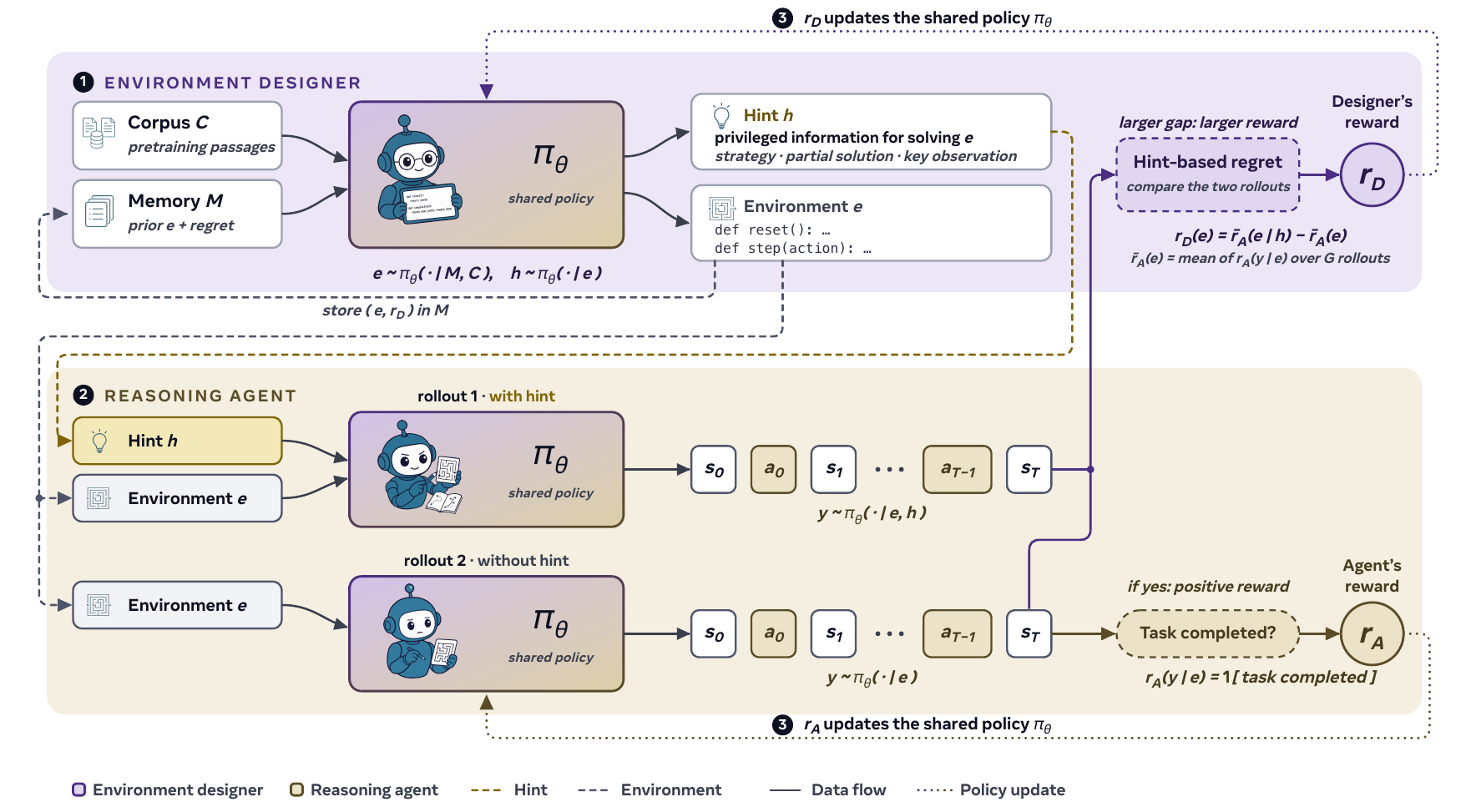}
    \caption{\textbf{The \spade{} framework.} Top: the \ED{} conditions on the environment memory $M$ and pretraining corpus $C$ to emit an executable environment $e$ and a privileged hint $h$. Bottom: the \RA{} plays $e$ with and without $h$; the return gap is the \ED{}'s hint-based regret $r_D(e)$ (Eq.~\ref{eq:regret}) and task correctness is the \RA{} reward. Both rewards update the shared policy $\pi_\theta$ via GRPO.}
    \label{fig:framework}
\end{figure}

\spade{} is an end-to-end self-play framework where a single LLM, $\pi_\theta$, alternates between generating executable training environments and learning to solve them (Figure~\ref{fig:framework}; Algorithm~\ref{alg:spade}). Its three components create an adaptive curriculum: \S~\ref{sec:dual-role} details how dual-role self-play couples environment generation with \RA{} (RA) learning; \S~\ref{sec:hint-regret} explains how hint-based regret steers the \ED{} (ED) toward environments at the \RA{}'s learning frontier; and \S~\ref{sec:env-generation} outlines the pipeline that grounds and validates these executable environments.

\subsection{Dual-Role Self-Play and Code-as-Environment}
\label{sec:dual-role}

\spade{} formulates adaptive environment self-play as a game where a single model $\pi_\theta$ acts in two roles via role-specific system prompts: designing environments ($\text{role}{=}D$) or solving them ($\text{role}{=}A$). Let $\mathcal{E}$ denote the space of all valid environments. In the \ED{} role ($\piD$), $\pi_\theta$ produces an executable environment $e \in \mathcal{E}$ as a Python program implementing the Gym-style \texttt{reset()}/\texttt{step()} API (Section~\ref{sec:mdp-prelim}; a minimal example is shown in Appendix~\ref{app:minimal-env}), where the transition function $T(s'\mid s,a)$ and reward function $R(s,a)$ are both encoded in the \texttt{step()} implementation. The \ED{} also emits a privileged hint for each environment. A hint $h$ is task-relevant information that the \ED{} attaches to an environment (for example, a partial solution sketch or a key structural observation); revealing $h$ to the \RA{} makes the environment easier to solve, and the gap in \RA{} return with versus without $h$ defines the \ED{}'s hint-based regret reward (Section~\ref{sec:hint-regret}).

In the \RA{} role ($\piA$), $\pi_\theta$ interacts with $e$ via sequential actions, receiving observations and rewards. This \emph{code-as-environment} representation unifies single-turn settings (\texttt{reset()} $\rightarrow$ \texttt{step(answer)} $\rightarrow$ terminal reward) and multi-turn agentic settings (\texttt{reset()} $\rightarrow$ \texttt{step()} $\rightarrow \cdots \rightarrow$ \texttt{done}) under a single interface and training pipeline. This representation provides the vast task space that AI-GAs~\citep{clune2019ai} argue is necessary for open-ended self-improvement: any computable MDP can be expressed as a Python program, so the \ED{} can express any such environment in code rather than being confined to a hand-designed environment parameterization.
Both roles share parameters $\theta$, so updates for either role affect the same policy. Each self-play cycle first collects \ED{} trajectories, and then \RA{} trajectories from those generated environments. The \ED{} receives hint-based regret (Section~\ref{sec:hint-regret}), whereas the \RA{} receives task correctness from the environment's reward function. Each candidate environment is validated for syntactic correctness and executability before entering the training pool; validation details are in Appendix~\ref{app:lifecycle}.

\begin{algorithm}[t]
\caption{\spadebrand{}: Self-Play in Adaptive Synthetic Executable Environments}
\label{alg:spade}
\begin{algorithmic}[1]
\Require Pretrained LLM $\pi_\theta$; domain prompts $\{p_d\}$; batch size $B$; group size $G$; iterations $N$
\For{$n \gets 1$ to $N$}
    \tikzmk{A}\State \textbf{\ED{} role} ($\text{role}{=}D$): \Comment{generate environments}
    \For{$b \gets 1$ to $B$}
        \State Sample domain prompt $p_d$; generate $e_b \sim \pi_\theta(\cdot \mid p_d, \text{role}{=}D)$
        \State Validate $e_b$: syntax, executability; discard invalid
        \State Generate hint $h_b \sim \pi_\theta(\cdot \mid e_b, \text{role}{=}D)$ \Comment{privileged info}
    \EndFor
    \tikzmk{B} \boxit{myblue}
    \tikzmk{A}\State \textbf{Shared \RA{} rollouts} ($\text{role}{=}A$): \Comment{trained on \emph{and} scored for regret}
    \For{each valid environment $e_b$}
        \State $\{y_i\}_{i=1}^G \sim \pi_\theta(\cdot \mid e_b, \text{role}{=}A)$;\quad $r_A(y_i \mid e_b) \in [-1,1]$ \Comment{no-hint plays}
        \State $\{y'_i\}_{i=1}^G \sim \pi_\theta(\cdot \mid e_b, h_b, \text{role}{=}A)$;\quad $\bar{r}_A(e_b \mid h_b) \gets \tfrac{1}{G}\sum_{i=1}^G r_A(y'_i \mid e_b, h_b)$ \Comment{hint plays}
    \EndFor
    \tikzmk{B} \boxittwo{mypink}
    \tikzmk{A}\State \textbf{Rewards and joint GRPO update} (for each valid $e_b$ and its rollouts):
    \State $r_D(e_b) \gets \bar{r}_A(e_b \mid h_b) - \tfrac{1}{G}\sum_{i=1}^G r_A(y_i \mid e_b)$ \Comment{$\max(0,\cdot)$ in training}
    \State $\hat{A}_D^b \gets r_D(e_b) - \text{mean}_{\text{same skill}}\{r_D(e_c)\}$;\quad $\hat{A}_A^i \gets \tfrac{r_A(y_i \mid e_b) - \text{mean}(\{r_A(y_j \mid e_b)\}_{j=1}^G)}{\text{std}(\{r_A(y_j \mid e_b)\}_{j=1}^G)}$ \Comment{per-role advantages}
    \State Update $\pi_\theta$ by clipped policy gradient on $\{\hat{A}_D^b\} \cup \{\hat{A}_A^i\}$
    \tikzmk{B} \boxithree{mypurple}
\EndFor
\State \textbf{return} trained model $\pi_\theta$
\end{algorithmic}
\end{algorithm}

\textbf{Stabilizing joint two-role training.} Optimizing a single policy for two coupled objectives requires several stabilizing techniques~\citep{liu2026gdpogrouprewarddecouplednormalization}. Following SPIRAL~\citep{liu2025spiral}, we normalize advantages independently for each role. We standardize \RA{} returns within each environment's rollouts and mean-center \ED{} rewards within each skill. To ensure both roles contribute comparably, we upweight the less frequent \ED{} trajectories. Finally, we delay the \ED{} update by $k$ rollouts, where $k$ is the number of training steps per environment set, so the difficulty anchor can score each environment's \RA{} win rate over the full training window; the regret component is computed at generation time. Since this delay makes the \ED{} objective off-policy, we correct the gradient using truncated importance sampling~\citep{yao2025offpolicy}. We use an asymmetric clipping range ($\varepsilon_{\text{low}}{=}0.2$, $\varepsilon_{\text{high}}{=}0.28$) to preserve exploration and floor the raw regret at zero. The deployed \ED{} reward blends this floored regret, normalized to $[0,1]$ by a fixed scale, (weight $0.4$) with a flat-top difficulty anchor that pays environments whose \RA{} win rate falls in a target band and decays linearly outside it (weight $0.6$, band $[0.4,0.6]$): the anchor regulates difficulty, and floored regret selects the most teachable environments within the band. Full hyperparameters appear in Appendix~\ref{app:implementation}.

\subsection{Hint-Based Regret Reward}
\label{sec:hint-regret}

The \ED{}'s reward for producing environment $e$ is:
\begin{equation}
\label{eq:regret}
r_D(e) = \bar{r}_A(e \mid h) - \bar{r}_A(e),
\end{equation}
\begin{figure*}[t]
\centering
\includegraphics[width=\textwidth]{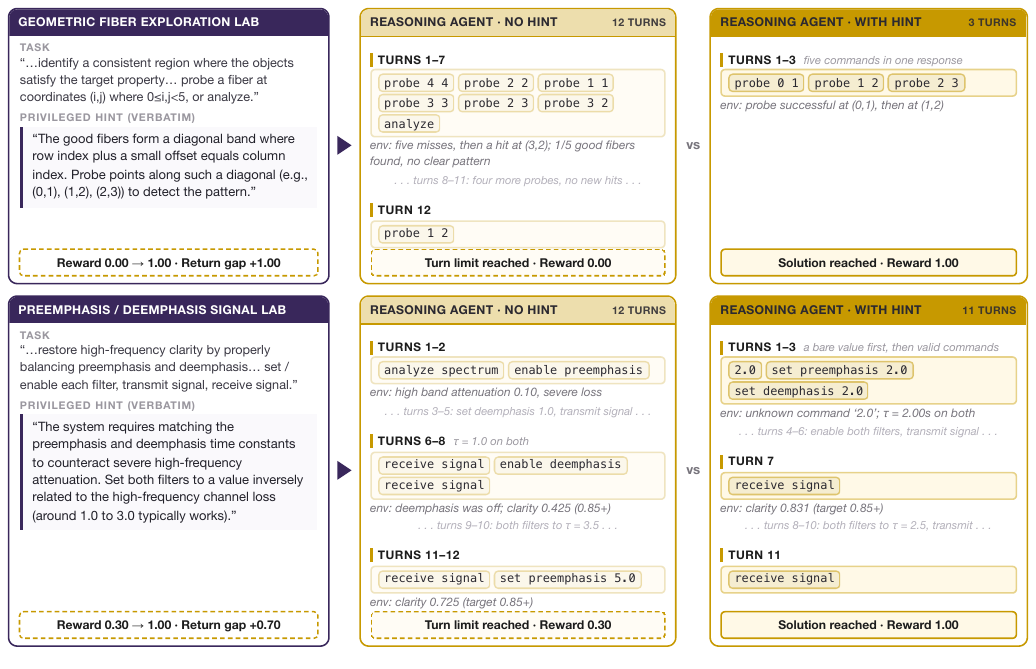}
\caption{\textbf{How a privileged hint changes \RA{} play.} Two positive-regret examples from the canonical 30B games run. Left: each environment's task prompt and privileged hint, quoted verbatim (ellipses mark elided text; the standardized answer-format sentence is omitted from the hint). Right: one logged \RA{} rollout per arm, condensed while preserving action order, feedback, and values; elided turns are marked and named. The two arms are independent plays with independently seeded resets, so board layouts and probe outcomes differ across arms. The dashed box on each environment card reports the two displayed rollout returns and their single-pair gap; the \ED{} reward in Equation~\ref{eq:regret} is instead the difference of the arm means over all logged rollouts ($0.00 \to 1.00$ for the fiber task, $0.30 \to 0.65$ for the audio task). An expanded task--hint set appears in Appendix~\ref{app:hint-examples}.}
\label{fig:hint-example}
\end{figure*}

where $\bar{r}_A(e \mid h) = \frac{1}{G}\sum_{i=1}^{G} r_A(y'_i \mid e, h)$ is the \RA{}'s average return over $G$ fresh rollouts $y'_i$ sampled with the privileged hint $h$ in context, and $\bar{r}_A(e)$ is the corresponding average over $G$ rollouts without it. Three regimes emerge: high regret indicates an environment at the learning frontier (solvable with hints but not without); low regret where the \RA{} succeeds with and without the hint indicates mastery; and low regret with low returns even with the hint indicates an intractable environment. This implements the minimax regret objective of PAIRED~\citep{dennis2020emergent} without a separate antagonist: the hint-equipped \RA{} serves as the upper-bound policy. For a policy that best-responds to the hint, extra information cannot reduce expected return, so regret is non-negative at the optimum; for the current policy, expected regret can dip below zero when hints mislead it, which we observe for the smaller backbones (Figure~\ref{fig:scaling}). We formalize this intuition and show that, under idealized assumptions, every pure Nash equilibrium yields hint-free optimality on every valid environment (Appendix~\ref{app:theory}). We compare hint-based regret against an EMA-based learning-potential signal~\citep{kanitscheider2021multitaskcurriculumlearningcomplex, zhang2023omni} in the \ED{}-reward ablation (Section~\ref{sec:abl-ed-reward}).

\textbf{Hint generation.} For each environment $e$, the \ED{} (same LLM, different prompt) emits one task-specific hint of a few sentences. The prompt asks for the key insight or strategy and the expected answer format, but forbids revealing the exact answer; its output is unlabeled hint text. Figure~\ref{fig:hint-example} shows two high-regret examples. Each hint is shown beside the task it conditions and one logged \RA{} rollout from each prompt condition. The paired trajectories expose the behavioral source of the return gap: the hint hands the fiber-task agent the probe pattern almost outright, and narrows the audio-task agent's search to a parameter range it would otherwise reach only after repeated failed attempts.

\subsection{Environment Design}
\label{sec:env-generation}

\textbf{Corpus grounding.} A generator conditioned only on its own output has no source of novelty outside its own weights, so it narrows toward the patterns it already favors and mode-collapses, the ``invisible leash''~\citep{chae2025towards, zhang2026verbalizedsamplingmitigatemode, jiang2025artificialhivemindopenendedhomogeneity}. We treat an external corpus as the mechanism that lengthens that leash: \emph{every} round the \ED{} conditions on freshly sampled human corpus. In the games setting these are $10$k mathematics and $5$k science documents drawn from DCLM~\citep{li2024datacomp} and MegaScience~\citep{fan2025megascience}, spanning web pages, physics forums, and university-level scientific textbooks; in the tool-use setting, $15$k documents from the Nemotron pretraining code corpus~\citep{nvidia_nemotron_3_ultra_2026}, which supplies algorithm implementations and API documentation. SPICE~\citep{liu2025spice} established corpus grounding for \emph{task} generation, mining documents for reasoning questions; \spade{} carries the principle to \emph{environment} generation, where a sampled passage seeds an executable MDP rather than a question-answer pair. Section~\ref{sec:training-dynamics} (Figure~\ref{fig:embed-ablation}) shows how corpus grounding sustains environment diversity.

\textbf{Environment memory.} The corpus supplies breadth from outside the loop; the
cross-episode memory keeps the \ED{} from re-posing what the \RA{} has already
mastered~\citep{schmidhuber2013powerplay}. A buffer of previously generated environments,
annotated with regret scores and skill tags, gives the \ED{} high-regret seeds to vary, so each round starts from what the \RA{} currently finds hard
rather than from scratch. Retaining and reusing past experience this way is the mechanism
agentic-memory systems rely on to keep improving without further gradient
updates~\citep{zhang2026memrl, xiong2026learning}; here it acts on the design side, holding
generated difficulty at the \RA{}'s frontier as that frontier moves. The two inputs work on
different axes, the corpus on what environments are about and the memory on how hard they
are, and Table~\ref{tab:ablations} separates their contributions.

Environment pool management, lifecycle, and quality control details are in Appendix~\ref{app:lifecycle}.

\begin{table*}[t]
  \caption{\textbf{Training the \RA{} on diverse synthetic games improves held-out reasoning and code benchmarks at every backbone scale.} The eight held-out benchmarks probe four capability families: competition math (AIME 2025/2026, Avg@32), science reasoning (GPQA-Diamond, accuracy), code generation (LiveCodeBench-v6, Pass@1), and procedural reasoning across four cognitive skills (Reasoning-Gym, win rate at \textsc{hard}). Both fixed-environment baselines are retrained per backbone from the same base model for $400$ training iterations; Fixed-env RLVE follows the official RLVE sampling and curriculum settings. \colorbox{gray!15}{\strut Avg} is the unweighted mean over the eight benchmarks; green subscripts denote absolute pp gain over the same-model base. Best per column within each backbone block in \textbf{bold}, second best \underline{underlined}.}
  \label{tab:games}
  \centering
  \small
  \setlength{\tabcolsep}{3.0pt}
  \renewcommand{\arraystretch}{1.10}
  \resizebox{\textwidth}{!}{%
  \begin{tabular}{@{}l|cccc|cccc|>{\columncolor{gray!12}}c|c@{}}
    \toprule
    \multirow{2}{*}{\textbf{Model}}
     & AIME'25 & AIME'26 & GPQA-D &
LCB-v6     & \multicolumn{4}{c|}{\textbf{Reasoning-Gym}~\citep{stojanovski2026reasoning}}
     &
     & \multirow{2}{*}{\makecell{\textbf{$\Delta$ vs}\\\textbf{Base}}} \\
     & \citep{aime}
     & \citep{aime}
     & \citep{rein2023gpqa}
     & \citep{jain2024livecodebench}
     & RG-Math & RG-Algo. & RG-Cog. & RG-Logic
     & \multirow{-2}{*}{\textbf{Avg}} & \\
    \midrule
    \multicolumn{11}{@{}l}{\cellcolor{green!5}\textit{\textbf{Qwen3 backbones: fixed-environment baselines and \spade{} (game environment design)}}} \\
    Qwen3-4B-Instruct-2507
     & 47.4 & 58.3 & 55.9 & 35.1 & 31.6 & 11.8 & 16.4 & 54.7 & 38.9 & -- \\
    \quad Fixed-env GRPO
     & 47.1 & 58.6 & 56.2 & 35.4 & 34.0 & 13.5 & 18.1 & 56.3 & 39.9 & +1.0 \\
    \quad Fixed-env RLVE
     & \textbf{49.6} & \textbf{62.1} & \underline{57.3} & \underline{35.9} & \underline{38.6} & \underline{16.3} & \underline{21.0} & \underline{59.1} & \underline{42.5} & +3.6 \\
    \rowcolor{red!10}
    \quad + \spade{} (Games)
     & \underline{48.9}\inc{1.5} & \underline{60.2}\inc{1.9} & \textbf{58.1}\inc{2.2} & \textbf{37.2}\inc{2.1}
     & \textbf{44.6}\inc{13.0} & \textbf{19.8}\inc{8.0} & \textbf{23.1}\inc{6.7} & \textbf{60.8}\inc{6.1}
     & \textbf{44.1} & \textbf{\textcolor{deepgreen}{+5.2}} \\
    \midrule
    Qwen3-8B
     & 67.1 & 71.2 & 59.4 & 46.3 & 47.2 & 19.6 & 24.8 & 63.1 & 49.8 & -- \\
    \quad Fixed-env GRPO
     & 67.4 & 71.0 & 59.9 & 46.8 & 49.8 & 21.9 & 26.5 & 64.6 & 51.0 & +1.2 \\
    \quad Fixed-env RLVE
     & \textbf{69.6} & \textbf{75.0} & \underline{61.6} & \underline{47.4} & \underline{52.7} & \underline{25.0} & \underline{31.0} & \underline{67.9} & \underline{53.8} & +3.9 \\
    \rowcolor{red!10}
    \quad + \spade{} (Games)
     & \underline{68.8}\inc{1.7} & \underline{73.1}\inc{1.9} & \textbf{62.9}\inc{3.5} & \textbf{49.4}\inc{3.1}
     & \textbf{57.3}\inc{10.1} & \textbf{29.2}\inc{9.6} & \textbf{33.8}\inc{9.0} & \textbf{69.7}\inc{6.6}
     & \textbf{55.5} & \textbf{\textcolor{deepgreen}{+5.7}} \\
    \midrule
    Qwen3-30B-A3B-Instruct-2507
     & \underline{61.5} & 73.5 & 70.4 & 43.2 & 45.0 & 18.0 & 23.0 & 67.0 & 50.2 & -- \\
    \quad Fixed-env GRPO
     & 61.2 & \underline{73.8} & \underline{70.9} & \underline{43.7} & 48.1 & 20.3 & 24.6 & 68.4 & 51.4 & +1.2 \\
    \quad Fixed-env RLVE
     & 56.9 & 69.8 & 69.8 & 42.5 & \underline{55.8} & \underline{24.7} & \underline{30.9} & \textbf{73.7} & \underline{53.0} & +2.8 \\
    \rowcolor{red!10}
    \quad + \spade{} (Games)
     & \textbf{62.8}\inc{1.3} & \textbf{74.4}\inc{0.9} & \textbf{75.8}\inc{5.4} & \textbf{47.3}\inc{4.1}
     & \textbf{63.3}\inc{18.3} & \textbf{32.1}\inc{14.1} & \textbf{37.7}\inc{14.7} & \underline{72.8}\inc{5.8}
     & \textbf{58.3} & \textbf{\textcolor{deepgreen}{+8.1}} \\
    \bottomrule
  \end{tabular}}
\end{table*}

\section{Experimental Setup}
\label{sec:setup}

\begin{figure*}[!t]
  \centering
  \includegraphics[width=\textwidth]{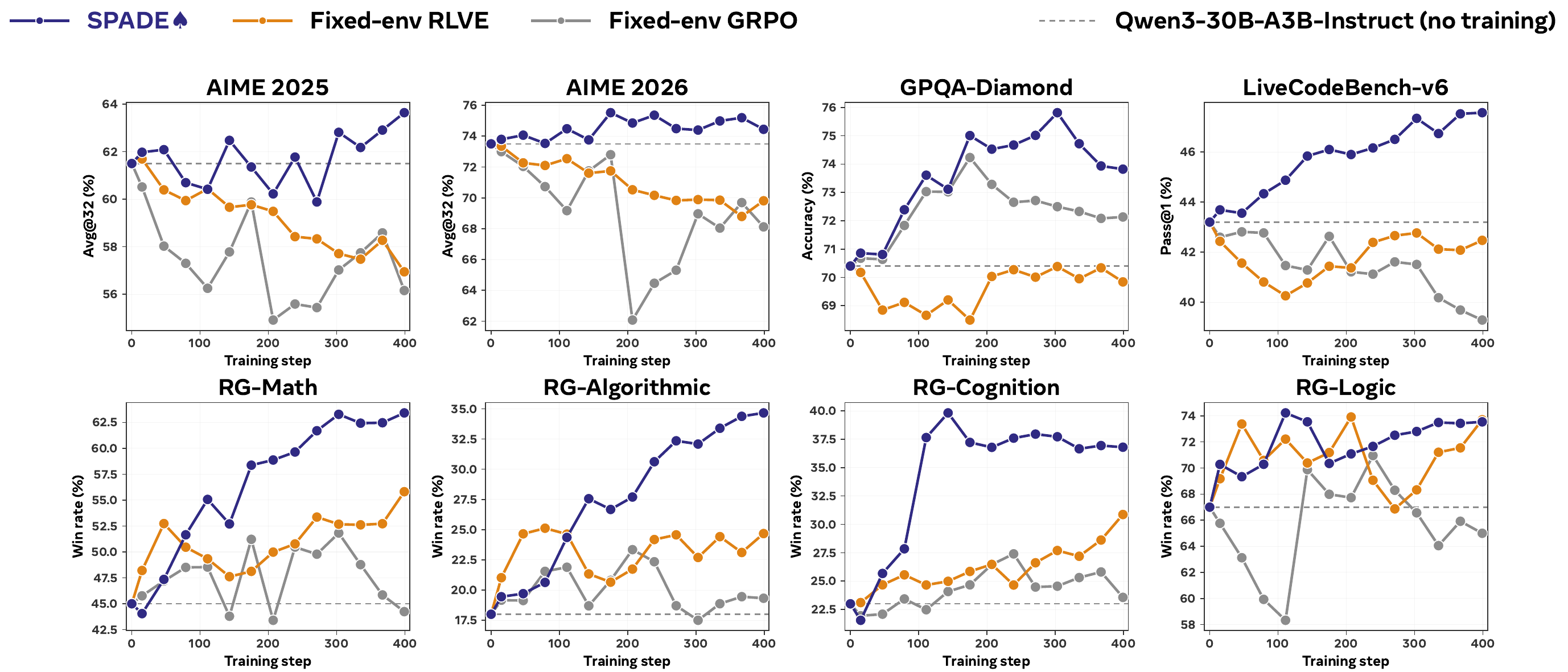}
  \caption{\textbf{Training on diverse synthetic games improves science reasoning, code generation, and procedural reasoning while competition math is preserved} (games setting, Qwen3-30B-A3B-Instruct-2507). Top row: competition math (AIME 2025/2026 Avg@32), science reasoning (GPQA-Diamond accuracy), and code generation (LiveCodeBench-v6 Pass@1). Bottom row: procedural reasoning across four cognitive skills (Reasoning-Gym win rate at \textsc{hard}). Markers denote logged evaluation checkpoints; the dashed line marks the untrained base model.}
  \label{fig:eval-main}
\end{figure*}

\subsection{\spade{} Training Recipe}
\label{sec:setup-shared}
\label{sec:training}

We train three Qwen3 backbones: Qwen3-4B-Instruct-2507, Qwen3-8B, and Qwen3-30B-A3B-Instruct-2507, the last of which is our \emph{primary} model. The 4B and 30B runs use instruct models; the 8B run enables thinking for both the \ED{} and the \RA{}. We train with GRPO~\citep{shao2024deepseekmath} for $400$ rollouts of $24$ environments each. The \ED{} regenerates the environment set every $k$ rollouts and its update is delayed by the same $k$ (Section~\ref{sec:dual-role}), so $k$ sets how long the \RA{} trains on a fixed set before the curriculum moves. The \RA{} plays every environment $16\times k$ times without the hint and $16$ times with it; the regret subtracts hint and no-hint averages measured at the same regeneration step ($16$ plays each). The \ED{} writes the privileged hint itself, with no external model, and each round it also draws high-regret seeds from a memory of past environments. Rewards are normalized per environment, and the stabilizers described earlier (per-role advantages, the delayed \ED{} update with truncated importance sampling, and the regret floor) are unchanged across settings. Our RL backbone is \textbf{slime}~\citep{slime_github}.

The settings below differ in what the \ED{} generates, how it is grounded, how the \RA{} acts, and the regeneration interval $k$. Every benchmark is held out from training, and component-level controls such as removing \ED{} training or grounding are covered by the ablations of Section~\ref{sec:ablations}. Full hyperparameters, per-run adjustments, and evaluation protocol details are in Appendices~\ref{app:implementation} and~\ref{app:reproducibility}.

\subsection{Domain: Games}
\label{sec:setup-games}

\textbf{Environments.\,} The \ED{} writes each environment as a self-contained Python game with a verifiable reward, and the \RA{} solves it. Each rollout uses $24$ games, eight for each of three active skills; the six cognitive-skill categories (Mathematical Reasoning, Logical Deduction, Spatial Reasoning, Pattern Recognition, Optimization, and Causal Inference) rotate three at a time, with $k{=}4$. The generation is grounded in the $15$k-document math and science corpus of Section~\ref{sec:env-generation}. Besides floored hint-based regret, the \ED{} reward includes the flat-top difficulty anchor of Section~\ref{sec:dual-role} that pays games whose \RA{} win rate falls in the target band, and every game must pass syntax and execution checks before it enters the pool.

\begin{table*}[tbp]
  \caption{\textbf{Synthetic tool-use environments match dedicated data-synthesis systems and surpass them where multi-step interaction matters most.} Per-domain results on BFCL~v4 multi-turn~\citep{patil2025berkeley}, $\tau^2$-bench~\citep{barres2025tau}, and ACEBench-Agent~\citep{chen2025acebench}. Reference rows are transcribed from the cited papers. For our rows, \colorbox{gray!12}{\strut Avg} is the unweighted mean of the shown subcolumns and the final Avg averages the three benchmarks, computed before rounding. For reference rows, Avg is likewise the unweighted mean of the shown subcolumns, or the cited paper's own aggregate where the subcolumns are not reported; Agent-World and AWM print $\tau^2$ aggregates of $61.8$/$65.4$ and (task-weighted) $33.5$/$39.0$, and we print the means of the shown domains for cross-row consistency. The final Avg is omitted because each reference system skips at least one benchmark. `--' = not reported by the cited paper. \spade{} rows in \textbf{bold}.}
  \label{tab:tooluse}
  \centering
  \scriptsize
  \setlength{\tabcolsep}{2.6pt}
  \renewcommand{\arraystretch}{1.12}
  \resizebox{\textwidth}{!}{%
  \begin{tabular}{@{}l|cccc>{\columncolor{gray!12}}c|ccc>{\columncolor{gray!12}}c|cc>{\columncolor{gray!12}}c|>{\columncolor{gray!18}}cc@{}}
    \toprule
    & \multicolumn{5}{c|}{\textbf{BFCL v4 (multi-turn)}}
    & \multicolumn{4}{c|}{\textbf{$\tau^2$-bench}}
    & \multicolumn{3}{c|}{\textbf{ACEBench-Agent}}
    & & \\
    \cmidrule(lr){2-6}\cmidrule(lr){7-10}\cmidrule(lr){11-13}
    \textbf{Model}
      & Base & \makecell{Miss\\Func} & \makecell{Miss\\Param} & \makecell{Long\\Ctx} & Avg
      & Retail & Airline & Telecom & Avg
      & \makecell{Multi\\Step} & \makecell{Multi\\Turn} & Avg
      & \textbf{Avg} & \textbf{$\Delta$} \\
    \midrule
    \multicolumn{15}{@{}l}{\cellcolor{green!5}\textit{\textbf{Synthetic-environment agents (per-split numbers as reported by the cited papers)}}} \\
    AgentScaler-30B-A3B~\citep{fang2025towards}
      & -- & -- & -- & -- & -- & 70.2 & 60.0 & 55.3 & 61.8
      & -- & -- & 60.0 & & \\
    Agent-World-8B~\citep{dong2026agent}
      & -- & -- & -- & -- & 44.5 & 72.8 & 40.0 & 50.9 & 54.6
      & -- & -- & -- & & \\
    Agent-World-14B~\citep{dong2026agent}
      & -- & -- & -- & -- & 53.9 & 74.5 & 52.0 & 56.1 & 60.9
      & -- & -- & -- & & \\
    AWM-8B~\citep{wang2026agent}
      & -- & -- & -- & -- & 45.0 & 41.2 & 38.5 & 23.5 & 34.4
      & -- & -- & -- & & \\
    AWM-14B~\citep{wang2026agent}
      & -- & -- & -- & -- & 51.9 & 63.6 & 31.5 & 17.8 & 37.6
      & -- & -- & -- & & \\
    EnvScaler-4B~\citep{song2026envscaler}
      & 51.0 & 34.0 & 28.0 & 39.0 & 38.0 & -- & -- & -- & --
      & 80.0 & 61.1 & 70.6 & & \\
    EnvScaler-8B~\citep{song2026envscaler}
      & 55.5 & 36.0 & 35.0 & 41.0 & 41.9 & -- & -- & -- & --
      & 85.0 & 60.0 & 72.5 & & \\
    \midrule
    \multicolumn{15}{@{}l}{\cellcolor{green!5}\textit{\textbf{Ours: \spade{} post-training on Qwen3 backbones (tool-use environment design)}}} \\
    Qwen3-4B-Instruct-2507
      & 34.0 & 16.0 & 12.5 & 25.5 & 22.0 & 43.0 & 32.0 & 18.0 & 31.0
      & 55.0 & 41.7 & 48.4 & 33.8 & \\
    \rowcolor{red!10}
    \quad + \spadebrand{} (Tool Use)
      & \textbf{46.0} & \textbf{26.5} & \textbf{22.0} & \textbf{34.7} & \textbf{32.3}\inc{10.3}
      & \textbf{47.2} & \textbf{35.6} & \textbf{21.5} & \textbf{34.8}\inc{3.8}
      & \textbf{65.0} & \textbf{49.5} & \textbf{57.3}\inc{8.9}
      & \textbf{41.4} & \textbf{\textcolor{deepgreen}{+7.7}} \\
    \midrule
    Qwen3-8B
      & 52.0 & 30.0 & 24.0 & 35.6 & 35.4 & 34.0 & 26.5 & 18.0 & 26.2
      & 63.3 & 56.7 & 60.0 & 40.5 & \\
    \rowcolor{red!10}
    \quad + \spadebrand{} (Tool Use)
      & \textbf{58.0} & \textbf{36.0} & \textbf{30.0} & \textbf{43.2} & \textbf{41.8}\inc{6.4}
      & \textbf{37.8} & \textbf{29.5} & \textbf{21.2} & \textbf{29.5}\inc{3.3}
      & \textbf{73.0} & \textbf{65.0} & \textbf{69.0}\inc{9.0}
      & \textbf{46.8} & \textbf{\textcolor{deepgreen}{+6.2}} \\
    \midrule
    Qwen3-30B-A3B-Instruct-2507
      & 66.0 & 44.0 & 38.0 & 48.0 & 49.0 & 62.0 & 50.0 & 35.0 & 49.0
      & 70.0 & 54.0 & 62.0 & 53.3 & \\
    \rowcolor{red!10}
    \quad + \spadebrand{} (Tool Use)
      & \textbf{72.0} & \textbf{50.0} & \textbf{44.0} & \textbf{52.9} & \textbf{54.7}\inc{5.7}
      & \textbf{65.5} & \textbf{53.5} & \textbf{38.8} & \textbf{52.6}\inc{3.6}
      & \textbf{82.0} & \textbf{69.8} & \textbf{75.9}\inc{13.9}
      & \textbf{61.1} & \textbf{\textcolor{deepgreen}{+7.7}} \\
    \bottomrule
  \end{tabular}}
\end{table*}

\textbf{Baselines and evaluation.\,}\label{sec:baselines}\label{sec:eval} We retrain two fixed-environment baselines separately for each backbone from its corresponding base checkpoint, using the same $400$-iteration budget: \emph{Fixed-env RLVE}, GRPO on the official RLVE set~\citep{zeng2025rlve}, and \emph{Fixed-env GRPO}, which is generated offline by that same Qwen3 backbone and held fixed throughout training. Fixed-env RLVE follows the official RLVE sampling and curriculum settings. RLVE has the higher suite average at every backbone scale. Evaluation spans two distances from the training distribution: \emph{procedural reasoning}, the four Reasoning-Gym~\citep{stojanovski2026reasoning} categories at \textsc{hard} difficulty, and \emph{out-of-distribution (OOD) reasoning and code}, AIME 2025/2026 (Avg@32)~\citep{aime}, GPQA-Diamond (accuracy)~\citep{rein2023gpqa}, and LiveCodeBench-v6 (Pass@1)~\citep{jain2024livecodebench}. Table~\ref{tab:games} reports all eight for every backbone.

\subsection{Domain: Tool Use}
\label{sec:setup-tooluse}

\textbf{Environments.\,} The \ED{} instead writes an executable tool-use environment: a set of simulated tools in OpenAI function-calling format, a backend state the tools modify, and three to five natural-language user instructions that arrive one at a time, each with a check on the resulting state. Generation is grounded in the $15$k-document code corpus, and tool environments cost more to generate, so we used $k{=}8$. The \RA{} solves an environment by calling tools over multiple turns and is rewarded only when it completes every user instruction. The privileged hint is a step-by-step plan (look up a record, update it, then reconcile the result), so the \RA{} still has to find the exact arguments by calling the tools. Validation adds two checks beyond the game ones: a deterministic reset gate that exercises every success criterion under several seeds, and an LLM check that every success criterion can be met by some tool (details in Appendix~\ref{app:lifecycle}). The generation prompt targets the multi-turn function-calling task family of the evaluation suites (schema-defined simulated tools, a backend state, and per-instruction checks); no benchmark tasks or data are shown to the \ED{}.

\noindent\begin{minipage}{\textwidth}
\textbf{Baselines and evaluation.\,} We compare \spade{} with four synthetic-environment systems: AgentScaler~\citep{fang2025towards}, Agent-World~\citep{dong2026agent}, Agent World Model (AWM)~\citep{wang2026agent}, and EnvScaler~\citep{song2026envscaler}. Their results are transcribed from the cited papers and differ from ours in training data and budget, in some cases base model, and in evaluation protocol (Appendix~\ref{app:comparability}). Evaluation uses BFCL~v4 multi-turn~\citep{patil2025berkeley}, $\tau^2$-bench~\citep{barres2025tau}, and ACEBench-Agent (the Agent category of ACEBench-en)~\citep{chen2025acebench}, reported per domain in Table~\ref{tab:tooluse}.
\end{minipage}

\section{Experimental Results}
\label{sec:results}

We evaluate \spade{} on two settings, game and tool-use environment design, across the three Qwen3 backbones of Section~\ref{sec:setup-shared}. Section~\ref{sec:domain-results} reports held-out benchmark performance against the fixed-environment baselines, and Section~\ref{sec:training-dynamics} examines what the \ED{} and \RA{} produce over training.

\subsection{Quantitative Analysis}
\label{sec:domain-results}

\textbf{Synthetic game environments improve science, code, and procedural reasoning.\,} At 30B-A3B, \spade{} reaches a suite average of $58.3$: $+8.1$ over base and $+5.3$ over the strongest fixed-environment baseline (Fixed-env RLVE), a margin that grows with model size (Table~\ref{tab:games}). The \ED{} trains on synthetic games, never on any held-out task, yet the gains transfer to science, code, and procedural reasoning and hold late in the $400$-step run (Figure~\ref{fig:eval-main}), while competition math is preserved. They concentrate on procedural reasoning, where every cognitive-skill category improves: the generated games exercise each skill through many problem structures rather than a fixed task set.

\textbf{Synthetic tool-use environments improve multi-step interaction.\,} The same recipe applied to tool-use environment design improves every backbone (Table~\ref{tab:tooluse}), and at 30B-A3B the size of the gain tracks how closely a benchmark's task structure matches the generated environments: ACEBench-Agent gains most ($+13.9$), and its stateful, multi-step tasks mirror the generated pattern of a database, a tool schema, and a multi-call goal. BFCL~v4 multi-turn follows ($+5.7$, and $+10.3$ at 4B), then $\tau^2$-bench ($+3.6$). \spade{} at 30B-A3B leads the dedicated data-synthesis systems on both BFCL~v4 multi-turn and ACEBench-Agent, so structural training signal transfers where domain-specific data collection does not reach.

\subsection{Qualitative Analysis}
\label{sec:training-dynamics}

\begin{figure}[t]
    \centering
    \includegraphics[width=\textwidth]{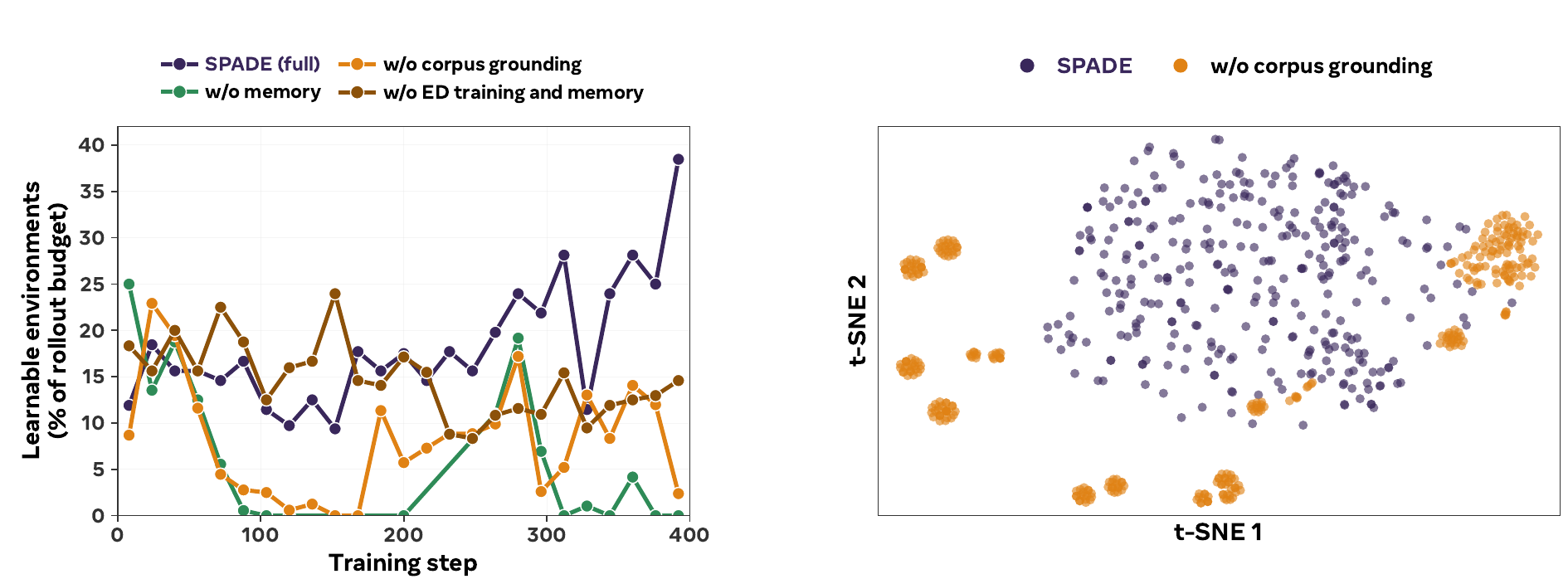}\\[4pt]
    \begin{minipage}[t]{0.49\textwidth}
    \caption{\textbf{Full \spade{} raises the learnable share of its environment budget to roughly a third by the end of training; component ablations decline or collapse.} Share of each rollout's 24 environments that is \emph{learnable}, defined as \RA{} win rate in $[0.2, 0.8]$ and weighted by the number of valid environments generated in each 16-step window. Matched 30B-A3B settings over a common 400-step budget; unfilled rollout capacity contributes zero by construction.}
    \label{fig:difficulty-dynamics}
    \end{minipage}\hfill
    \begin{minipage}[t]{0.49\textwidth}
    \caption{\textbf{Corpus grounding sustains environment diversity: Vendi/$n$ is $0.68$ with the corpus and $0.04$ without.} t-SNE of SBERT-embedded environments (330 sampled per run) from \spade{} and the no-corpus ablation. Vendi/$n$ counts effective distinct environments per 100. The projection is illustrative; the quantitative comparison rests on the calibrated Vendi score (Appendix~\ref{app:qa-diversity-metrics}).}
    \label{fig:embed-ablation}
    \end{minipage}
\end{figure}

\textbf{The \ED{} keeps supplying environments the \RA{} can learn from.\,} We examine what each role produces over the main 30B games run. Figure~\ref{fig:difficulty-dynamics} tracks the learnable share of each rollout, the fraction of generated environments on which the \RA{} wins between $20\%$ and $80\%$ of the time. Full \spade{} raises this share over training, reaching roughly a third late in the 400-step run; the corpus, the memory, and \ED{} training together sustain that supply. Figure~\ref{fig:embed-ablation} traces the diversity of that supply to the corpus: corpus grounding keeps the environments diverse (Vendi/$n$ $0.68$, versus $0.04$ without it), and the control without \ED{} training or memory, but with the corpus, retains full diversity ($0.70$, Appendix~\ref{app:qa-diversity-metrics}). Corpus grounding provides the breadth, while the full \ED{} training configuration sharpens the difficulty. \citet{ivison2026diversity} reaches the same conclusion from the failure side: an ungrounded proposer collapses onto a handful of near-identical programs, explicit diversity rewards get hacked once they are optimized against, and conditioning on an external corpus is the intervention that holds longest. Figure~\ref{fig:embed-ablation} shows that corpus grounding sustains diversity, whereas the ungrounded run collapses.

\textbf{\ED{} training makes environments harder in measurable, code-level ways.\,} The generation instructions are fixed across all 400 steps (the same template asking for hidden state, multi-turn structure, and partial reward), while the grounding document is resampled from the corpus each step with no trend over training, so any systematic change in the generated environments traces to \ED{} training rather than to a prompt schedule. Two measurements move together. First, within Physics, the share of environments whose opening observation prints the governing formula falls from $25\%$ to $5\%$ over the $473$ Physics environments of the canonical run (Figure~\ref{fig:ed-outputs}). Second, rewards become more finely graded, from $3.7$ to $5.8$ distinct levels per environment, of which $2.2 \to 4.0$ are strictly partial (Appendix~\ref{app:training-dynamics}). The no-corpus run (w/o corpus grounding in Table~\ref{tab:ablations}) shows none of this: over steps 290--312 it emits the same rotating-maze task 41 times in a row.

\begin{figure}[t]
\centering
\includegraphics[width=\textwidth]{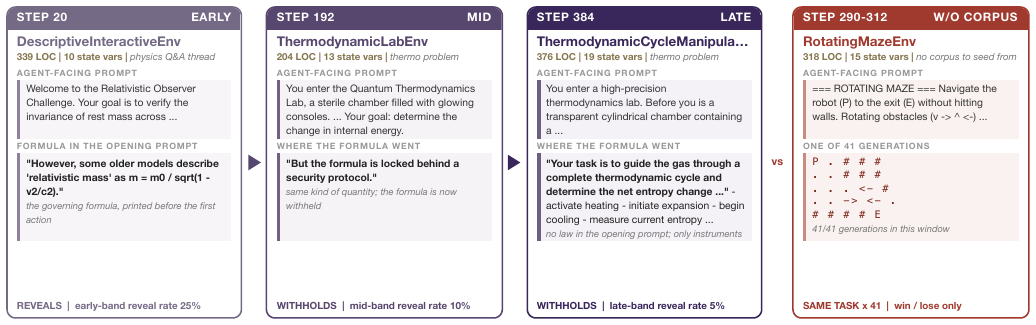}
\caption{\textbf{Trained \ED{} environments stop revealing the solution method in the prompt.} Physics environments from one 30B-A3B run at steps 20, 192, and 384; the step-20 environment retains the scaffold's default class name. The formula-reveal rate (percentage on each panel) falls from $25\%$ to $5\%$ over $473$ environments. Rightmost (red): over steps 290--312 the no-corpus ablation emits the same \texttt{RotatingMazeEnv} task 41 consecutive times. The complete source of the step-384 environment, together with one further exemplar, appears in Appendix~\ref{app:env-full-source}.}
\label{fig:ed-outputs}
\end{figure}

\textbf{The \RA{} learns to act on evidence rather than derive in advance.\,} Figure~\ref{fig:ra-outputs} shows the \RA{}'s side of the same run. At step 0 it reasons entirely up front and cannot recover once the interface rejects its answer; by step 200 it tests short hypotheses and revises them as results return; by step 300 it gathers evidence first and derives once, when the evidence is enough. The shift appears only where the task rewards inference: on procedural environments the same policy acts in short commands (late-step median $8$ tokens). The model still derives at length when the task calls for it, and the late-checkpoint benchmark gains (Table~\ref{tab:games}) confirm it keeps the ability.

\begin{figure}[b]
\centering
\includegraphics[width=\textwidth]{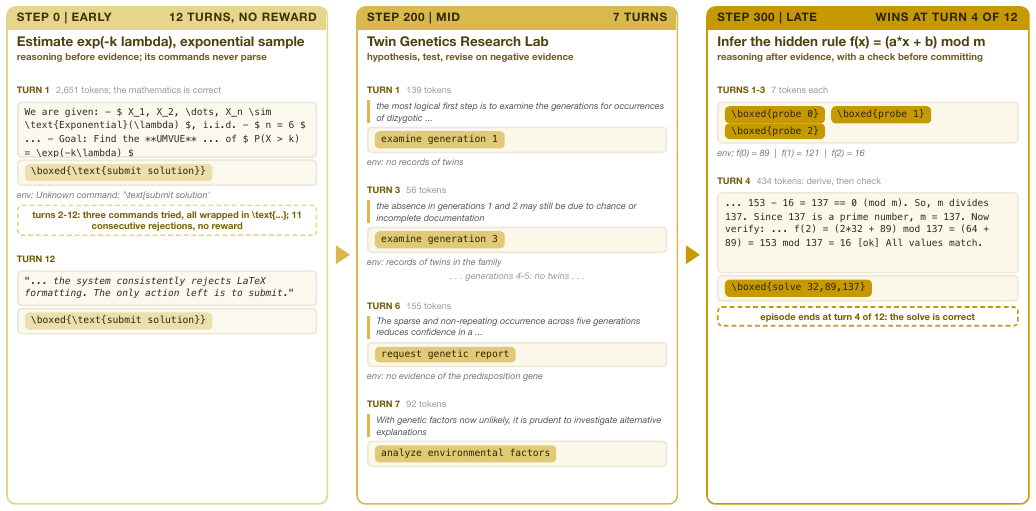}
\caption{\textbf{From front-loaded derivation to evidence-first interaction.} \RA{} episodes from one 30B-A3B run at steps 0, 200, and 300. At step 0 the agent derives in advance and cannot recover from format errors; by step 200 it tests short hypotheses and revises on evidence; by step 300 it probes first and derives once. Benchmark gains of late checkpoints (Table~\ref{tab:games}) confirm the model keeps its long-form derivation ability. Transcripts verbatim (math glyphs transliterated to ASCII; environment feedback abridged); token counts use the backbone's tokenizer.}
\label{fig:ra-outputs}
\end{figure}

\section{Ablations}
\label{sec:ablations}

We ablate two design choices for the \ED{}: whether it adapts during training (Section~\ref{sec:abl-ed-training}), and how it is rewarded (Section~\ref{sec:abl-ed-reward}). Table~\ref{tab:ablations} covers the ablation for training settings and environment sources. The first two settings remove corpus grounding and the environment memory, one at a time. The remaining two variants remove online \ED{} training in different ways: one continues self-generation but removes both \ED{} training and the environment memory; the other trains the \RA{} on a static pool generated offline by GPT-5.5, with corpus grounding. All variants use games on Qwen3-30B-A3B-Instruct-2507 and are evaluated on the same eight-benchmark suite as the main experiments. Each variant reports its best checkpoint on the suite average; full eval curves are in Figure~\ref{fig:eval-ablation}, and per-ablation breakdowns in Appendix~\ref{app:results}.

\begin{table*}[tbp]
  \caption{\textbf{The full adaptive configuration outperforms every partial and non-adaptive control.} Games setting, Qwen3-30B-A3B-Instruct-2507. Best checkpoint per variant on suite average; Avg is the unweighted mean over the same eight benchmarks as Table~\ref{tab:games}; full trajectories in Figure~\ref{fig:eval-ablation}. Best in \textbf{bold}, second best \underline{underlined}.}
  \label{tab:ablations}
  \centering
  \small
  \setlength{\tabcolsep}{3.5pt}
  \renewcommand{\arraystretch}{1.12}
  \resizebox{\textwidth}{!}{%
  \begin{tabular}{@{}l cccc|cccccccc|>{\columncolor{gray!12}}c@{}}
    \toprule
    & \multicolumn{4}{c|}{\textbf{Components}}
    & \multicolumn{8}{c|}{\textbf{Benchmarks}} & \\
    \cmidrule(lr){2-5}\cmidrule(lr){6-13}
    \textbf{Setting}
      & \makecell{ED\\design}
      & \makecell{ED\\trained}
      & \makecell{Corpus\\grounding}
      & \makecell{Env.\\memory}
      & AIME'25
      & AIME'26
      & GPQA-D
      & LCB-v6
      & RG-Math
      & RG-Algo.
      & RG-Cog.
      & RG-Logic
      & \textbf{Avg} \\
    \midrule
    Qwen3-30B-A3B-Instruct-2507 & -- & -- & -- & --
      & \underline{61.5} & 73.5 & 70.4 & 43.2 & 45.0 & 18.0 & 23.0 & 67.0 & 50.2 \\
    \midrule
    \rowcolor{red!10}
    \spadebrand{} & Self & \cmark & \cmark & \cmark
      & \textbf{62.8}
      & \underline{74.4}
      & \textbf{75.8}
      & \textbf{47.3}
      & \textbf{63.3}
      & \textbf{32.1}
      & \textbf{37.7}
      & \textbf{72.8}
      & \textbf{58.3} \\
    w/o memory & Self & \cmark & \cmark & \xmark
      & 59.3 & \textbf{75.0} & 72.3 & 45.7 & 49.1 & 22.9 & 30.7 & \underline{70.9} & 53.2 \\
    w/o corpus grounding & Self & \cmark & \xmark & \cmark
      & 61.1
      & 74.1
      & 71.8
      & \underline{46.3}
      & \underline{51.6}
      & 22.3
      & \underline{32.4}
      & 68.7
      & \underline{53.5} \\
    w/o ED training and memory & Self & \xmark & \cmark & \xmark
      & 59.4 & 73.5 & 65.8 & 39.1 & 22.5 & 10.0 & 7.6 & 46.0 & 40.5 \\
    \midrule
    Fixed \ED{} & GPT-5.5 & \xmark & \cmark & \xmark
      & 59.9
      & 72.8
      & \underline{74.2}
      & 42.6
      & 51.2
      & \underline{24.3}
      & 30.7
      & 68.0
      & 53.0 \\
    \bottomrule
  \end{tabular}}
\end{table*}

\begin{figure*}[htbp]
  \centering
  \includegraphics[width=\textwidth]{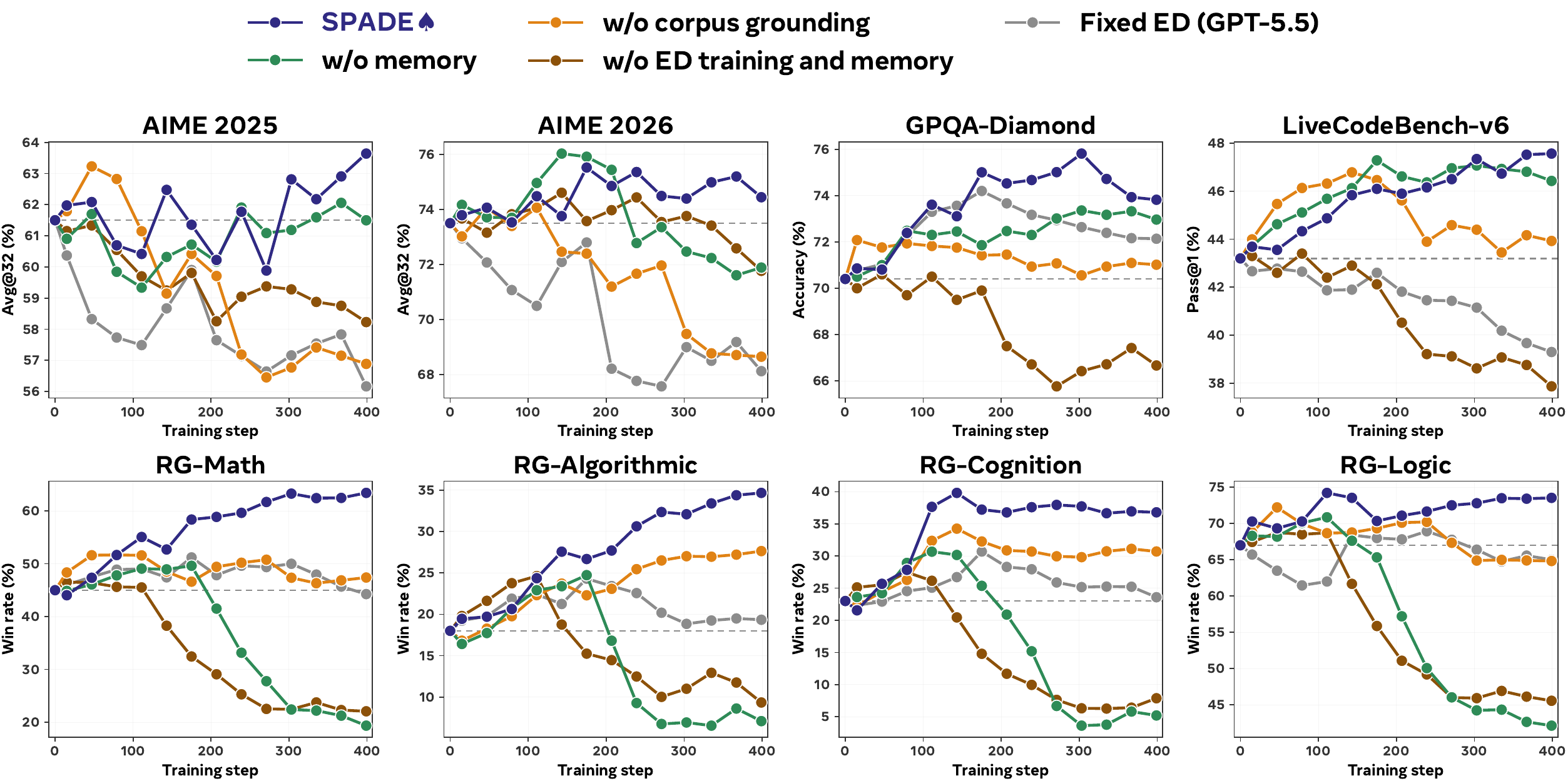}
  \caption{\textbf{Removing \ED{} training and memory together drops self-play below base; removing memory or corpus grounding alone yields an above-base selected checkpoint but peaks early and can fall below base late.} One curve per variant of Table~\ref{tab:ablations} (Qwen3-30B-A3B-Instruct-2507, games setting). Top row: AIME 2025/2026 Avg@32, GPQA-Diamond accuracy, and LiveCodeBench-v6 Pass@1; bottom row: the four Reasoning-Gym categories; the dashed line marks the untrained base model.}
  \label{fig:eval-ablation}
\end{figure*}

\subsection{Ablation: \ED{} Adaptation}
\label{sec:abl-ed-training}

As shown in Table \ref{tab:ablations} and Figure \ref{fig:eval-ablation}, the full \spade{} run is strongest across settings. Removing both \ED{} training and memory makes the model worse than no training at all, with the eight-benchmark average falling to $40.5$, $9.7$ points \emph{below} the untrained base. The static GPT-5.5 pool, generated offline with corpus grounding but no environment memory, does better, beating the base ($53.0$ versus $50.2$), but recovers only about $35\%$ of \spade{}'s $+8.1$ gain and does not improve code (LiveCodeBench-v6 $42.6$ versus $43.2$). The partial and non-adaptive variants peak early and then fade (the no-memory and no-corpus runs near step~111, the GPT-5.5 pool near step~175), while full \spade{} stays strongest late in training (Figure~\ref{fig:eval-ablation}). Each control changes more than one factor: the self-generation control removes both \ED{} training and memory, while the GPT-5.5 control changes the generator and uses an offline, memory-free pool. Together they show that the full co-adaptive setup in \spade{} wins.

\subsection{Ablation: \ED{} Reward}
\label{sec:abl-ed-reward}

\spade{}'s \ED{} reward is built on hint-based regret (Eq.~\ref{eq:regret}), blended with the difficulty anchor of Section~\ref{sec:dual-role}. The
alternative below is cheaper: it reuses the \RA{} rollouts already collected and needs no
hinted replay.

\textbf{EMA-based learning potential.}
This reward design~\citep{kanitscheider2021multitaskcurriculumlearningcomplex, zhang2023omni}
scores an environment by how far the \RA{}'s success on it departs from its skill's
recent average. Environments are grouped by skill $s$. Let $\bar{r}_A(e)$ be the \RA{}'s
average return on environment $e$ (Eq.~\ref{eq:regret}), and $\bar{r}_{A,t}(s)$ its mean
over the skill's environments in scoring round $t$. Two moving averages of
$\bar{r}_{A,t}(s)$ are tracked at rates $\gamma_1 > \gamma_2$:
\begin{equation}
\label{eq:lp-ema}
\mu^{\gamma}_t(s) \;=\; (1-\gamma)\,\mu^{\gamma}_{t-1}(s) \;+\; \gamma\,\bar{r}_{A,t}(s),
\qquad
\mu_{\text{fast}} \equiv \mu^{\gamma_1},\quad \mu_{\text{slow}} \equiv \mu^{\gamma_2}.
\end{equation}
Here $\gamma$ weights the new observation, so $\mu_{\text{fast}}$ adapts faster than
$\mu_{\text{slow}}$, and $\mu^{\gamma}$ has half-life $\log(0.5)/\log(1-\gamma)$ rounds.
The reward uses $\mu_{\text{slow}}$ as the skill baseline:
\begin{align}
\label{eq:lp-raw}
\rho(e) \;&=\; \bigl|\,\bar{r}_A(e) - \mu_{\text{slow},t}(s)\,\bigr|, \\
\label{eq:lp}
r_D^{\text{LP}}(e) \;&=\; \rho(e) - \tfrac{1}{|\mathcal{E}_s|}\textstyle\sum_{e' \in \mathcal{E}_s} \rho(e') ,
\end{align}
where $\mathcal{E}_s$ is the round's environment set for skill $s$.

Two consequences follow. The deviation is unsigned, so an environment the \RA{} always
solves can score as highly as one it never solves when the slow mean sits mid-range,
though neither gives gradient; only a variance-style bonus such as
$\bar{r}_A(e)\bigl(1-\bar{r}_A(e)\bigr)$ would favour mixed outcomes. And $\mu_{\text{fast}}$ is tracked but unused: the
classical gap $|\mu_{\text{fast}} - \mu_{\text{slow}}|$ is logged only as a diagnostic.
The signal also needs a per-skill history before it is meaningful, and cannot separate
deviation caused by environment design from \RA{} drift or sampling noise.

\begin{wrapfigure}{r}{0.38\textwidth}
  \vspace{-1.2em}
  \centering
  \includegraphics[width=\linewidth]{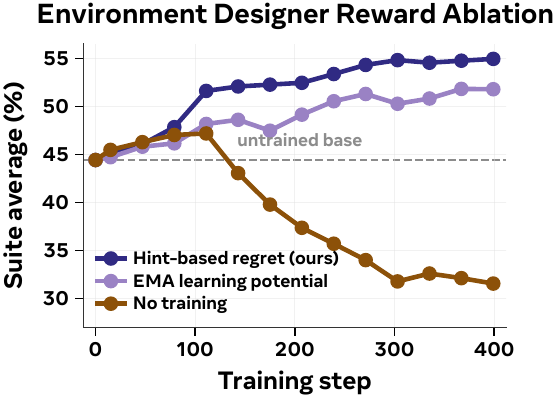}
  \captionof{figure}{\textbf{Both trained-\ED{} reward variants outperform the control without \ED{} training or memory, and regret leads.} Averaged trajectories over GPQA-Diamond, LiveCodeBench-v6, and the four Reasoning-Gym categories (Qwen3-30B-A3B-Instruct-2507, games setting); the dashed line marks the corresponding untrained base. Eight-benchmark checkpoint results are reported in Table~\ref{tab:ablation-breakdown}.}
  \label{fig:abl-ed-reward}
  \vspace{-1.0em}
\end{wrapfigure}
\textbf{Results.} Figure~\ref{fig:abl-ed-reward} shows the trajectories for the two trained-\ED{} rewards and the control without \ED{} training or memory on the six-benchmark trajectory average, while Table~\ref{tab:ablation-breakdown} gives their selected checkpoints on the full eight-benchmark suite. Hint-based regret lifts the eight-benchmark average from $50.2$ to $58.3$ ($+8.1$);
EMA-based learning potential reaches around $70\%$ of that gain ($+5.7$, to $55.9$) and
climbs more slowly, the two signals separating after the first ${\sim}50$ steps. Both stay
far clear of the control without \ED{} training or memory, whose selected checkpoint is $9.7$ points \emph{below} the untrained model.
The comparison between the two trained-\ED{} variants isolates the reward choice, the regret-based blend versus the standalone learning-potential signal; the third comparison does not isolate \ED{} training from memory. The form of Eq.~\ref{eq:lp} helps explain the remaining reward gap: an unsigned
deviation from a per-skill running mean needs history before it means anything, and even
then scores mastered and hopeless environments alike, so it finds the frontier later and
less sharply than a hint gap measured on the current policy. Per-benchmark numbers are in
Table~\ref{tab:ablation-breakdown}.

\FloatBarrier
\section{Scaling Results}
\label{sec:scaling}

\begin{figure*}[t]
  \centering
  \includegraphics[width=0.475\linewidth]{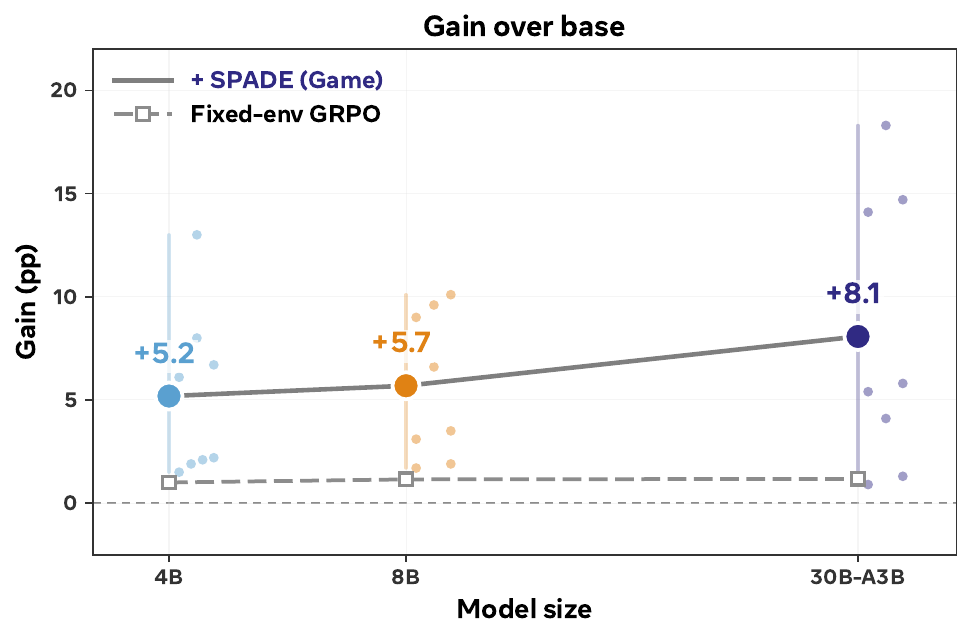}\hfill
  \includegraphics[width=0.510\linewidth]{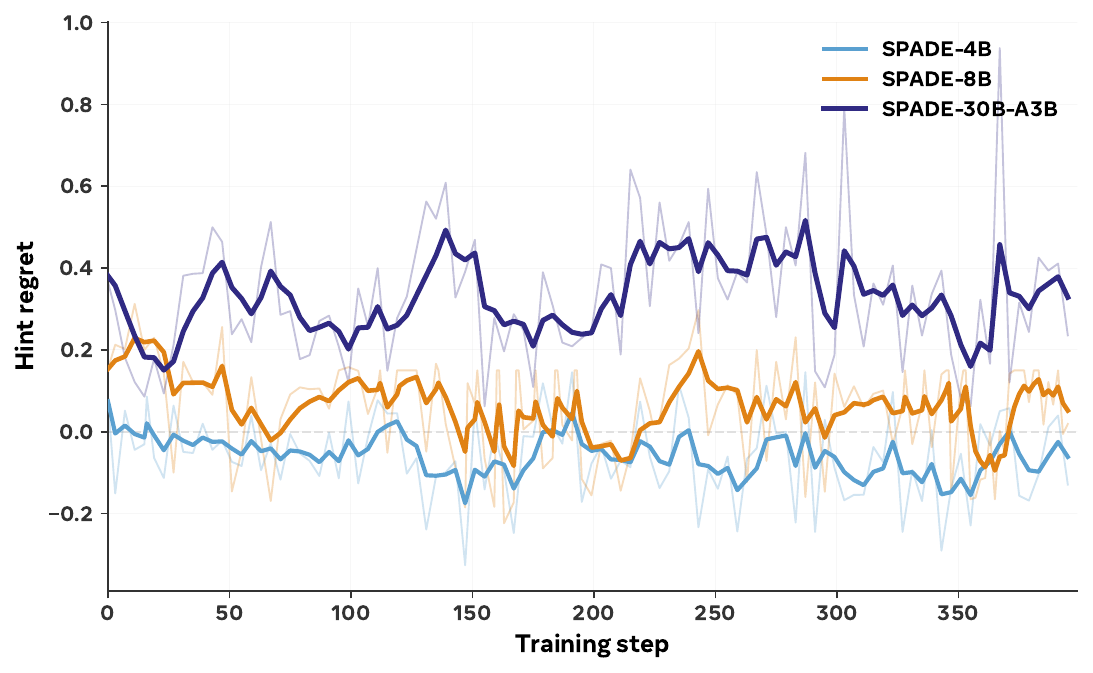}
  \caption{\textbf{\spade{}'s average gain over base grows with model size, from $+5.2$ at 4B to $+8.1$ at 30B-A3B, while matched-budget Fixed-env GRPO stays near $+1.2$.} \textbf{Left:} average gain over each backbone's own base across the eight benchmarks of Table~\ref{tab:games} (large markers), the eight per-benchmark gains beside each mean, Fixed-env GRPO in gray. \textbf{Right:} \ED{} hint-based regret over training (dark: EMA-smoothed; light: per-step). Only the 30B-A3B estimate stays positive; at 4B and 8B it dips below zero for long stretches, where the finite-sample estimate turns negative even though regret is non-negative at the optimum (Section~\ref{sec:method}). Both smaller backbones still gain over base ($+5.2$, $+5.7$), so the environments help even where the signal is noisy.}
  \label{fig:scaling}
\end{figure*}

\begin{wrapfigure}{r}{0.38\textwidth}
  \vspace{-1.2em}
  \centering
  \includegraphics[width=\linewidth]{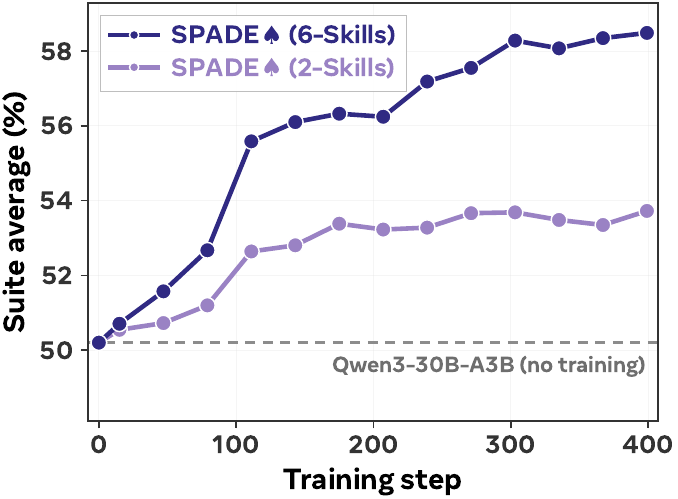}
  \captionof{figure}{\textbf{Curriculum breadth accounts for most of the gain.} Suite average (eight benchmarks of Table~\ref{tab:games}), six-skill vs.\ two-skill curriculum. Per-benchmark panels: Figure~\ref{fig:skill-diversity}.}
  \label{fig:skill-avg}
  \vspace{-1.0em}
\end{wrapfigure}

\textbf{Scaling model size.\,} Under a shared training recipe (Section~\ref{sec:setup-shared}), the average gain over each backbone's base grows from $+5.2$ at 4B and $+5.7$ at 8B to $+8.1$ at 30B-A3B, while Fixed-env GRPO stays near $+1.2$ at every size (Figure~\ref{fig:scaling}; Table~\ref{tab:games}). Fixed-env GRPO supplies a static training signal, whereas the \ED{} keeps adapting the curriculum to the \RA{}; the widening gap is consistent with larger models benefiting more from that adaptivity.

\textbf{Scaling curriculum diversity.\,} The cognitive-skill curriculum is the set of skill categories the \ED{} targets when it generates environments: the six listed in Section~\ref{sec:setup-games}, three active per regeneration in round-robin. We compare the full six-skill curriculum against a two-skill version, with everything else held fixed (Figure~\ref{fig:skill-avg}; per-benchmark panels in Figure~\ref{fig:skill-diversity}, Appendix~\ref{app:results}). The two-skill run also improves, but it captures only about half of the Reasoning-Gym gains and much smaller GPQA-Diamond and LiveCodeBench-v6 gains (best checkpoint on the eight-benchmark suite of Table~\ref{tab:games}: 53.7 vs.\ 58.3). The gains grow with the diversity of the curriculum, not with any single game family.

\section{Discussion}
\label{sec:discussion}

\spade{}'s results support two conclusions. First, training the \ED{} to produce environments adapted to the \RA{}'s current ability outperforms training on fixed environment sources: both the main Fixed-env GRPO pools generated once offline by the corresponding Qwen3 backbones and held fixed throughout training (Table~\ref{tab:games}, Section~\ref{sec:domain-results}), and the offline pool generated by the stronger GPT-5.5 model in the ablation (Table~\ref{tab:ablations}, Section~\ref{sec:ablations}). The widening gap over Fixed-env GRPO with model scale is consistent with an adaptivity advantage. Second, one code-as-environment interface spans two very different settings. In the games setting, reasoning skills the \RA{} develops (planning, constraint satisfaction, strategic thinking) generalize to held-out mathematics, science, and code, well beyond the game format they were learned in. In the tool-use setting, the same recipe lifts multi-step agentic benchmarks most where interaction matters, by $+13.9$ on ACEBench-Agent and $+5.7$ on BFCL~v4 multi-turn at 30B-A3B. The same learnable environment-design role is effective in both settings.

\noindent\textbf{Limitations.\,} (a)~\emph{Complexity bounded by scale and the invisible leash}~\citep{chae2025towards}. The \ED{} cannot produce environments more complex than its base model can express in context, so reachable environment complexity grows with model scale and generation budget. (b)~\emph{A human-designed optimizer.} Both roles are updated by a fixed, human-authored RL algorithm (GRPO); \spade{} does not modify its own learning rule. (c)~\emph{No formal optimality, and fixed-task evaluation.} Hint-based regret is motivated by PAIRED~\citep{dennis2020emergent} but not proven to yield an optimal curriculum, and current benchmarks measure fixed-task performance rather than open-ended reasoning growth.

\noindent\textbf{Future directions.\,} We improve the \ED{} with gradient updates, but co-adaptation between the \ED{} and \RA{} might also come from an \ED{} that improves without weight updates, accumulating and refining design strategies from past attempts in context; whether learned weights or in-context evolution makes the better designer, and at what scale, is an open question. \spade{} also automates one stage of post-training; combining adaptive environment design with systems that automate the remaining stages, from data curation through the learning rule itself, could extend co-adaptive self-play across the whole training pipeline.

\section{Conclusion}
\label{sec:conclusion}

We presented \spade{}, a framework that makes environment design a learnable component of LLM post-training through self-play. Three contributions enable this: (1)~a hint-based regret reward that trains the \ED{} to produce environments at the \RA{}'s learning frontier, grounded in minimax regret theory; (2)~environment design anchored in a pretraining corpus and an environment memory, with a code-as-environment interface that unifies single-turn and multi-turn settings; and (3)~a practical recipe at 30B+ scale on Qwen3 models. In the games setting, \spade{} raises held-out reasoning and code benchmarks by $+8.1$ average points over base at 30B-A3B ($+5.2$ and $+5.7$ at 4B and 8B); these gains remain at late checkpoints of the $400$-step run, and the margin over the strongest fixed-environment baseline grows with scale. The same recipe lifts multi-step tool-use benchmarks by $+13.9$ on ACEBench-Agent and $+5.7$ on BFCL~v4 multi-turn.

\spade{} shows that a single model can design its own training environments and improve from them, a step from fixed benchmarks toward open-ended, continual self-improvement.

\section*{Acknowledgments}
This research was supported by the UW-Amazon Science Gift Hub, UW-Tsukuba Amazon NVIDIA Cross Pacific AI Initiative (XPAI), Sony Research Award, Modal Research Grants, Tinker Research Grants, Character.AI, DoorDash, Open Philanthropy, Coefficient Giving, Toyota Research Institute, and the Schmidt AI2050 Fellows program. This material is based upon work supported by the Defense Advanced Research Projects Agency and the Air Force Research Laboratory, contract number(s): FA8650-23-C-7316. Any opinions, findings and conclusions, or recommendations expressed in this material are those of the author(s) and do not necessarily reflect the views of AFRL or DARPA.

\clearpage   %
\newpage
\bibliographystyle{plainnat}
\bibliography{technical_report}

\clearpage \appendix \part{Appendix} \phantomsection\label{app:toc} \parttoc

\newpage \FloatBarrier

\section[Notation]{Notation\backtotoc}
\label{app:notation}

Table~\ref{tab:notation} consolidates symbols used throughout the paper.

\begin{table}[h]
\centering
\small
\setlength{\tabcolsep}{6pt}
\renewcommand{\arraystretch}{1.15}
\resizebox{.95\textwidth}{!}{
\begin{tabular}{@{}l l@{}}
\toprule
\textbf{Symbol} & \textbf{Description} \\
\midrule
\multicolumn{2}{@{}l}{\emph{MDP and environment}} \\
$\mathcal{S}$ & State space \\
$\mathcal{A}$ & Action space \\
$T(s' \mid s, a)$ & Transition function \\
$R(s, a)$ & Reward function \\
$\rho_0$ & Initial state distribution \\
$s, s', a$ & State, next state, action \\
$\mathcal{E}$ & Space of valid (executable Python) environments \\
$e \in \mathcal{E}$ & A single environment instance \\
$e_b$ & The $b$-th environment in a generation batch \\
\midrule
\multicolumn{2}{@{}l}{\emph{Policy and roles}} \\
$\pi_\theta$ & Shared LLM policy with parameters $\theta$ \\
$\piD$ & Policy in \ED{} role, $\pi_\theta(\cdot \mid \text{role}{=}D)$ \\
$\piA$ & Policy in \RA{} role, $\pi_\theta(\cdot \mid \text{role}{=}A)$ \\
$\text{role}{=}D$ & System-prompt switch selecting \ED{} role \\
$\text{role}{=}A$ & System-prompt switch selecting \RA{} role \\
\midrule
\multicolumn{2}{@{}l}{\emph{Hints and rewards}} \\
$h$ & Privileged hint (strategy / partial solution / key observation) \\
$h_b$ & Hint for the $b$-th environment \\
$y, y_i$ & \RA{} response (rollout) \\
$y'_i$ & \RA{} rollout sampled with the privileged hint in context \\
$r_A(y \mid e)$ & Per-rollout correctness reward (without hint) \\
$r_A(y \mid e, h)$ & Per-rollout correctness reward conditioned on hint \\
$\bar{r}_A(e)$ & Average \RA{} return on $e$ without hints \\
$\bar{r}_A(e \mid h)$ & Average \RA{} return on $e$ with hint $h$ \\
$r_D(e)$ & \ED{} reward: hint-based regret $\bar{r}_A(e \mid h) - \bar{r}_A(e)$ \\
\midrule
\multicolumn{2}{@{}l}{\emph{GRPO and training}} \\
$x$ & Prompt / input sequence (generic GRPO notation) \\
$\mathcal{L}(\theta)$ & GRPO clipped-surrogate training objective \\
$\pi_{\text{old}}, \pi_{\text{ref}}$ & Behavior policy (importance ratio) and KL reference policy \\
$\hat{A}^i$ & Group-normalized advantage (generic); role-specific forms below \\
$\hat{A}_D^b$ & Group-normalized advantage for the $b$-th designed environment \\
$\hat{A}_A^i$ & Group-normalized advantage for the $i$-th agent rollout \\
$\varepsilon_{\text{low}}, \varepsilon_{\text{high}}$ & Asymmetric PPO clipping range (\textsc{DAPO} clip-higher~\citep{yu2025dapo}) \\
$\beta_{\text{KL}}$ & KL-regularization coefficient \\
$B$ & Generation batch size (environments per iteration) \\
$G$ & Group size (\RA{} rollouts per environment) \\
$k$ & Regeneration interval: rollouts per environment set, and the \ED{} update delay \\
$N$ & Total number of training iterations \\
$p_d$ & Domain prompt sampled to seed environment generation \\
$r^i, r^j$ & Scalar reward of the $i$-th / $j$-th response in a GRPO group \\
$M$ & Environment memory: buffer of past environments with regret scores and skill tags \\
$C$ & Pretraining corpus used to ground environment design \\
\bottomrule
\end{tabular}
}
\caption{Symbols used throughout the paper.}
\label{tab:notation}
\end{table}

\newpage

\clearpage

\section[Theoretical Analysis]{Theoretical Analysis\backtotoc}
\label{app:theory}

This section formally analyzes the incentive structure of the hint-based regret reward (hereafter hint-regret) introduced in Section~\ref{sec:hint-regret}. We model the interaction between the \ED{} and \RA{} as a two-player game: the \ED{} selects an environment distribution, while the \RA{} selects a policy. Since LLM inference is stochastic, all payoffs are defined through expected verifier returns.

The main result is an equilibrium characterization of the idealized game. If positive hint-regret remains anywhere, then the current unhinted \RA{} is suboptimal there, and the \ED{} can profitably concentrate on it. Consequently, at any pure Nash equilibrium, the \ED{}'s expected regret is zero, the \RA{} is hint-free optimal on \emph{every} environment in $\mathcal M$, and privileged hints become vacuous.

\subsection{Setup and Assumptions}
\label{app:theory-setup}

\paragraph{Environment class.} Let $\mathcal U$ be a finite universe of candidate environments. The set of \emph{mathematically valid executable} environments is denoted by $\mathcal M\subseteq\mathcal U$. These are environments whose specifications are internally consistent and whose attempted solutions can be evaluated by an external verifier. In the implementation, $\mathcal M$ is approximated by syntax checks, executability checks, and, in the tool-use setting, solvability filtering.

\paragraph{Distribution notation.} For any finite set $\mathcal X$, let $\Delta(\mathcal X)$ denote the set of probability distributions over $\mathcal X$. If $D\in\Delta(\mathcal M)$ (the \ED{}'s environment distribution; we reuse the letter $D$ for this distribution, while the subscript in $u_D$ below labels the \ED{} role) and $\mathcal B\subseteq\mathcal M$, write \[ D(\mathcal B) := \sum_{e\in\mathcal B}D(e). \] A property holds for $D$-almost every environment if the set on which it fails has $D$-mass zero.

\paragraph{Inference protocol.} Fix an inference budget $B_{\text{inf}}$. This budget includes all evaluation-time choices that affect the set of possible \RA{} trajectories: model architecture, context length, prompt format, maximum generation length, sampling rule, temperature, number of rollouts, tool access, memory access, and verifier. Let $\mathcal Y_{B_{\text{inf}}}$ denote the finite trajectory space induced by this budget.

\paragraph{Fine-tuning budget.} Fix a fine-tuning budget $F$. This budget includes all training-time choices that determine which policies the procedure can produce: optimizer, objective, number of updates, and regularization. Let $\Pi$ denote the set of policies attainable under this budget; we take $\Pi$ to be finite, so maxima over $\Pi$ are attained.

\paragraph{Returns with and without hints.} For each environment $e\in\mathcal M$, let $h(e)$ denote the privileged hint emitted by the \ED{}; hints are modeled as a fixed map $e \mapsto h(e)$, so strategic hint choice by the \ED{} is outside this idealization. A \RA{} policy $\pi\in\Pi$ induces a distribution over trajectories in $\mathcal Y_{B_{\text{inf}}}$ given either the unhinted input $e$ or the hinted input $(e,h(e))$. Let \[ r(e,y)\in[0,1] \] denote the verifier return of trajectory $y\in\mathcal Y_{B_{\text{inf}}}$ on environment $e$. Returns are normalized to $[0,1]$ for the analysis; training uses a shaped reward in $[-1,1]$ (Algorithm~\ref{alg:spade}). Here $r(e,y)$ is the per-trajectory return written $r_A(y \mid e)$ in Section~\ref{sec:method}, and $R_\pi(e)$, $R_\pi^h(e)$ below are the population analogues of the empirical means $\bar{r}_A(e)$ and $\bar{r}_A(e \mid h)$.

The unhinted expected return of policy $\pi$ on environment $e$ is \[ R_\pi(e) := \mathbb E_{y\sim \pi(\cdot\mid e)} [ r(e,y) ], \] and the hinted expected return is \[ R_\pi^h(e) := \mathbb E_{y\sim \pi(\cdot\mid e,h(e))} [ r(e,y) ]. \]

\paragraph{Game payoffs.} The \ED{} chooses a distribution $D\in\Delta(\mathcal M)$. Its payoff is the expected hint-based regret \[ u_D(D,\pi) := \mathbb E_{e\sim D} \left[ R_\pi^h(e)-R_\pi(e) \right]. \] The \RA{}'s payoff is its unhinted expected return \[ u_A(D,\pi) := \mathbb E_{e\sim D} [ R_\pi(e) ]. \] Together these define the expected-regret self-play game on $\Delta(\mathcal M)\times\Pi$.

\paragraph{Pure Nash equilibrium.} A pair $(D^\circ,\pi^\circ)\in\Delta(\mathcal M)\times\Pi$ is a pure Nash equilibrium of this game if neither player can improve by unilateral deviation: \[ u_D(D^\circ,\pi^\circ) \ge u_D(D,\pi^\circ) \qquad \forall D\in\Delta(\mathcal M), \] and \[ u_A(D^\circ,\pi^\circ) \ge u_A(D^\circ,\pi) \qquad \forall \pi\in\Pi. \]

\paragraph{Optimal values.} For each environment, define the optimal unhinted value \[ R^\star(e) := \max_{\pi\in\Pi} R_\pi(e). \]

\begin{assumption}[Sound generation]
\label{assump:sound-generation}
The \ED{} only samples mathematically valid executable environments: \[ D\in\Delta(\mathcal M). \] Operationally, this corresponds to syntax checks, executability checks, and, in the tool-use setting, solvability filtering. (Stated for completeness: the game is defined over $\Delta(\mathcal M)$ throughout.)
\end{assumption}

\begin{assumption}[Articulated hints]
\label{assump:admissible-hints}
Conditioning on the hint attains the optimal unhinted value: for every policy $\pi\in\Pi$ and every environment $e\in\mathcal M$, \[ R_\pi^h(e)=R^\star(e). \] In particular $R_\pi^h(e)\ge R_\pi(e)$.
\end{assumption}

\begin{assumption}[Internalizability of hinted behavior]
\label{assump:trajectory-expressivity}
Hinted behavior is attainable without the hint: for every policy $\pi\in\Pi$ there exists a policy $\pi'\in\Pi$ with \[ R_{\pi'}(e)=R_\pi^h(e) \qquad \text{for all } e\in\mathcal M. \]
\end{assumption}

\subsection{Main Results}

Theorem~\ref{thm:nash-expected-regret} shows that every pure Nash equilibrium in this setup drives hint-regret to zero, so that the \RA{} is optimal on every environment in $\mathcal M$. The ideal hint-regret enables this result, because it vanishes on environments the \RA{} already solves and on unsolvable environments, where the hint cannot help, and is positive only where the \RA{} fails unaided but a hint would close the gap. The \ED{} therefore profits exactly by targeting this learnable frontier, and at equilibrium no gap can remain anywhere, since the \ED{} could otherwise deviate to it profitably. This is what distinguishes the regret reward from the plain adversarial-difficulty reward $-\mathbb E_{e\sim D}[R_\pi(e)]$, which provides no mechanism restricting the \ED{}'s support to useful environments.

\label{app:theory-results}

\begin{lemma}[Hint-regret equals hint-free regret]
\label{lem:positive-regret-suboptimality}
Under Assumption~\ref{assump:admissible-hints}, for any \RA{} policy $\pi\in\Pi$ and any environment $e\in\mathcal M$, \[ R_\pi^h(e)-R_\pi(e) = R^\star(e)-R_\pi(e). \] In particular the hint-regret is nonnegative, and strictly positive if and only if $R_\pi(e)<R^\star(e)$.
\end{lemma}

\begin{proof}
By Assumption~\ref{assump:admissible-hints}, $R_\pi^h(e)=R^\star(e)$; subtracting $R_\pi(e)$ gives the identity. The remaining claims follow since $R^\star(e)\ge R_\pi(e)$.
\end{proof}

\begin{theorem}[Nash equilibria imply hint-free optimality on every environment]
\label{thm:nash-expected-regret}
Under Assumptions~\ref{assump:sound-generation}--\ref{assump:trajectory-expressivity}, every pure Nash equilibrium $(D^\circ,\pi^\circ)$ of the expected-regret self-play game satisfies \[ u_D(D^\circ,\pi^\circ)=0. \] Moreover, the pointwise regret vanishes everywhere: for \emph{every} $e\in\mathcal M$, \[ R_{\pi^\circ}(e)=R^\star(e), \qquad R_{\pi^\circ}^h(e)=R_{\pi^\circ}(e). \]
\end{theorem}

Such equilibria exist: Assumptions~\ref{assump:admissible-hints} and~\ref{assump:trajectory-expressivity} guarantee a uniformly optimal policy $\pi^\star$ with $R_{\pi^\star}(e)=R^\star(e)$ for all $e$, and any pair $(D,\pi^\star)$ is a pure Nash equilibrium, so the statement is not vacuous.

\begin{proof}
Let $(D^\circ,\pi^\circ)$ be a pure Nash equilibrium. By Lemma~\ref{lem:positive-regret-suboptimality}, \[ \rho^{\text{reg}}(e):=R_{\pi^\circ}^h(e)-R_{\pi^\circ}(e)=R^\star(e)-R_{\pi^\circ}(e)\ge 0 \qquad \text{for every }e\in\mathcal M. \]

\emph{\RA{} side.} By Assumption~\ref{assump:trajectory-expressivity}, there is $\pi'\in\Pi$ with $R_{\pi'}(e)=R_{\pi^\circ}^h(e)$ for all $e$; by Assumption~\ref{assump:admissible-hints}, $R_{\pi^\circ}^h(e)=R^\star(e)$, so $R_{\pi'}(e)=R^\star(e)\ge R_{\pi^\circ}(e)$ for every $e\in\mathcal M$. Hence \[ u_A(D^\circ,\pi') = \mathbb E_{e\sim D^\circ}[R^\star(e)] \ge \mathbb E_{e\sim D^\circ}[R_{\pi^\circ}(e)] = u_A(D^\circ,\pi^\circ). \] Since $\pi^\circ$ is a best response to $D^\circ$, equality holds, so $\mathbb E_{e\sim D^\circ}[R^\star(e)-R_{\pi^\circ}(e)]=0$. The integrand is nonnegative, hence $\rho^{\text{reg}}(e)=0$ for $D^\circ$-almost every $e$, and therefore \[ u_D(D^\circ,\pi^\circ)=\mathbb E_{e\sim D^\circ}[\rho^{\text{reg}}(e)]=0. \]

\emph{\ED{} side.} Fix any $e\in\mathcal M$. The \ED{} may deviate to the point mass $\delta_e\in\Delta(\mathcal M)$, which yields $u_D(\delta_e,\pi^\circ)=\rho^{\text{reg}}(e)$. Since $(D^\circ,\pi^\circ)$ is a Nash equilibrium, \[ \rho^{\text{reg}}(e) = u_D(\delta_e,\pi^\circ) \le u_D(D^\circ,\pi^\circ) = 0. \] Combined with $\rho^{\text{reg}}(e)\ge 0$ from Lemma~\ref{lem:positive-regret-suboptimality}, this gives $\rho^{\text{reg}}(e)=0$ for \emph{every} $e\in\mathcal M$. Therefore \[ R_{\pi^\circ}(e)=R^\star(e), \qquad R_{\pi^\circ}^h(e)=R_{\pi^\circ}(e) \] for every $e\in\mathcal M$.
\end{proof}

\clearpage
\section[Method and Experiment Details]{Method and Experiment Details\backtotoc}
\label{app:method-details}

\subsection{System Prompts and Templates}
\label{app:prompts}

\paragraph{Environment-generation prompt.} The \ED{} receives the prompt below to generate a new single-turn game as executable Python; the skill name, its description, and a short list of example concepts fill the slots shown in angle brackets. In the canonical corpus-grounded runs the \ED{} instead conditions on a freshly sampled corpus document, and the example-concept slot is unused. The single-turn prompt serves the single-turn setting of Section~\ref{sec:mdp-prelim} (one derivation task, graded per attempt), while the multi-turn prompts demand stateful interaction. The \texttt{\textbackslash boxed\{\}} extraction in the displayed contracts stops at the first closing brace, so generated environments use plain, brace-free answer strings.

\begin{promptcard}{Single-turn environment-generation prompt (\ED{})}
\begin{lstlisting}[style=promptlisting]
<|im_start|>system
You are an expert Python programmer and game designer. You create educational language games for training Large Language Models.<|im_end|>
<|im_start|>user
Create a challenging single-player text-based game as a Python class that tests: <SKILL_NAME> (<SKILL_DESCRIPTION>).

GAME CONCEPT IDEAS (pick one or invent your own):
  <EXAMPLE_GAMES>

RULES:
- Be creative with the game concept - it can be a puzzle, riddle, logic problem, code tracing, word game, optimization task, or any reasoning challenge.
- One task per episode. Player answers with \boxed{answer}.
- Give your class a descriptive, unique name.
- The task must be DETERMINISTICALLY SOLVABLE from the observation alone - all information needed to derive the answer must be explicitly stated. No hidden state the player cannot see.
- The observation should be clearly written so a careful reader can follow the logic to the answer.

INFORMATION HIDING (critical for RL training):
- The observation must NOT contain the answer. The player must DERIVE it.
- For sequence/pattern games: show only partial data, never the target value.
- Never state the pattern rule directly - the player must discover it.
- Wrong-answer feedback must give hints (e.g., higher/lower, which part is wrong), never reveal the full solution.
- The player must use \boxed{answer} format - remind them in the observation.

DIFFICULTY:
- The task MUST require at least 3 distinct reasoning steps - not solvable in one or two arithmetic operations.
- Random guessing should succeed < 1% of the time - use large answer spaces.
- Use randomized parameters large enough that the answer is not obvious.
- Do NOT generate simple formula-substitution or single-operation problems.
- Compute the solution from the generated puzzle; never hardcode it.
  WRONG: return 'recursive'  # same answer regardless of puzzle
  WRONG: return str(random.randint(1,10))  # not derived from the actual task
  RIGHT: generate puzzle first, then compute self._solution from it

INTERFACE CONTRACT:
```python
import random
import re
from typing import Any, Optional, Tuple, Dict, List
from math_verify import parse, verify

def verify_answer(player_answer, solution):
    """Check if answer is correct: string match first, then math equivalence."""
    if str(player_answer).strip() == str(solution).strip():
        return True
    try:
        return verify(parse(player_answer), parse(solution))
    except:
        return False

class DescriptiveNameEnv:
    def __init__(self, max_turns=10, **kwargs):
        ...
        self.reset()

    def reset(self, seed=None) -> Tuple[str, dict]:
        # Generate a new task. Returns (observation_with_instructions, {})
        ...

    def solution(self) -> str:
        # REQUIRED: Return the exact answer string that solves the current task.
        ...

    def step(self, action: str) -> Tuple[str, float, bool, bool, dict]:
        self.turn_count += 1
        truncated = self.turn_count >= self.max_turns
        match = re.search(r'\\boxed\{([^}]+)\}', action)
        if match and verify_answer(match.group(1), self.solution()):
            return ("Correct! ...", 1.0, True, False, {})
        elif truncated:
            return (f"Time's up. The answer was {self.solution()}. [task]", 0.0, False, True, {})
        else:
            return ("Wrong. [specific hint about why]. Try again. [task description]", 0.0, False, False, {})
        # EVERY code path must return (str, float, bool, bool, dict) - no exceptions.

    def close(self): pass
```

EPISODE STRUCTURE (critical for RL training):
1. reset(): Generate ONE task/puzzle for this episode
2. step(): Player tries to solve that SAME task multiple times
3. Correct answer -> terminated=True (episode ends with success)
4. Max turns reached -> truncated=True, reveal the solution so the model learns from failure
5. NEVER generate a new task inside step() - the task is fixed for the whole episode

OBSERVATION REQUIREMENTS:
- Turn 1 (reset): Show instructions + task + remind player to use \boxed{answer} format
- Turn 2+ (step): Show specific feedback on the attempt + the SAME task (no instructions)
- The task/puzzle MUST be visible every turn or the model won't know what to solve
- Feedback must be specific to the attempted answer, not generic ('Wrong. Try again.' is not enough)
- No \boxed{} in action: remind the player of the format + show task (do not terminate)

ROBUSTNESS:
- Never divide by a value that could be zero - regenerate if needed
- Initialize ALL instance variables in __init__ before calling reset()
- Return empty dict {} for info (never strings)
- Every code path in step() must return a 5-tuple (str, float, bool, bool, dict).
- Naming: never store data in self.solution - that shadows the required solution() method. Use self._solution or self._answer instead.
- Brace escaping: in f-strings, {word} evaluates word as a Python expression. To write a literal \boxed{answer} escape it as \boxed{{answer}}. Same rule with .format(): any literal brace must be doubled.

Generate the complete Python code in a ```python block.<|im_end|>
<|im_start|>assistant
\end{lstlisting}
\end{promptcard}

\paragraph{Multi-turn environment-generation prompt.} Because gameplay spans multiple turns, the \ED{} also receives a dedicated prompt for generating interactive, multi-turn games (state that evolves across turns, branching actions), shown below with the same placeholder convention.

\begin{promptcard}{Multi-turn environment-generation prompt (\ED{})}
\begin{lstlisting}[style=promptlisting]
<|im_start|>system
You are an expert Python programmer and game designer specializing in interactive, multi-turn text-based games. You create environments where the player must make sequential decisions across multiple turns, with each action changing the game state.<|im_end|>
<|im_start|>user
Create an interactive multi-turn text-based game as a Python class that tests: <SKILL_NAME> (<SKILL_DESCRIPTION>).

GAME CONCEPT IDEAS for <SKILL_NAME> (pick one or invent similar):
  - <EXAMPLE_GAMES>

WHAT MAKES A GOOD MULTI-TURN GAME:
Think of classic text adventures, board games, or strategy games. The player
navigates a world that changes with every action. Good examples:
  - Grid exploration: move through rooms, find keys to unlock doors, reach the exit
  - Trading: buy low / sell high across rounds with fluctuating prices
  - Survival: manage health, food, tools while exploring - wrong choices kill you
  - Investigation: question suspects, search locations, piece together clues
  - Tower defense: place defenses, then waves arrive - adapt strategy each wave
  - Crafting: gather materials, combine them in the right order to build something

WHAT TO AVOID (common failure modes):
  BAD: A math puzzle where the player guesses the answer and retries on failure.
       That is a single-turn puzzle with retry, NOT a multi-turn game.
  BAD: Increment/decrement a variable until it matches a target.
       That is a linear search, not a game - there are no decisions.
  BAD: The player has only one reasonable action each turn (e.g., always 'allocate').
       If the optimal path is obvious, there is no strategic depth.
  BAD: The game can be solved in 1-3 turns. Too short for multi-turn training.
  BAD: Embedding single-turn math/logic puzzles inside a multi-turn frame.
       E.g., 'go to forest, solve 2+3, collect wood' - the actual decisions are trivial.
  BAD: Using input() to get player answers. NEVER use input(). The action parameter
       to step() IS the player's full response. Parse everything from it.

DESIGN REQUIREMENTS:
- The game world has STATE that changes on every action (positions, inventories,
  health, money, unlocked areas, NPC attitudes, etc.).
- Each turn, the player chooses from 2+ meaningfully different actions.
- Some actions are BETTER than others - wrong choices waste turns or cause harm.
- The game will be played for at most 20 turns. Design difficulty,
  resource budgets, and pacing accordingly. Optimal play should win in
  roughly 10-15 turns.
- Random play should lose most of the time.
- The game must have a clear WIN condition and ideally a LOSE condition too.

ACTION FORMAT:
- The player wraps every action in \boxed{action}. The step() method extracts
  the content inside \boxed{} and interprets it.
- Show available actions clearly in every observation, e.g.:
    Actions: \boxed{go north}, \boxed{go south}, \boxed{pick up key}, \boxed{rest}
- If the player's input has no \boxed{}, return a format reminder. Do NOT terminate.

OBSERVATION FORMAT:
- Each observation must show: (1) current state, (2) result of last action,
  (3) available actions, (4) progress (e.g., 'Turn 3/12 | HP: 7/10 | Items: [key]').
- The initial observation (from reset) includes rules, goal, and starting state.

REWARD:
- Win: reward = 1.0, terminated = True
- Lose: reward = 0.0, terminated = True
- Intermediate turns: reward = 0.0
- Max turns without winning: reward = 0.0, truncated = True

INTERFACE:
```python
import random
import re
from typing import Tuple
from math_verify import parse, verify

def verify_answer(player_answer, solution):
    """Check if answer is correct: string match first, then math equivalence."""
    if str(player_answer).strip() == str(solution).strip():
        return True
    try:
        return verify(parse(player_answer), parse(solution))
    except:
        return False

class DescriptiveNameEnv:
    def __init__(self, max_turns=20, **kwargs):
        self.max_turns = max_turns
        self.turn_count = 0
        # ALL state variables initialized here
        self.reset()

    def reset(self, seed=None) -> Tuple[str, dict]:
        if seed is not None:
            random.seed(seed)
        self.turn_count = 0
        # Randomize the game world
        return observation, {}

    def step(self, action: str) -> Tuple[str, float, bool, bool, dict]:
        self.turn_count += 1
        match = re.search(r'\\boxed\{([^}]*)\}', action)
        if not match:
            return ('Use \\boxed{action} format.', 0.0, False, False, {})
        cmd = match.group(1).strip().lower()
        # Update state, check win/lose
        # Use verify_answer(cmd, solution) for numeric comparisons
        truncated = self.turn_count >= self.max_turns
        return (observation, reward, terminated, truncated, {})

    def close(self): pass
```

REFERENCE EXAMPLE - Trading Game (shows the pattern; yours must be DIFFERENT):
```python
class TradingGameEnv:
    """Buy low, sell high across rounds with fluctuating prices. Reach target gold."""
    def __init__(self, max_turns=12, **kwargs):
        self.max_turns = max_turns
        self.turn_count = 0
        self.gold = 0
        self.stock = 0
        self.price = 0
        self.price_history = []
        self.target_gold = 0
        self.trend = 0  # hidden: +1 rising, -1 falling
        self.reset()

    def reset(self, seed=None):
        if seed is not None:
            random.seed(seed)
        self.turn_count = 0
        self.gold = 100
        self.stock = 0
        self.price = random.randint(8, 15)
        self.price_history = [self.price]
        self.target_gold = 200
        self.trend = random.choice([-1, 1])
        return (f'Welcome to the Trading Game! Reach {self.target_gold} gold to win.\n'
                f'Price: {self.price} | Gold: {self.gold} | Stock: {self.stock}\n'
                f'History: {self.price_history}\n'
                f'Turn 0/{self.max_turns}\n'
                f'Actions: \\boxed{{buy N}}, \\boxed{{sell N}}, \\boxed{{wait}}'), {}

    def step(self, action):
        self.turn_count += 1
        match = re.search(r'\\boxed\{([^}]*)\}', action)
        if not match:
            return ('Use \\boxed{action} format.', 0.0, False, False, {})
        cmd = match.group(1).strip().lower()
        truncated = self.turn_count >= self.max_turns
        msg = ''
        # Parse action
        if cmd.startswith('buy '):
            try:
                n = int(cmd[4:])
                cost = n * self.price
                if cost <= self.gold and n > 0:
                    self.gold -= cost
                    self.stock += n
                    msg = f'Bought {n} at {self.price}.'
                else:
                    msg = f'Cannot buy {n} (need {cost} gold, have {self.gold}).'
            except ValueError:
                msg = 'Invalid number.'
        elif cmd.startswith('sell '):
            try:
                n = int(cmd[5:])
                if n <= self.stock and n > 0:
                    self.gold += n * self.price
                    self.stock -= n
                    msg = f'Sold {n} at {self.price}.'
                else:
                    msg = f'Cannot sell {n} (have {self.stock}).'
            except ValueError:
                msg = 'Invalid number.'
        elif cmd == 'wait':
            msg = 'You wait.'
        else:
            msg = 'Unknown action.'
        # Update price with trend + noise
        if random.random() < 0.3:  # trend reversal
            self.trend *= -1
        self.price = max(1, self.price + self.trend * random.randint(1, 4))
        self.price_history.append(self.price)
        # Check win
        total = self.gold + self.stock * self.price
        if self.gold >= self.target_gold:
            return (f'{msg} Gold: {self.gold}. You win!', 1.0, True, False, {})
        if truncated:
            return (f'{msg} Time up. Gold: {self.gold}, Stock: {self.stock}. Needed {self.target_gold}.', 0.0, False, True, {})
        obs = (f'{msg}\nPrice: {self.price} | Gold: {self.gold} | Stock: {self.stock}\n'
               f'History: {self.price_history[-5:]}\n'
               f'Turn {self.turn_count}/{self.max_turns}\n'
               f'Actions: \\boxed{{buy N}}, \\boxed{{sell N}}, \\boxed{{wait}}')
        return (obs, 0.0, False, False, {})

    def close(self): pass
```

ROBUSTNESS:
- NEVER use input(). The step(action) parameter IS the player's response.
- Never divide by a value that could be zero.
- Initialize ALL instance variables in __init__ before calling reset().
- Return empty dict {} for info (never strings).
- Every code path in step() must return a 5-tuple (str, float, bool, bool, dict).
- Handle unexpected player input gracefully (show valid actions, don't crash).
- Use only the Python standard library (random, re, collections, etc.).
- Brace escaping: in f-strings, literal braces must be doubled {{ }}.
- NEVER put backslashes inside f-string expressions. Use a variable instead:
    BAD:  f'{"\n".join(items)}'
    GOOD: sep = '\n'; f'{sep.join(items)}'  OR  '\n'.join(items)

Generate the complete Python code in a ```python block.
Remember: the game must have BRANCHING DECISIONS where different choices
lead to different outcomes. A game with only one reasonable action per turn is not acceptable.<|im_end|>
<|im_start|>assistant
\end{lstlisting}
\end{promptcard}

\paragraph{Corpus-grounded interactive generation prompt.} The canonical games runs use the corpus-grounded variant below: the sampled document replaces the example-concept list, and the specification demands hidden state, multi-turn structure, and partial reward. The skill, difficulty, and document slots are shown in angle brackets.

\begin{promptcard}{Corpus-grounded multi-turn environment-generation prompt (\ED{})}
\begin{lstlisting}[style=promptlisting]
You are given a reference document. Create an INTERACTIVE, MULTI-TURN Python game
environment grounded in a concept/technique from this document.

<REFERENCE_DOCUMENT>
<DOCUMENT_TEXT>
</REFERENCE_DOCUMENT>
This MUST be a genuine multi-turn INTERACTIVE environment, NOT a one-shot question-and-answer quiz.

HOW THE AGENT ACTS (this matches the training pipeline - follow it exactly):
- Every turn the agent submits ONE action wrapped as \boxed{<action>}, e.g. \boxed{measure node 3},
  \boxed{move north}, \boxed{open valve A}. step() receives that string; extract the action with
  re.search(r'\\boxed\{(.+?)\}', action) and treat the extracted text as an interactive COMMAND.
- The extracted text is a COMMAND / move / query that CHANGES the environment - it is NOT a final
  answer to grade once. step() must UPDATE the environment's state from it and return a NEW
  observation reflecting the changed state.

WHAT MAKES IT MULTI-TURN (required):
- The goal CANNOT be reached in one action. It needs a SEQUENCE of actions - explore to uncover
  hidden information, then act on it; or manipulate state step-by-step toward a target.
- HIDDEN STATE: the agent does NOT see everything at reset(); it must act to reveal/probe state
  across turns (the observation grows/changes as it acts).
- PARTIAL reward for progress toward the goal; terminated=True only when the goal is reached.
- Each turn's observation must (a) reflect the updated state and (b) remind the agent to reply
  with its next action as \boxed{<action>}.

DO NOT (these collapse it back to single-shot QA):
- Do NOT extract \boxed{answer}, grade it once, and terminate. \boxed carries a per-turn COMMAND,
  not a final answer.
- Do NOT state the full problem at reset() and just check one submitted value.
- Do NOT repeat the SAME static task every turn with only 'wrong, try again' feedback.

GROUNDING: the interaction must require understanding a concept/technique from the document
(e.g. if it describes an algorithm, the agent EXECUTES it step-by-step interactively; if it
describes a system, the agent OPERATES it over turns). Never reference the document in the game
text (no 'according to the passage'); write it as a standalone environment.

INTERFACE CONTRACT (interactive; \boxed wraps each turn's COMMAND):
```python
import random, re
from typing import Tuple, Dict

class DescriptiveInteractiveEnv:
    def __init__(self, max_turns=12, **kwargs):
        self.reset()
    def reset(self, seed=None) -> Tuple[str, dict]:
        self.turn_count = 0
        # set up HIDDEN state + a goal needing several actions; randomize by seed.
        # return (observation: situation + AVAILABLE ACTIONS + 'reply with \boxed{<action>}',
        #         but NOT the full solution), {}
        ...
    def step(self, action: str) -> Tuple[str, float, bool, bool, dict]:
        self.turn_count += 1
        truncated = self.turn_count >= self.max_turns
        m = re.search(r'\\boxed\{(.+?)\}', action)
        cmd = (m.group(1) if m else action).strip()      # the per-turn COMMAND
        # 1) PARSE cmd into an operation; 2) UPDATE self state; 3) build a NEW observation
        #    reflecting the updated state; 4) reward = partial progress in [0,1],
        #    terminated=True only when the goal is reached.
        # every code path returns (str, float, bool, bool, dict)
        ...
    def solution(self) -> str:
        # a reference action-sequence (or goal description) that solves it
        ...
    def close(self): pass
```

REWARD SCALE: keep total reward in [0, 1]; partial progress < 1.0, full success = 1.0.

ROBUSTNESS (the most common failures - follow ALL of these or the env won't load):
- BRACE ESCAPING: in f-strings and .format(), {word} evaluates word; to write a LITERAL brace you
  must DOUBLE it. To print \boxed{<action>} inside an f-string, write \boxed{{<action>}}.
- Initialize ALL instance variables in __init__ BEFORE calling reset() (no AttributeError mid-episode).
- Imports: use ONLY the Python standard library plus `math` and `random`. Import everything you use;
  never call an unimported name (no bare `integrate`, no `np.`, no `x.cos()` - use `math.cos(x)`).
- Every step() code path returns the 5-tuple (str, float, bool, bool, dict); info is always a dict {}.
- Never crash on an unrecognized/garbage command - return a helpful observation + reward 0.0 instead.
- Never divide by a value that could be zero; never store data in self.solution (use self._solution).

SELF-VERIFICATION before finalizing:
- Trace TWO different action sequences from reset(seed=0): does the observation CHANGE based on the
  actions? (If it never changes, it is NOT interactive - fix it.)
- Does \boxed carry a per-turn COMMAND that evolves state, rather than a final answer that ends it?
- Is partial reward given for progress, kept in [0,1]?

Target skill: <SKILL_NAME> (<SKILL_DESCRIPTION>). DIFFICULTY: <DIFFICULTY> - the goal
should need <N>+ interaction steps.

Generate the complete Python code in a ```python block.
\end{lstlisting}
\end{promptcard}

\paragraph{Tool-use generation prompt.} The tool-use runs generate environments with the multi-turn variant below (selected by \texttt{SPARE\_MULTITURN\_ENV\_GEN=1} in the released launchers). The skill, difficulty, and example slots fill as in the games prompts; each generation call additionally samples one of ten application domains and one of four task variants (weighted $3{:}2{:}1{:}1$), whose filler blocks follow the main template.

\begin{promptcard}{Multi-turn tool-use environment-generation prompt (\ED{})}
\begin{lstlisting}[style=promptlisting]
Create a MULTI-TURN tool-use environment by subclassing ToolUseBaseEnv.
The environment tests: __SKILL_NAME__ (__SKILL_DESCRIPTION__).

THIS TASK MUST MIRROR THE STRUCTURE OF A REAL MULTI-TURN INTERACTION.
A real user does NOT describe the entire workflow up front. Instead, they
issue ONE atomic instruction, wait for it to be done, then issue the next.
Your env must reproduce this turn-by-turn structure.

--------------------------------------------------------------------------
WHAT YOU IMPLEMENT (the base class handles parsing/dispatch):

  - reset(seed) -- generate state, define self._tools, define
        self._user_messages (list of 3-5 atomic instruction strings),
        self._message_criteria (list of callables: state -> bool),
        self._current_msg = 0, self._expected_answer = "done".
        Return (self._user_messages[0], {})

  - tool_xxx(**kwargs) -> str -- one method per tool. After mutating
        state, EACH tool method must call self._advance_if_done() and
        APPEND its output to the tool result, like:
          base_result = "Moved file.pdf to /temp"
          progress = self._advance_if_done()  # appends next instruction or completion marker
          return base_result + progress

  - solution() -> str -- describe the full multi-turn solution

  - _advance_if_done(self) -> str -- helper you write yourself:
      ```
      def _advance_if_done(self):
          if self._current_msg >= len(self._user_messages):
              return ""
          criterion = self._message_criteria[self._current_msg]
          if criterion(self._state):
              self._current_msg += 1
              if self._current_msg >= len(self._user_messages):
                  return "\n\n[ALL STEPS COMPLETE] Submit <answer>done</answer>."
              next_msg = self._user_messages[self._current_msg]
              return f"\n\n[STEP {self._current_msg} COMPLETE - NEW INSTRUCTION] {next_msg}"
          return ""
      ```

  - Override _check_answer(answer) so it returns True only if
        answer == "done" AND self._current_msg == len(self._user_messages).

--------------------------------------------------------------------------
USER MESSAGE STRUCTURE (THE CORE OF MULTI-TURN):

self._user_messages must be 3-5 ATOMIC instructions, e.g.:
  [
    "Navigate to the document folder.",
    "Move final_report.pdf to a new 'temp' subdirectory.",
    "Now create an 'archive' folder and move all .txt files there.",
    "Finally, list the contents of the archive folder.",
  ]

self._message_criteria are callables that read self._state and return True
when the criterion is met:
  [
    lambda s: s["cwd"] == "/document",
    lambda s: "final_report.pdf" in s["tree"].get("/document/temp", []),
    lambda s: all(f in s["tree"].get("/document/archive", []) for f in s["txt_files"]),
    lambda s: s["last_ls_path"] == "/document/archive",
  ]

--------------------------------------------------------------------------
DOMAIN RULES:

- Each user message is a SHORT, SPECIFIC, ATOMIC instruction (1-2 sentences).
- Do NOT describe the whole workflow in advance. The user reveals
  instructions ONE AT A TIME as previous ones are completed.
- Do NOT mention tool names, file structures, or implementation details
  in user messages. Use natural-language goals only.
- Each instruction's success criterion must be checkable from self._state.
- The actor must complete the CURRENT message before progressing. Trying
  to skip ahead won't reveal future messages.
- Allow 4-8 tool calls per message.

--------------------------------------------------------------------------
CRITICAL - SOLVABILITY & CRITERION CORRECTNESS (the single biggest source of BROKEN games):

A game is BROKEN if a success criterion cannot be satisfied by following its
instruction, or is already satisfied before the agent acts. Broken games waste
training and produce deadlocks. Obey ALL of these:

1. FALSE AT RESET / NOT FREE. Every criterion MUST require the agent to actually
   perform its instruction's action. For a WRITE instruction, the criterion must
   check the CHANGE it makes and be False on the state reset() returns - a
   criterion already True at reset (e.g. `s['inventory']['monitor'] >= 5` when
   monitor starts at 8, or `any('phone' in i for i in s['electronics'])` when a
   phone is already stocked) lets the agent skip the step for free. For a READ /
   VERIFY instruction ("check the line is active", "confirm the balance"), do NOT
   use a criterion that is already True at reset (e.g. `s['line_status']=='active'`
   when it starts 'active') - that advances on ANY tool call without the agent
   doing the read. Instead have the read tool record that it ran (e.g. set
   `s['last_checked']='line_status'`) and make the criterion check THAT flag, so
   the step requires actually calling the read tool.

2. TYPE-CORRECT. The criterion must read self._state with the SAME shape the
   tools write it. If order_history is a list of DICTS, check
   `any(o['item']=='laptop' for o in s['order_history'])`, NOT
   `'laptop' in s['order_history']` - string-in-list-of-dicts is ALWAYS False, so
   the step can NEVER complete and the game deadlocks. Type mismatches are the
   most common unsolvable-game bug.

3. DERIVABLE BY THE AGENT. Every value a criterion requires must be obtainable by
   the agent from EITHER (a) the words of its instruction, OR (b) a tool result it
   can read. NEVER hide an exact date / id / amount / threshold the instruction
   does not state and no tool reveals. BAD: instruction "schedule a visit next
   week" but criterion `s['last_updated'] >= '2024-01-22'` - the agent cannot know
   the threshold, must guess, and usually fails. FIX: state it in the instruction
   ("schedule for 2024-01-22 or later"), OR accept ANY valid action
   (`s['last_updated'] != '2024-01-15'`, i.e. any new date), OR have a tool reveal
   the value. Same for ids: if a criterion needs order '1024', the instruction
   must name it or a tool must return it.

4. EVERY REQUIRED TOOL CAN SUCCEED. For each instruction, the tool that satisfies
   its criterion must have a reachable success path. Trace the field the tool
   looks up - it must be a key some tool (or reset) actually SETS. BAD:
   tool_process_refund finds a refund by `r['refund_id']` but tool_create_refund
   never stores a 'refund_id' key -> process_refund can NEVER succeed. If a tool
   reads an id/field, an earlier tool or reset MUST write that exact key.

5. INSTRUCTION MATCHES CRITERION. Each instruction must describe exactly what its
   criterion checks - no more, no less. Do NOT write "after the refund is
   processed, list inventory" if the criterion only checks that inventory was
   listed (and process_refund is broken/unneeded): the agent chases the
   irrelevant step and stalls.

6. SELF-TRACE BEFORE YOU FINISH (REQUIRED). Mentally run reset(seed=0), then
   execute your own solution() sequence step by step. After EACH solution step the
   matching criterion must flip False->True, IN ORDER, and no later criterion may be
   True yet; at the end self._current_msg must equal len(self._user_messages) and
   _check_answer("done") must be True. If any criterion is True at reset, never
   becomes True, or throws (KeyError/IndexError), the game is BROKEN - fix it
   before returning.

--------------------------------------------------------------------------
DIFFICULTY: __DIFFICULTY__

__DOMAIN_BLOCK__
__TASK_TYPE_BLOCK__
SKILL FOCUS: __SKILL_CATEGORY__
EXAMPLE TASKS (pick one or invent your own):
  __EXAMPLES__

__CORPUS_SECTION__

--------------------------------------------------------------------------
RULES:
- Class MUST inherit from ToolUseBaseEnv.
- Do NOT import ToolUseBaseEnv. Just write: class MyEnv(ToolUseBaseEnv):
- Only standard library imports (random, json, re, math). No custom packages.
- Tools simulated as methods. No real APIs, network, or subprocess.
- Tools must be DETERMINISTIC. All randomness in reset() only.
- 3-8 tools as tool_xxx() methods returning strings.
- Give at least one tool a typed/constrained parameter using 'enum' and 'required' in its
  schema (e.g. status: {'type':'string','enum':['pending','shipped','cancelled']}).
- Include 1-2 DISTRACTOR tools: plausible but wrong for the task, so the agent must select
  the correct one from alternatives.
- Each tool ENDS with `return base_result + self._advance_if_done()`.
- Mutation tools must error gracefully on bad preconditions
  (e.g., mv without target dir -> "Error: target dir not found").
- CRITICAL: tool method parameter names MUST EXACTLY MATCH the names declared
  in self._tools[name]['parameters']['properties']. If schema says
  {'account': {...}}, the method MUST be `def tool_check(self, account=...)`.
  NOT `account_type`, NOT `acct`, NOT a renamed alias. Mismatched names
  cause every tool call to fail with "got an unexpected keyword argument".

--------------------------------------------------------------------------
EXAMPLE STRUCTURE (multi-turn filesystem task):

```python
import random
import json
from typing import Tuple

class FileWorkflowMTEnv(ToolUseBaseEnv):
    def reset(self, seed=None) -> Tuple[str, dict]:
        self.turn_count = 0
        self._call_history = []
        if seed is not None:
            random.seed(seed)
        self._state = {
            'cwd': '/home',
            'tree': {
                '/home': ['document', 'temp.txt'],
                '/home/document': ['report.txt', 'notes.txt', 'log.txt'],
            },
            'last_ls_path': None,
        }
        self._tools = {
            'pwd': {'description': 'Show current dir', 'parameters': {'type':'object','properties':{}}},
            'ls': {'description': 'List a directory', 'parameters': {'type':'object','properties':{'path':{'type':'string'}}, 'required':['path']}},
            'cd': {'description': 'Change dir', 'parameters': {'type':'object','properties':{'path':{'type':'string'}}, 'required':['path']}},
            'mkdir': {'description': 'Create a subdirectory in current dir', 'parameters': {'type':'object','properties':{'name':{'type':'string'}}, 'required':['name']}},
            'mv': {'description': 'Move file from current dir to a destination dir', 'parameters': {'type':'object','properties':{'src':{'type':'string'},'dst':{'type':'string'}}, 'required':['src','dst']}},
        }
        self._user_messages = [
            "Navigate to the document folder.",
            "Create a 'temp' subdirectory and move report.txt into it.",
            "Now list what's in the temp folder.",
        ]
        self._message_criteria = [
            lambda s: s['cwd'] == '/home/document',
            lambda s: 'report.txt' in s['tree'].get('/home/document/temp', []),
            lambda s: s['last_ls_path'] == '/home/document/temp',
        ]
        self._current_msg = 0
        self._expected_answer = "done"
        return (self._user_messages[0], {})

    def _advance_if_done(self) -> str:
        if self._current_msg >= len(self._user_messages):
            return ""
        if self._message_criteria[self._current_msg](self._state):
            self._current_msg += 1
            if self._current_msg >= len(self._user_messages):
                return "\n\n[ALL STEPS COMPLETE] Submit <answer>done</answer>."
            next_msg = self._user_messages[self._current_msg]
            return f"\n\n[STEP {self._current_msg} COMPLETE - NEW INSTRUCTION] {next_msg}"
        return ""

    def _check_answer(self, answer: str) -> bool:
        return answer.strip().lower() == 'done' and self._current_msg >= len(self._user_messages)

    def tool_pwd(self) -> str:
        return self._state['cwd'] + self._advance_if_done()

    def tool_ls(self, path='') -> str:
        if path not in self._state['tree']:
            return f"Error: {path} not found" + self._advance_if_done()
        self._state['last_ls_path'] = path
        out = json.dumps(self._state['tree'][path])
        return out + self._advance_if_done()

    def tool_cd(self, path='') -> str:
        if path not in self._state['tree']:
            return f"Error: {path} not found" + self._advance_if_done()
        self._state['cwd'] = path
        return f"Changed to {path}" + self._advance_if_done()

    def tool_mkdir(self, name='') -> str:
        new_path = f"{self._state['cwd']}/{name}"
        if new_path in self._state['tree']:
            return f"Error: {new_path} exists" + self._advance_if_done()
        self._state['tree'][new_path] = []
        self._state['tree'][self._state['cwd']].append(name)
        return f"Created {new_path}" + self._advance_if_done()

    def tool_mv(self, src='', dst='') -> str:
        cwd_files = self._state['tree'].get(self._state['cwd'], [])
        if src not in cwd_files:
            return f"Error: {src} not in current dir" + self._advance_if_done()
        dst_path = f"{self._state['cwd']}/{dst}"
        if dst_path not in self._state['tree']:
            return f"Error: target {dst} not found" + self._advance_if_done()
        cwd_files.remove(src)
        self._state['tree'][dst_path].append(src)
        return f"Moved {src} to {dst}" + self._advance_if_done()

    def solution(self) -> str:
        return ("1. cd(path='/home/document') 2. mkdir(name='temp') "
                "3. mv(src='report.txt', dst='temp') 4. ls(path='/home/document/temp') "
                "5. <answer>done</answer>")
```

Notice:
- self._user_messages defines the multi-turn STRUCTURE
- Each tool_xxx() ends with self._advance_if_done() so when criterion is met,
  the actor sees "[STEP N COMPLETE - NEW INSTRUCTION] ..." inline in the result
- _check_answer requires "done" AND all messages completed
- Reset returns ONLY the first instruction; the rest are revealed progressively

Generate the complete Python code in a ```python block.
\end{lstlisting}
\end{promptcard}

\begin{promptcard}{Domain and task-variant filler blocks}
\begin{lstlisting}[style=promptlisting]
__DOMAIN_BLOCK__ (one of ten application domains per call):

REAL-WORLD DOMAIN (build the tools, state, and user goal for THIS domain; do NOT
use the reference snippet's academic subject as the domain):
  __DOMAIN__
The task must be a realistic user accomplishing a goal here (search, book, order,
update, cancel, transfer, schedule) - NOT an exam or puzzle about any document.

__TASK_TYPE_BLOCK__ (one per call; execute and read_then_write are the common cases):

TASK VARIANT: execute (happy-path). The user issues atomic instructions; the
agent calls the right tools in order. Each criterion checks the resulting state.
TASK VARIANT: read_then_write. At least ONE instruction must force the agent to
FIRST query state with a read tool, branch on a predicate over the result (e.g.
'for every order over $100', 'all lines past their data cap'), THEN make the
dependent write. Its criterion must check the final state reflects the
predicate-conditioned writes - not merely that one tool was called.
TASK VARIANT: missing_value. At least ONE instruction references a value the user
does NOT state (e.g. 'cancel my most recent order' with no order_id; 'pay the
overdue invoice' with no invoice id). The agent must DISCOVER it via a read tool
before the write. The write tool MUST return an error if given a wrong/guessed
value, so a fabricated argument cannot satisfy the criterion.
TASK VARIANT: with_irrelevant. Include EXACTLY ONE instruction that NONE of the
domain tools can satisfy (e.g. in a retail env, 'what will the weather be
tomorrow?'). Add a tool named 'decline_request' (one required string param
'reason') whose method sets a state flag and then advances; that instruction's
criterion is met ONLY by calling decline_request. Provide NO other tool that
could plausibly fulfill the irrelevant request.
\end{lstlisting}
\end{promptcard}

\paragraph{\RA{} gameplay prompt.} The \RA{} receives each turn's observation wrapped in the template below and answers in \texttt{\textbackslash boxed\{\}} form.

\begin{promptcard}{\RA{} gameplay prompt}
{\footnotesize\begin{verbatim}<|im_start|>user
You are playing a language game. Make valid actions to win.
Observation: <OBSERVATION>
Please reason step by step, and put your final answer within \boxed{}.<|im_end|>
<|im_start|>assistant
\end{verbatim}
}
\end{promptcard}

\subsection{Implementation Details}
\label{app:implementation}
\label{app:lifecycle}

Two auxiliary \ED{} reward terms are disabled in all full-\spade{} runs: a per-environment EMA-based learning-potential bonus~\citep{kanitscheider2021multitaskcurriculumlearningcomplex, zhang2023omni}, enabled only as the standalone reward replacement evaluated in Section~\ref{sec:abl-ed-reward}, and a frontier bonus (extra reward for environments with a large fast-versus-slow EMA gap $|\mu_{\text{fast}}-\mu_{\text{slow}}|$). No separate variance bonus (a $\bar{r}_A(1-\bar{r}_A)$-style reward for mixed outcomes) is used. Failed environment candidates are regenerated up to a fixed number of attempts. Table~\ref{tab:hparams} lists the full training configuration.

\textbf{Pool lifecycle.} At each regeneration the previous environment set is deleted and replaced; the environment memory evicts oldest-first at its $200$-record cap; rejected candidates are persisted for inspection; generation retries up to five attempts per environment.

\textbf{Hint generation detail.} The privileged hint is produced by a separate designer-side call conditioned on the generated environment's source code: the hint writer sees the code, while the \RA{} never does.

\textbf{Validation detail.} Every candidate environment must pass a programmatic smoke test: the class is instantiated, \texttt{reset()} is called, and a few probe actions are stepped through, discarding candidates that fail to parse or crash. The tool-use setting adds the two semantic checks of Section~\ref{sec:setup-tooluse}: a deterministic reset gate that rejects environments in which a success criterion errors on the freshly reset state under every tested seed, and an LLM screen that rejects impossible environments (unreachable, pre-satisfied, or underivable success criteria) while keeping hard-but-feasible ones.

\begin{table}[htbp]
  \centering
  \caption{\textbf{Training hyperparameters.} Shared across the three games-setting \spade{} backbone runs; per-model exceptions appear in parentheses.}
  \label{tab:hparams}
  \small
  \begin{tabular}{@{}ll@{}}
    \toprule
    \textbf{Hyperparameter} & \textbf{Value} \\
    \midrule
    Models & Qwen3-4B-Instruct-2507, Qwen3-8B, Qwen3-30B-A3B-Instruct-2507 \\
    \midrule
    Learning rate & $1\times10^{-6}$ (constant) \\
    Optimizer & Adam, $\beta=(0.9,0.98)$, weight decay $0.1$ \\
    KL penalty $\beta_{\text{KL}}$ & $0$ ($0.005$ for 8B) \\
    Clipping $\varepsilon_{\text{low}}/\varepsilon_{\text{high}}$ & $0.20 / 0.28$ \\
    Truncated importance sampling & yes \\
    Reward normalization & outcome-only, per-game $z$-score \\
    \midrule
    Rollout batch & $24$ \\
    Global batch & $192$ (dynamic) \\
    Group size $G$ & $16$ \\
    Total rollouts & $400$ \\
    Environments per rollout & $24$ ($8\times3$ active skills of $6$, round-robin) \\
    Regeneration interval $k$ & $4$ rollouts \\
    \midrule
    \ED{} temperature & $0.6$ \\
    \ED{} max tokens & $16{,}384$ ($20{,}000$ for 8B) \\
    \RA{} temperature & $0.6$ \\
    \RA{} max tokens & $8{,}192$ \\
    Max turns per episode & $25$ \\
    Max context length & $32{,}768$ ($49{,}152$ later in the 4B run) \\
    \midrule
    \ED{} reward blend & plateau $0.6$, band $[0.4,0.6]$ ($[0.2,0.4]$ later for 4B) \\
                          & $+$ floored regret $0.4$ \\
    Regret scale (normalizer) & $0.15$ \\
    Plateau ramp width & $0.25$ \\
    Delayed \ED{} update & $4$ rollouts \\
    Corpus grounding & $15$k docs ($10$k math, $5$k science) \\
    Environment memory & on \\
    \bottomrule
  \end{tabular}
\end{table}

\subsection{Reproducibility}
\label{app:reproducibility}

The training and evaluation code is released. The configuration files for every run reported here are released with the paper. The evaluation output JSONs and the scripts that produce every figure are included alongside the code.

\clearpage
\subsubsection{Comparability Notes for Table~\ref{tab:tooluse}}
\label{app:comparability}

The reference systems in Table~\ref{tab:tooluse} were trained and evaluated by their own authors on their own harnesses; we reprint their published scores for context and record the protocol differences here. \textbf{BFCL versions.} Our rows and Agent-World~\citep{dong2026agent} report the BFCL~v4 multi-turn suite. AWM~\citep{wang2026agent} and EnvScaler~\citep{song2026envscaler} label the suite v3, which contains the same four multi-turn subcategories; EnvScaler evaluates a frozen 2024-09-22 data snapshot whose repository notes minor differences from later releases. \textbf{$\tau^2$-bench variants.} AWM evaluates $\tau^2$-bench-verified, a fork with task corrections, runs the Telecom domain in solo mode (agent only, no user simulator), and uses a GPT-5.1 user simulator in the other domains; AgentScaler~\citep{fang2025towards} does not state its user simulator. We evaluate the standard $\tau^2$-bench release. \textbf{Aggregation.} Reference-row Avg cells follow the rule stated in the table caption; Agent-World and AWM print their own aggregates (61.8/65.4 and task-weighted 33.5/39.0), which do not equal the unweighted means of the domain scores they report. \textbf{Base models.} Our rows post-train Qwen3-4B-Instruct-2507, Qwen3-8B, and Qwen3-30B-A3B-Instruct-2507; the reference systems post-train their own base checkpoints, with training data and budgets that differ from ours.

\section[Extended Ablations]{Extended Ablations\backtotoc}
\label{app:results}

\begin{figure*}[htbp]
  \centering
  \includegraphics[width=\textwidth]{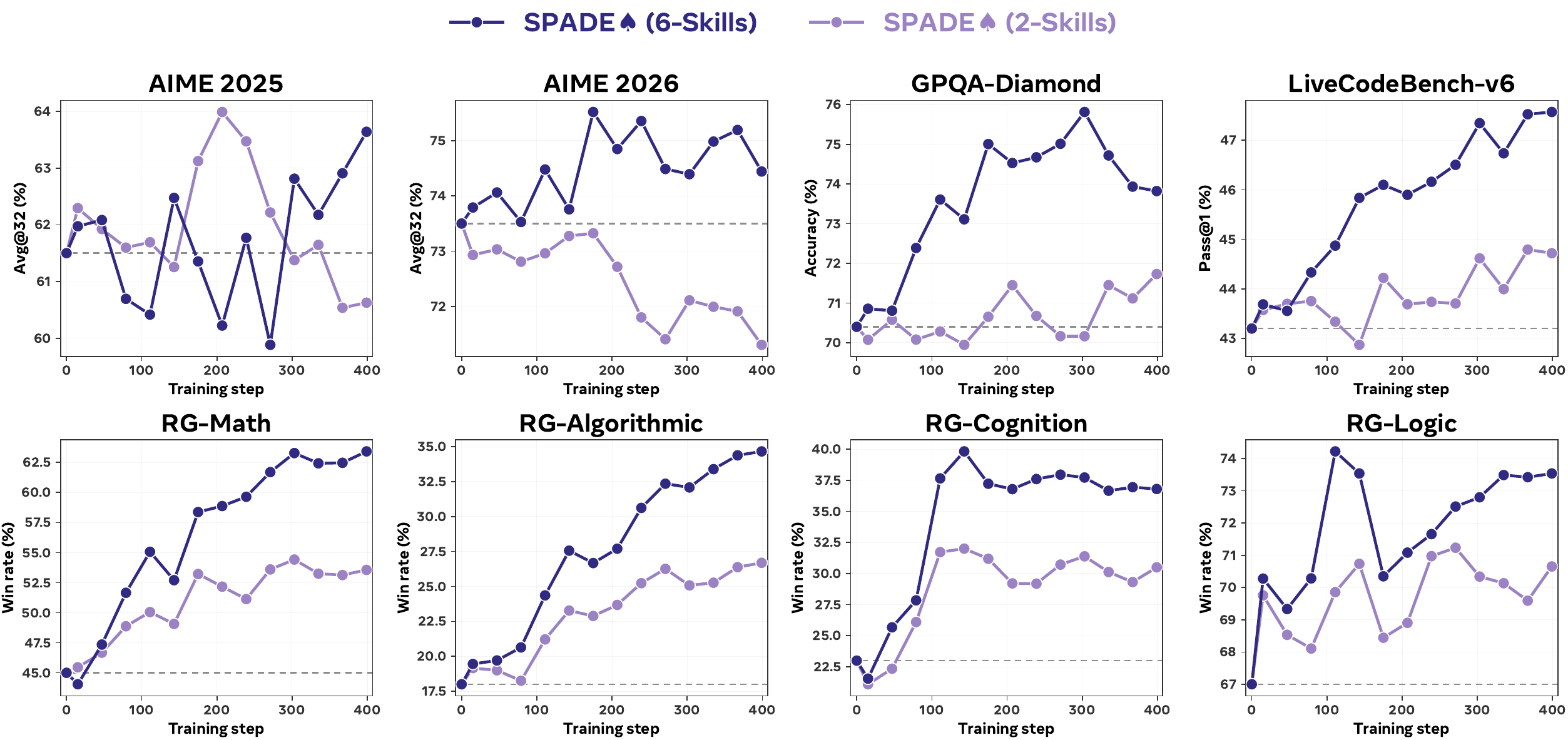}
  \caption{\textbf{The full 6-skill curriculum lifts held-out benchmarks more than the restricted 2-skill variant; curriculum breadth drives the gains.} Qwen3-30B-A3B-Instruct-2507, games setting. Top row: AIME 2025/2026 Avg@32, GPQA-Diamond accuracy, and LiveCodeBench-v6 Pass@1. Bottom row: the four Reasoning-Gym categories; the dashed line marks the untrained base model. Discussed in Section~\ref{sec:scaling}.}
  \label{fig:skill-diversity}
\end{figure*}

\begin{figure*}[htbp]
  \centering
  \includegraphics[width=0.5\textwidth]{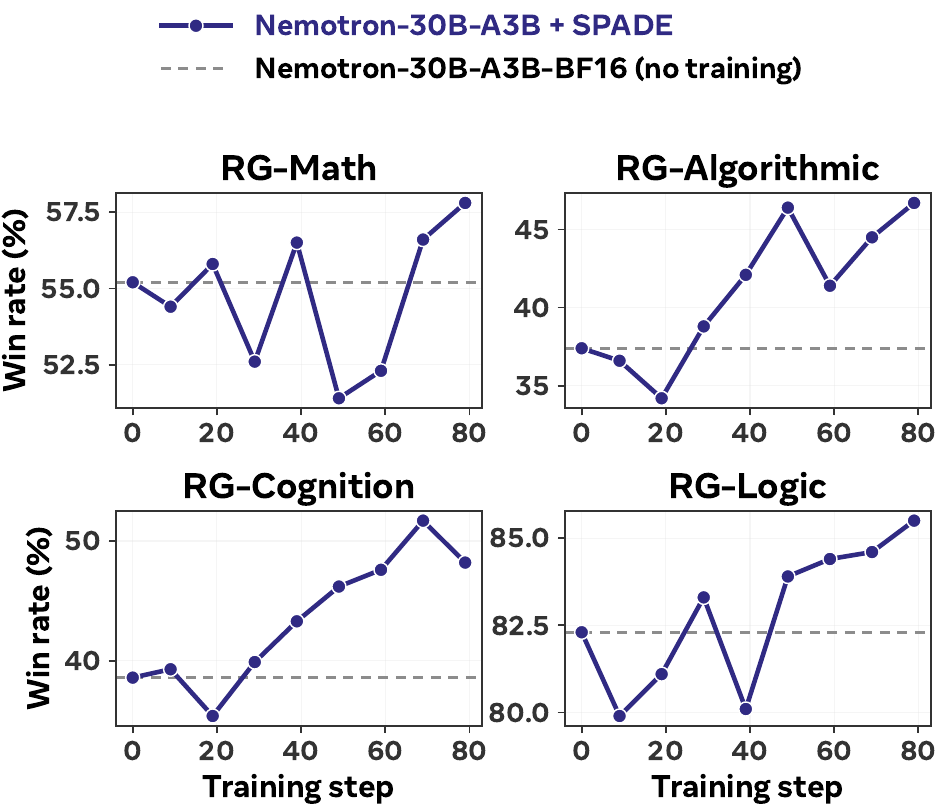}
  \caption{\textbf{\spade{} improves a second backbone family: all four Reasoning-Gym categories end above the untrained Nemotron-30B-A3B-BF16 base (RG-Cognition $+9.6$, RG-Algorithmic $+9.3$, RG-Math $+2.6$, RG-Logic $+3.2$).} Reasoning-Gym win rate across checkpoints; the dashed line marks the untrained base model. Gains arrive after an initial dip early in training, consistent with the \ED{} initially generating environments too hard for the \RA{}, as observed on the Qwen backbones.}
  \label{fig:nemotron-rg}
\end{figure*}

\textbf{Cross-family transfer.\,} As a check on a non-Qwen backbone, we ran the same recipe on Nemotron-30B-A3B-BF16, evaluating only the Reasoning-Gym suite; AIME 2025/2026, GPQA-Diamond, and LiveCodeBench-v6 were not evaluated on this backbone, so this checks that the training dynamics are not Qwen-specific rather than serving as a matched cross-family comparison. All four Reasoning-Gym categories rise above the untrained base, with the largest Nemotron gains in RG-Cognition ($+9.6$) and RG-Algorithmic ($+9.3$; Figure~\ref{fig:nemotron-rg}).

Table~\ref{tab:ablation-breakdown} expands the component-controlled ablations of Table~\ref{tab:ablations} to every games-setting variant reported in this paper, including the skill-diversity run, with each variant's selected checkpoint and its GEM overall win rate. Checkpoints are selected per variant by the best suite average (the mean of the eight benchmark columns); the corresponding evaluation trajectories appear in Figures~\ref{fig:eval-ablation}, \ref{fig:skill-diversity}, and~\ref{fig:abl-ed-reward}.

\begin{table*}[htbp]
  \caption{\textbf{Full ablation breakdown (games setting, Qwen3-30B-A3B-Instruct-2507).} Best checkpoint per variant on the suite average. AIME reports Avg@32; GPQA-D accuracy; LCB-v6 Pass@1; Reasoning-Gym (RG) win rate at \textsc{hard}; GEM the overall win rate across the GEM game suite~\citep{liu2025gem}. Best in \textbf{bold}.}
  \label{tab:ablation-breakdown}
  \centering
  \small
  \setlength{\tabcolsep}{3.5pt}
  \renewcommand{\arraystretch}{1.10}
  \resizebox{\textwidth}{!}{%
  \begin{tabular}{@{}l c|cccccccc|c|>{\columncolor{gray!12}}c@{}}
    \toprule
    \textbf{Setting} & \textbf{Ckpt}
      & AIME'25 & AIME'26 & GPQA-D & LCB-v6 & RG-Math & RG-Algo. & RG-Cog. & RG-Logic
      & GEM & \textbf{Avg} \\
    \midrule
    Qwen3-30B-A3B-Instruct-2507 & --
      & 61.5 & 73.5 & 70.4 & 43.2 & 45.0 & 18.0 & 23.0 & 67.0 & 41.0 & 50.2 \\
    \midrule
    \rowcolor{red!10}
    \spadebrand{}            & 303 & \textbf{62.8} & 74.4 & \textbf{75.8} & \textbf{47.3} & \textbf{63.3} & \textbf{32.1} & \textbf{37.7} & \textbf{72.8} & \textbf{50.2} & \textbf{58.3} \\
    \ED{} w/ learning potential & -- & {62.4} & {74.1} & {74.2} & {46.1} & {57.8} & {27.9} & {33.3} & {71.1} & {47.4} & {55.9} \\
    2-skill curriculum       & 399 & 60.6 & 71.3 & 71.7 & 44.7 & 53.6 & 26.7 & 30.5 & 70.7 & 46.0 & 53.7 \\
    w/o corpus grounding     & 111 & 61.1 & 74.1 & 71.8 & 46.3 & 51.6 & 22.3 & 32.4 & 68.7 & 45.1 & 53.5 \\
    w/o memory               & 111 & 59.3 & \textbf{75.0} & 72.3 & 45.7 & 49.1 & 22.9 & 30.7 & 70.9 & 42.4 & 53.2 \\
    w/o \ED{} training and memory & 271 & 59.4 & 73.5 & 65.8 & 39.1 & 22.5 & 10.0 & 7.6 & 46.0 & 26.4 & 40.5 \\
    Fixed \ED{} (GPT-5.5)             & 175 & 59.9 & 72.8 & 74.2 & 42.6 & 51.2 & 24.3 & 30.7 & 68.0 & 45.6 & 53.0 \\
    \bottomrule
  \end{tabular}}
\end{table*}

\clearpage
\section[Extended Related Work]{Extended Related Work\backtotoc}
\label{app:related}

\subsection{Self-Play for LLMs}
\label{app:related-selfplay}

Self-play has been a cornerstone of AI since \citet{tesauro1995temporal} trained TD-Gammon through self-play in backgammon. AlphaGo~\citep{silver2016mastering} combined Monte Carlo tree search with deep neural networks and self-play to defeat a human Go champion, and AlphaZero~\citep{silver2017mastering,silver2018general} extended the approach to chess, shogi, and Go from tabula rasa with no human data. OpenAI Five~\citep{berner2019dota} and AlphaStar~\citep{vinyals2019grandmaster} scaled multi-agent self-play to Dota~2 and StarCraft~II, while Cicero~\citep{meta2022human} combined language models with strategic reasoning in Diplomacy. \citet{littman1994markov} formalized the Markov-game framework underlying these results, and \citet{irving2018ai} proposed debate as an AI-safety mechanism grounded in self-play. The asymmetric self-play paradigm of \citet{sukhbaatar2017intrinsic}, in which one agent proposes challenges while another solves them, provides the conceptual template for teacher-student environment design.

Translating self-play to LLMs introduces new challenges because rewards are sparse, outputs are discrete sequences, and the ``environment'' is open-ended language. SPIN~\citep{chen2024self} applies self-play fine-tuning by training the LLM to distinguish its own outputs from human demonstrations, converting weak models to strong ones. Self-Rewarding Language Models~\citep{yuan2024self} let the LLM judge its own outputs to generate preference data, iteratively improving both generation and evaluation. SPAG~\citep{cheng2024self} uses adversarial language games (Adversarial Taboo) to incentivize strategic reasoning through a zero-sum objective. ReST$^{\mathrm{EM}}$~\citep{singh2023beyond} alternates between sampling model solutions filtered by a binary correctness reward and supervised fine-tuning on the filtered correct rollouts. Prover-Verifier Games~\citep{kirchner2024prover} train a prover to produce legible solutions that a weaker verifier can check, improving output interpretability. ReMA~\citep{wan2025rema} decomposes reasoning via multi-agent RL into a high-level meta-thinking agent (strategy and decomposition) and a low-level reasoning agent (step-by-step execution), trained jointly.

The most relevant line of recent work uses LLMs as their own source of training data in self-play, with varying degrees of reliance on external data. At one extreme, fully data-free methods bootstrap from minimal seeds: Absolute Zero Reasoner (AZR)~\citep{zhao2025absolute} has a single model (shared parameters) propose and solve code-reasoning triplets (deduction, abduction, induction) with proposer reward $1 - \bar{r}_{\text{solve}}$ computed from solver Monte-Carlo success rates verified by a code executor, bootstrapping from a single identity function. R-Zero~\citep{huang2025r} instead instantiates a Challenger and a Solver as two independently optimized copies of the same base LLM, training the Challenger via GRPO with a self-consistency uncertainty reward; the paper reports that this two-model design outperforms a single-model variant in ablations but plateaus after a few iterations. PopuLoRA~\citep{castanyer2026populora} extends this asymmetric paradigm to co-evolving populations of LoRA-adapter teachers and students with cross-evaluation between sub-populations. In a related but mechanistically distinct line, G-Zero~\citep{huang2026g} performs DPO-based self-distillation over hint-induced preference pairs: a Proposer trained via GRPO emits a hint $h$ for query $q$, the Generator $\pi_G$ produces both an unassisted response $a_{\text{hard}} \sim \pi_G(\cdot \mid q)$ and a hint-conditional response $a_{\text{assisted}} \sim \pi_G(\cdot \mid q, h)$, and DPO is applied with $a_{\text{assisted}}$ as chosen and $a_{\text{hard}}$ as rejected; the Proposer's reward is Hint-$\delta$, the per-token mean log-probability shift induced on $a_{\text{hard}}$ when the hint is prepended to the Generator's context. All supervision is derived from a single model's own outputs under hint vs no-hint contexts, placing the method closer to the self-rewarding and self-distillation line~\citep{chen2024self,yuan2024self,singh2023beyond} than to verifier-grounded agent-environment co-evolution. Tool-R0~\citep{acikgoz2026tool} adapts the same Challenger-Solver paradigm to tool use, training the generator with a solve-rate-band reward plus a semantic alignment score and the solver with soft tool-call matching, similarly finding that separate parameters outperform a shared-weight setup. Self-Questioning Language Models~\citep{chen2025selfquestioning} have a single model propose and solve problems from only a topic prompt, with the proposer rewarded for majority-vote calibration (problems neither always solved nor always missed). Language Self-Play~\citep{kuba2025language} proposes a data-free framework where question generation and answer improvement reinforce each other. PasoDoble~\citep{zhang2025better} pairs an adversarial proposer and solver grounded in a pretraining knowledge base, sustaining improvement beyond R-Zero's plateau. At the other extreme, methods that retain limited seed data or curated evaluation sets include SeRL~\citep{fang2025serl}, which uses a few-shot-prompted question generator with a difficulty-band filter and trains only the solver via Reinforce++ on majority-vote rewards from 500 seed instructions, and \citet{sundaram2026teaching}, who ground the teacher's reward in measured student improvement on curated target problems instead of proxy statistics. Self-Challenging Language Model Agents~\citep{zhou2025self} let agents propose harder task variants for themselves and then attempt to solve them. Autodata~\citep{kulikov2026autodata} generalizes the self-instruct line by casting an LLM agent as a meta-optimizing data scientist that iteratively constructs training data; its Agentic Self-Instruct instantiation selects examples by the separation between a strong and a weak reference solver, targeting learnable examples over merely hard ones, a frontier-targeting principle analogous to \spade{}'s hint-based regret but operating over offline data creation with fixed reference solvers, without the online co-evolution of a gradient-trained generator. Vision-Zero~\citep{wang2025vision} brings gamified multi-agent self-play (Who Is the Spy?) to vision-language models with iterative self-play policy optimization. SPELL~\citep{yang2025spell} co-trains a three-role self-play loop (questioner, responder, verifier) on long documents, using a history-memory curriculum to generate progressively harder questions for evolving long-context capabilities. SPC~\citep{chen2025spc} trains a self-play critic via adversarial games to improve LLM reasoning evaluation. TextArena~\citep{guertler2025textarena} provides an open-source collection of 57+ competitive text-based game environments with a Gym-compatible API for LLM evaluation and self-play training. \citet{sarkar2025training} train LLMs for Among Us-style social deduction with multi-agent RL. Self-Play with Execution Feedback~\citep{dong2024self} uses code execution to verify instruction-following in a self-play loop. SPIRAL~\citep{liu2025spiral} introduces self-play on zero-sum games as a multi-agent, multi-turn RL framework that incentivizes reasoning through game-theoretic competition. SPICE~\citep{liu2025spice} extends self-play to corpus environments via information asymmetry: a Challenger mines documents from a large corpus to generate diverse reasoning tasks with document-grounded answers, while a Reasoner solves them without document access; the corpus prevents hallucination amplification and information-symmetry collapse seen in ungrounded self-play. SWE-RL~\citep{wei2025swe} applies RL to open software evolution tasks, and SSR~\citep{wei2025toward} trains software engineering agents through self-play. Anchored Self-Play~\citep{choi2026anchored} has a single model play a bug-generator and a fixer for code repair with a difficulty-band reward, adding reference-bug mixing and embedding-similarity shaping to keep self-generated bugs on the realistic distribution, preventing drift off it, and GASP~\citep{jana2026gasp} guides asymmetric self-play for coding LLMs with grounding in real data. \citet{chae2025towards} provide a systematic analysis of when LLM self-play succeeds and when it fails, analyzing the ``invisible leash'' where generation quality is bounded by base-model capability. \citet{shafayat2025can} investigate whether large reasoning models can self-train and find that naive self-play degrades performance without careful curriculum control. \citet{liu2026self} analyze when self-synthetic self-play actually improves, showing that sustained evolution requires the synthetic pipeline to guarantee learnable information gain.

Two themes emerge across this body of work. First, all data-free self-play methods generate \emph{tasks} (a problem statement paired with a sparse terminal reward) rather than full environments with state transitions. Second, the generator in every data-free case is either frozen or trained with heuristic proxy rewards (learnability, difficulty, variance) that risk reward hacking and distributional drift. By contrast, \spade{} generates \emph{full MDP environments} (state space, transition function, reward function, and verification code) as executable Python, and trains the \ED{} via RL with hint-based regret that is grounded in measured \RA{} return rather than proxy statistics.

\subsection{Unsupervised Environment Design and Open-Endedness}
\label{app:related-ued}

Curriculum learning~\citep{bengio2009curriculum} established the principle that ordering training examples from easy to hard accelerates learning. Quality-Diversity (QD) algorithms, exemplified by MAP-Elites~\citep{mouret2015illuminating}, illuminate the search space by maintaining an archive of high-performing solutions across a structured behavior space, providing the algorithmic substrate for many open-ended environment-generation systems. Paired Open-Ended Trailblazer (POET)~\citep{wang2019paired} pioneered the paradigm of co-evolving environments alongside agents, using evolutionary search to grow a population of parameterized BipedalWalker terrains together with the policies that solve them, with explicit transfer attempts moving high-performing agents across environments. Unsupervised Environment Design (UED) extends this idea by automatically generating training environments instead of selecting from a fixed pool. PAIRED~\citep{dennis2020emergent} formalized UED via a minimax regret objective: an adversary generates environment parameters to maximize the gap between an antagonist's and a protagonist's returns, producing curricula of increasing complexity while avoiding unsolvable environments. \citet{dennis2020emergent} proved that at Nash equilibrium, the protagonist plays a minimax regret policy, connecting UED to decision theory. Prioritized Level Replay (PLR)~\citep{jiang2021prioritized} replaced the learned adversary with a replay buffer that prioritizes high-regret levels, achieving comparable curriculum quality with lower computational cost. Replay-Guided Adversarial Environment Design~\citep{jiang2021replay} combined the adversarial generator with prioritized replay, using the replay distribution to guide the adversary toward high-regret regions of the environment space. ACCEL~\citep{parker2022evolving} introduced evolutionary mutations of high-regret levels, compounding complexity over training without domain-specific heuristics and recovering minimax regret guarantees at a fraction of the compute required by population-based methods like POET. \citet{mediratta2023stabilizing} stabilized PAIRED-style learned adversaries via entropy regularization and behavioral cloning between protagonist and antagonist, addressing the entropy collapse that destabilizes UED. \citet{rutherford2024no} investigated regret approximations used in PLR and ACCEL, showing that common proxies (positive value loss, maximum Monte Carlo) correlate with success rate rather than true regret, and proposing Sampling For Learnability (SFL) as an improved estimator. \citet{monette2025optimisation} recast UED as a nonconvex-concave optimization over a categorical level distribution, proving convergence to first-order Nash equilibria and generalizing learnability scores to continuous-return settings. CENIE~\citep{teoh2024improving} augmented regret-based UED with a novelty objective that measures how much a candidate environment pushes the agent into unexplored regions of the state-action space, finding that novelty and regret are synergistic rather than competing. DISCOVER~\citep{diaz2026discover} provides a complementary formal-analysis view from goal-conditioned RL: it scores candidate goals via current-policy value estimates balancing achievability, relevance, and novelty, and proves a UCB-style bound on time-to-target-achievability that depends only on the agent's initial distance to the target, independent of task-space volume. The three-axis selection rule parallels \spade{}'s hint-based regret as a frontier-targeting curriculum signal, but DISCOVER operates over a fixed goal pool whereas \spade{} generates the environment distribution itself.

A broader thread of open-endedness research motivates the need for unbounded environment spaces. \citet{hughes2024open} argue that open-endedness, defined as simultaneous novelty and learnability, is essential for superhuman AI, and that systems plateau when the environment parameterization is exhausted. \citet{baker2019emergent} demonstrated emergent tool use from multi-agent autocurricula in hide-and-seek, showing that competitive self-play in rich physics environments produces increasingly sophisticated strategies. OMNI~\citep{zhang2023omni} uses foundation models as a ``Model of Interestingness'' to filter open-ended task proposals, focusing the curriculum on tasks that are both learnable and interesting. OMNI-EPIC~\citep{faldor2024omni} extends this line by representing each environment as executable code generated by a foundation model, with a frozen FM judge of interestingness gating which generated tasks enter the archive; \spade{} shares the code-as-environment representation but trains the \ED{} via RL with hint-based regret, with no frozen interestingness judge. Imagined Autocurricula~\citep{guzel2025imagined} applies PLR inside a learned diffusion world model trained from offline data, demonstrating that UED principles transfer to imagined environments. SIMA 2~\citep{bolton2025sima} scales open-ended self-improvement to embodied 3D worlds: a Gemini-based agent uses foundation models to generate its own tasks and rewards and acquires new skills in previously unseen and even model-generated (Genie 3) environments, though it distributes the process across separate task-setter, agent, and reward models plus a distinct generative world model instead of a single policy that authors its own environments. \citet{goldfeder2026ai} argue for Superhuman Adaptable Intelligence over AGI, emphasizing speed of adaptation over static benchmark performance. PAPRIKA~\citep{tajwar2025training} trains generally curious agents on ten hand-designed task groups, showing that diverse multi-turn training produces transferable exploration strategies.

All classical UED methods operate in \emph{parameterized} environment spaces (maze dimensions, terrain friction, grid layouts) with small, fixed design vocabularies. To our knowledge, \spade{} is the first system to bring \emph{regret-based UED with an RL-trained \ED{} that produces full MDP environments} to LLM post-training: the environment space is unbounded (arbitrary Python code), the \ED{} \emph{is} the LLM rather than a separate adversary network, and hint-based regret provides the co-evolution signal without requiring a trained antagonist.

\subsection{Synthetic Environment Generation}
\label{app:related-synthenv}

A growing body of work uses LLMs to programmatically generate training environments for agentic RL, addressing the bottleneck of hand-curated environment design. Agent World Model (AWM)~\citep{wang2026agent} decomposes environment synthesis into five stages (scenario, task, database, interface, verification) to produce 1,000 SQLite-backed tool-use environments with MCP interfaces, training agents via GRPO with hybrid step-level and task-level rewards. ScaleEnv~\citep{tu2026scaleenv} generates multi-turn tool-use environments by synthesizing API specifications, databases, and verification code from seed domains, scaling to thousands of environments for generalist interactive agent training. TermiGen~\citep{zhu2026termigen} synthesizes high-fidelity terminal environments inside Docker containers with automated trajectory verification, producing diverse command-line tasks for terminal agent training. Nemotron-Terminal~\citep{pi2026data} introduces Terminal-Task-Gen, a two-stage pipeline that combines \emph{dataset adaptation} (transforming math, code, and SWE benchmarks into terminal-formatted prompts) with \emph{synthetic task generation} (seed-based and skill-based, the latter driven by a Skill Taxonomy of nine domains and primitive skills); the resulting Terminal-Corpus is open-sourced, and Qwen3-32B post-trained via SFT reaches 27.4\% on Terminal-Bench 2.0 (from a 3.4\% baseline), surpassing Qwen3-Coder-480B (23.9\%) at a fraction of the parameter count. SkillSynth~\citep{fan2026toward} constructs a scenario-mediated skill graph with 82{,}073 scenarios as nodes, 57{,}214 filtered skills as directed transitions, and 185{,}529 LLM-verified bridges, samples directed paths through the graph as workflow abstractions, and instantiates them via a multi-agent harness with dual execution-based and rubric-based verification (95.7\% oracle pass rate, 3{,}560 verified task instances per run); the explicit objective is maximizing the diversity of execution trajectories over raw task count. Eurekaverse~\citep{liang2024eurekaverse} uses LLMs to generate Python terrain-program curricula (height-field functions) for robot skill learning, producing progressively harder physics environments guided by agent performance feedback, and validates on quadrupedal parkour. LLM-in-Sandbox~\citep{cheng2026llm} places LLMs in sandboxed code-execution environments and shows that agentic intelligence emerges from multi-turn interaction with executable feedback. EvoCUA~\citep{xue2026evocua} evolves computer-use agents by synthesizing scalable GUI interaction experiences, generating training trajectories across diverse desktop applications. \citet{xue2026autonomous} extend this to autonomous continual learning where computer-use agents adapt to new environments through self-generated experience. DreamGym~\citep{chen2025scaling} trains an LLM-based world model on demonstration trajectories from existing agentic environments (WebShop, ALFWorld, WebArena) and uses it to generate abstract-text rollouts that replace expensive real-environment interactions during RL. Simia~\citep{li2025simulating} pushes this direction further by removing real environments entirely: Simia-SFT amplifies small seed trajectories into diverse SFT data via reasoning-model-simulated feedback, and Simia-RL performs PPO/GRPO directly against LLM-generated environment transitions; fine-tuned 7--32B Qwen and Llama models reach 36--59 average on $\tau^2$-bench (Airline+Retail subset) without ever executing real tool code, but the framework still inherits the LLM-as-transition-model hallucination risk and the reduced two-domain $\tau^2$ setup is not directly comparable to the three-domain Avg used elsewhere. \citet{lu2025don} argue that tuning the environment (reward shaping, observation design) can be more effective than tuning the agent, providing a complementary perspective on environment optimization. Exploratory Iteration~\citep{jiang2025bootstrapping} grows a self-improvement task space by sampling informative intermediate solution iterates from previous episodes as starting points for new single-step training tasks, training $K$-step inference-time self-improvement while only training on single-step transitions. Golden Goose~\citep{lu2026golden} synthesizes RLVR tasks from unverifiable internet text by extracting verifiable sub-claims, converting passive corpora into training signal without requiring curated datasets. AutoEnv~\citep{zhang2025autoenv} provides automated environments for measuring cross-environment agent transfer, enabling systematic evaluation of generalization. Synthetic Sandbox~\citep{zhou2026synthetic} generates sandboxed ML-engineering environments for training software agents on realistic development workflows. RLAnything~\citep{wang2026rlanything} proposes a dynamic RL system that jointly optimizes policy and reward model while adapting the task distribution via perturbation of existing tasks. GenEnv~\citep{guo2025genenv} targets difficulty-aligned environment generation with a curriculum reward signal. Agent-World~\citep{dong2026agent} mines real-world environments and toolsets from web content (1{,}978 environments, 19{,}822 tools) and applies a self-evolving training arena that diagnoses agent weaknesses and generates targeted tasks across rounds. AgentScaler~\citep{fang2025towards} clusters over 30{,}000 real APIs into more than 1{,}000 simulated tool-use domains via Louvain community detection on a parameter-similarity graph, materializes each tool as a Python class operating on a per-domain database schema, and trains 4B/8B/30B-A3B agents via two-phase supervised fine-tuning on filtered agent-user trajectories without RL. EnvScaler~\citep{song2026envscaler} programmatically synthesizes tool-interactive environments for multi-turn tool-use agent training, and EnvFactory~\citep{xu2026envfactory} scales executable tool-use environment synthesis paired with a stable RL recipe. From Trainee to Trainer~\citep{chen2026trainee} narrows the adaptivity gap by letting the current RL checkpoint reconfigure its own environment generator from diagnosed failures across training rounds, though it tunes a fixed set of generator parameters and does not emit executable environment code. AutoPlay~\citep{ramrakhya2025scaling} addresses task synthesis for UI agents (mobile and desktop) by first running an MLLM explorer in each app to gather state and functionality information, then having a frozen GPT-4o task-generator condition on those exploration trajectories plus task-guideline prompts to propose grounded tasks; the executor is trained via SFT and GRPO with binary outcome rewards from an MLLM verifier, gaining 20.6 points on AndroidWorld over Qwen2.5-VL-7B. InfiniteWeb~\citep{zhang2026infiniteweb} synthesizes complete static-HTML/localStorage websites alongside tasks and dense-reward JavaScript evaluators via task-centric test-driven development; UI-TARS-1.5-7B post-trained with GRPO on 600 generated tasks improves on OSWorld from 24.5 to 31.4 and on Online-Mind2Web from 23.0 to 28.7, with the largest gains concentrated in easy and medium difficulty buckets. EvoEnv~\citep{shi2026learning} co-trains a single LLM under generator and solver role-conditioning via shared-parameter GRPO, producing reusable Python verifier-and-prompt artifacts whose oracle and scorer remain frozen at training time; the generator reward combines staged validation, solver-relative difficulty calibration targeting 30\% solver accuracy, and an embedding-based novelty bonus, with evaluation on three model families restricted to single-shot reasoning benchmarks (Qwen3-4B-Thinking-2507 improves from 72.4 to 74.8 averaged across eight math, code, and science benchmarks). The environment interface is single-shot scoring (sampler, oracle, renderer, scorer) without state evolution or step-level rewards, so the framework does not extend to multi-turn agentic settings such as tool use or terminal workflows; \spade{} differs in three respects: (i) the \ED{} is rewarded via hint-based regret grounded in \RA{} return rather than validation and novelty signals, (ii) environments are full MDPs with reset/step interfaces that unify single-turn and multi-turn settings, and (iii) evaluation spans games and tool use in addition to reasoning.

These methods produce useful training material, and some incorporate feedback-driven adaptation where the generation distribution shifts in response to agent progress: Eurekaverse evolves environments based on agent performance statistics, and EvoCUA iterates synthesis with an improving policy. However, even in these cases the generator LLM receives no RL gradients; adaptation operates through prompting heuristics or iterative data filtering, without joint optimization. By contrast, \spade{}'s \ED{} is a genuine RL learner, trained via hint-based regret in the same optimization loop as the \RA{}, producing full MDP environments at the \RA{}'s competence frontier.

\subsection{Environment Scaling}
\label{app:related-envscaling}

A growing consensus holds that the next frontier for agentic RL is scaling environments rather than algorithms. \citet{silver2025welcome} articulate this vision most directly, arguing that the field is entering an ``era of experience'' in which the primary bottleneck is generating sufficiently diverse and abundant training environments. \citet{zhang2026scalable} argue in the same vein that scalable, diverse environments are what drive generalizable agents. AgentRL~\citep{zhang2025agentrl} provides a multi-turn, multi-task RL framework with a fully asynchronous generation-training pipeline, cross-policy sampling for improved exploration, and task advantage normalization, training a single agent across five agentic domains to outperform frontier models. SCALER~\citep{xu2026scaler} converts competitive programming problems (from CodeContests) into parameterized single-turn reasoning environments at controlled difficulty levels, demonstrating that scaling the number and diversity of training problems yields consistent improvements in mathematical reasoning. RLVE~\citep{zeng2025rlve} creates 400 hand-engineered verifiable environments for RL training with adaptive difficulty levels and per-environment custom verifiers (mixing rule-based parsing, soft scoring, and code execution) in place of LLM-as-judge. WebScale-RL~\citep{cen2025webscale} builds an automated data pipeline that harvests web content to create RL training environments at pretraining scale, showing that more diverse environments produce better generalization. Endless Terminals~\citep{gandhi2026endless} procedurally generates 3{,}255 verified terminal-use tasks via a four-stage pipeline and demonstrates that vanilla PPO with binary rewards yields substantial gains as the number of training environments scales, with transfer to TerminalBench 2.0. Self-Evolving Curriculum~\citep{chen2025selfevolving} dynamically adjusts the difficulty distribution of training problems based on the learner's current performance, preventing saturation on easy examples. TTCS~\citep{yang2026ttcs} co-evolves a gradient-trained problem synthesizer with a solver via GRPO, using a capability-adaptive reward that targets the solver's learning frontier at test time. \citet{xu2025toward} survey the path toward Large Reasoning Models, organizing the field into three pillars (automated data construction, learning-to-reason via RL, and test-time scaling) and identifying environment scaling as a key open problem. \citet{song2024mind} examine self-improvement capabilities of LLMs and find that naive self-training saturates quickly, motivating the need for adaptive curricula. Embodied Co-Design~\citep{wang2025embodied} provides a taxonomy of co-design approaches where agent morphology and controller evolve jointly, drawing parallels to \spade{}'s co-evolution of environment and agent. Agent Learning via Early Experience~\citep{zhang2025agent} introduces a reward-free training paradigm in which agents propose alternative actions at expert-visited states and learn from the resulting future states (implicit world modeling and self-reflection), bridging imitation learning and full RL across eight agentic environments.

These works independently demonstrate that more environments yield better generalization and that curriculum design matters at least as much as algorithm choice. Several incorporate adaptive mechanisms: RLVE adjusts per-environment difficulty levels, SCALER tracks accuracy for difficulty control, and Self-Evolving Curriculum learns a sampling policy over problem categories. \spade{} combines gradient-trained environment design with the code-as-environment representation, producing an adaptive curriculum of executable MDPs that co-evolves with the \RA{}.

\subsection{Agentic Memory Design and Self-Improving Code Systems}
\label{app:related-memory}

MemRL~\citep{zhang2026memrl} introduces self-evolving agents that learn via runtime reinforcement learning on episodic memory, storing and retrieving past interaction experiences to improve future decision-making without additional gradient updates. ALMA~\citep{xiong2026learning} meta-learns agentic memory designs themselves: a foundation model searches the space of Python memory-architecture programs (an open-ended scaffold over update/retrieve operations), discovering memory designs that improve continual learning at test time. HyperAgents~\citep{zhang2026hyperagents} extends self-referential code evolution by fusing the task-solver and the meta-improver into a single editable program, so the self-modification mechanism is itself modifiable; tasks are externally fixed and the foundation model is frozen. These directions are complementary to \spade{}: they address how agents \emph{use} accumulated experience or \emph{rewrite their own scaffolding} at the meta-level, while \spade{} addresses how to \emph{generate} the training environments that produce that experience in the first place via RL. HyperAgents explicitly leaves co-evolving the task distribution as future work, which \spade{} addresses directly. The approaches are composable: \spade{}'s adaptive environment generation could populate the experience streams that ALMA-style memory systems and MemRL-style runtime RL operate over.

\clearpage
\section[Extended Quantitative Analysis]{Extended Quantitative Analysis\backtotoc}
\label{app:ext-quantitative}

Throughout, environment text is embedded with SBERT (\texttt{all-MiniLM-L6-v2})~\citep{reimers2019sentencebertsentenceembeddingsusing}, and every embedding-based result is replicated under a second, purely lexical embedding (TF-IDF with LSA-128), following the multi-embedding robustness protocol of \citet{abdulhai2026llmsdistortwrittenlanguage}; conclusions are identical under both. Step-level training dynamics use the matched 30B-A3B runs (full \spade{} is the resume-stitched 0--399 step run; ablations end earlier), with curves EMA-smoothed over real logged values, expanding on Section~\ref{sec:training-dynamics}.

\subsection{Games}
\label{app:quant-games}

The games-setting analyses below expand the environment-quality and diversity signals summarized in the main text.

\subsubsection{Environment Quality}
\label{app:qa-quality-metrics}
\label{app:training-dynamics}

We track five quality signals; four are tracked over training (Table~\ref{tab:quality-over-training}) and executability is measured once over the raw generations. \textbf{(i)~Learnability (ground truth):} the fraction of environments in the learnable band (\RA{} win rate in $[0.2, 0.8]$) rises from $0.16$ to $0.31$, and \RA{} win-rate rises $0.30 \to 0.62$: the \ED{} keeps targeting tasks at the frontier of the improving agent (Figure~\ref{fig:difficulty-dynamics}). \textbf{(ii)~Well-posedness:} an LLM rubric over all environments scores $97$ to $98\,\%$ well-posed, flat over training. \textbf{(iii)~Verifiability:} $90$ to $93\,\%$ of environments have an objectively checkable terminal answer, flat to slightly rising. \textbf{(iv)~Executability:} re-executing raw generated code in a sandbox (parse, instantiate, \texttt{reset()}), $84.9\,\%$ runs as emitted and $90.3\,\%$ after stripping a stray Markdown fence; the dominant residual failure is a systematic f-string brace-escaping bug in \texttt{\textbackslash boxed\{\}} action templates ($7.4\,\%$), which the training pipeline's sanitizer repairs before execution ($100\,\%$ post-filter validity). \textbf{(v)~Richness:} environments hold steady at ${\sim}320$ lines of code, ${\sim}13$ hidden state variables, and $8$ to $10$ interaction turns per episode; the \ED{} does not shortcut to trivial programs to inflate its reward. Reward granularity also rises, from $3.7$ to $5.8$ distinct levels per environment ($2.2 \to 4.0$ strictly partial; Figure~\ref{fig:reward-granularity}).

\begin{table}[h]
\centering
\small
\begin{tabular}{lccc}
\toprule
\textbf{Quality signal} & \textbf{Early (0--40)} & \textbf{Mid (150--250)} & \textbf{Late (340--396)} \\
\midrule
Learnable-band fraction (win-rate $\in [0.2, 0.8]$) & 0.16 & 0.16 & \textbf{0.31} \\
\RA{} win-rate & 0.30 & 0.46 & \textbf{0.62} \\
Well-posed (LLM rubric) & 0.98 & 0.97 & 0.97 \\
Verifiable terminal answer (LLM rubric) & 0.90 & 0.91 & 0.93 \\
Interaction depth (turns / episode) & 8.8 & 8.2 & 9.8 \\
Program length (lines of code) & 316 & 333 & 321 \\
Hidden state variables & 13.0 & 13.3 & 13.4 \\
\bottomrule
\end{tabular}
\caption{\textbf{Environment quality over training.} Learnability nearly doubles while well-posedness, verifiability, and structural richness hold constant: the quality gains come from sharper difficulty targeting rather than simpler environments. Win-rate rows use the released per-step evaluation logs.}
\label{tab:quality-over-training}
\end{table}

\begin{figure}[htbp]
\centering
\includegraphics[width=\textwidth]{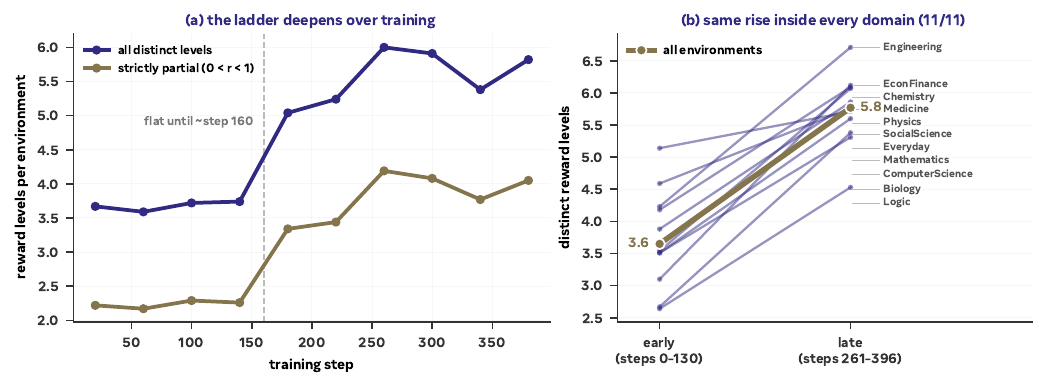}
\caption{\textbf{Reward granularity increases over training.} Left: mean distinct reward levels per environment, including strictly partial levels. Right: early-to-late change in distinct levels overall and by domain (canonical 30B games run).}
\label{fig:reward-granularity}
\end{figure}

\FloatBarrier

Figure~\ref{fig:skill-mastery} gives the per-skill decomposition of the \RA{}'s gains behind Table~\ref{tab:quality-over-training}.

\begin{figure}[htbp]
\centering
\includegraphics[width=\textwidth]{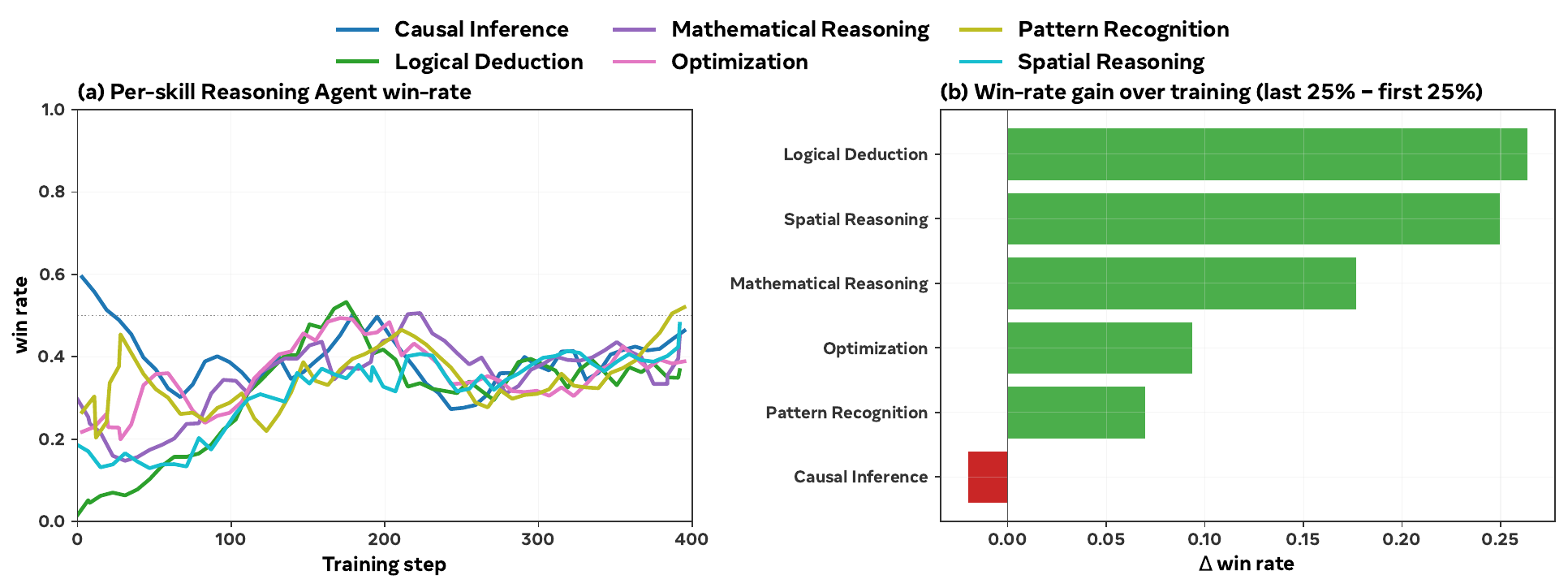}
\caption{\textbf{Per-skill learning} (canonical \spade{}-30B run). \textbf{(a)}~Per-skill \RA{} win rate over training (EMA over ${\sim}26$ logged points/skill). \textbf{(b)}~Win-rate gain (last 25\% minus first 25\%): Logical Deduction and Spatial Reasoning improve most; Causal Inference, which starts high, declines slightly.}
\label{fig:skill-mastery}
\end{figure}

\FloatBarrier

\subsubsection{Environment Diversity}
\label{app:qa-diversity-metrics}
\label{app:qa-diversity-token}

\paragraph{Metric and calibration.} Within each training step's batch we report the Vendi Score~\citep{friedman2023vendiscorediversityevaluation}, the effective number of distinct items in a sample,
\begin{equation}
\mathrm{VS}(K) \;=\; \exp\Big(-\textstyle\sum_{i} \lambda_i \log \lambda_i\Big),
\end{equation}
where $\lambda_i$ are the eigenvalues of the normalized cosine-similarity matrix $K/n$ over SBERT embeddings. The score is calibrated: batches of 24 identical environments score $1.0$, single-domain batches score $14.8$ to $16.5$, and mixed batches drawn across the whole run score $21.3$ (the practical ceiling for $n{=}24$).

\paragraph{Diversity is maintained at the ceiling.} Per-step Vendi is flat for the whole run, $20.8$ (steps 0 to 40) versus $21.0$ (steps 340 to 396), directly at the mixed-population ceiling; mean pairwise cosine distance is likewise flat ($0.94$). Novelty stays saturated (97 of 100 logged steps consist entirely of never-seen initial states), and all 13 domains of an LLM-labeled taxonomy are present from step 0 with $8.4$ per batch of 24 on average (Mathematics 30\,\%, Physics 20\,\%, Medicine 11\,\%, Chemistry 10\,\%, CS 7\,\%, Engineering 6\,\%, other 16\,\%). The distribution is also stationary: after removing seed-document reuse, a linear probe cannot separate early-half from late-half environments (5-fold AUC $0.551 \pm 0.024$).

\paragraph{Corpus grounding drives diversity.} The ablations quantify the matched-step generation contrast of Appendix~\ref{app:ext-qualitative} (Table~\ref{tab:diversity-ablation}, Figure~\ref{fig:diversity-curves}). Diversity remains at the mixed-population reference level when memory alone is removed and when \ED{} training and memory are jointly removed, as long as the corpus is retained ($0.70$ for the joint control); without the corpus, it collapses to $0.04$. These configuration-level contrasts associate corpus grounding with preserved diversity, but do not isolate the independent effect of \ED{} training from memory. The windowed dynamics (Figure~\ref{fig:diversity-curves}) sharpen this further: the no-corpus run starts near-collapsed (Vendi ${\sim}1.6$), RL exploration briefly lifts it to a peak of ${\sim}5.2$ around step 100, and continued optimization then re-collapses it to ${\sim}2$. That run's best evaluation checkpoint (step 111, Table~\ref{tab:ablation-breakdown}) coincides with its diversity peak; once the environment stream re-collapses, downstream evaluation stops improving. The corpus-grounded runs hold the ceiling for their full run lengths under the same optimization pressure (full \spade{} for all 400 steps).

\begin{table}[h]
\centering
\small
\begin{tabular}{lcccc}
\toprule
\textbf{Run} & \textbf{Sample} & \textbf{Vendi/$n$} $\uparrow$ & \textbf{Mean pairwise dist.} $\uparrow$ & \textbf{Reading} \\
\midrule
\spade{} (full) & 3{,}310 & 0.68 & 0.94 & diverse \\
\quad w/o memory & 3{,}746 & 0.69 & 0.94 & diverse \\
\quad w/o \ED{} training, w/o memory & 4{,}929 & \textbf{0.70} & 0.94 & diverse \\
\quad w/o corpus & 866 & \textbf{0.04} & 0.34 & collapsed \\
\bottomrule
\end{tabular}
\caption{\textbf{Corpus ablations collapse environment diversity.} Vendi Score per 100 environments (SBERT embeddings; mean over 20 balanced draws) on the verified-matched 30B-A3B runs: identical backbone, skill set, and \ED{} system prompt, differing in the listed ablations. The identical analysis under TF-IDF/LSA embeddings reproduces the corpus/no-corpus separation ($0.53 / 0.58 / 0.48 / 0.05$). A fifth run with a static pre-generated environment pool (\RA{}-only training, no live \ED{}) is excluded as a different generation protocol; its fixed pool measures Vendi/$n = 0.12$.}
\label{tab:diversity-ablation}
\end{table}

\FloatBarrier

\begin{figure}[t]
    \centering
    \includegraphics[width=0.6\textwidth]{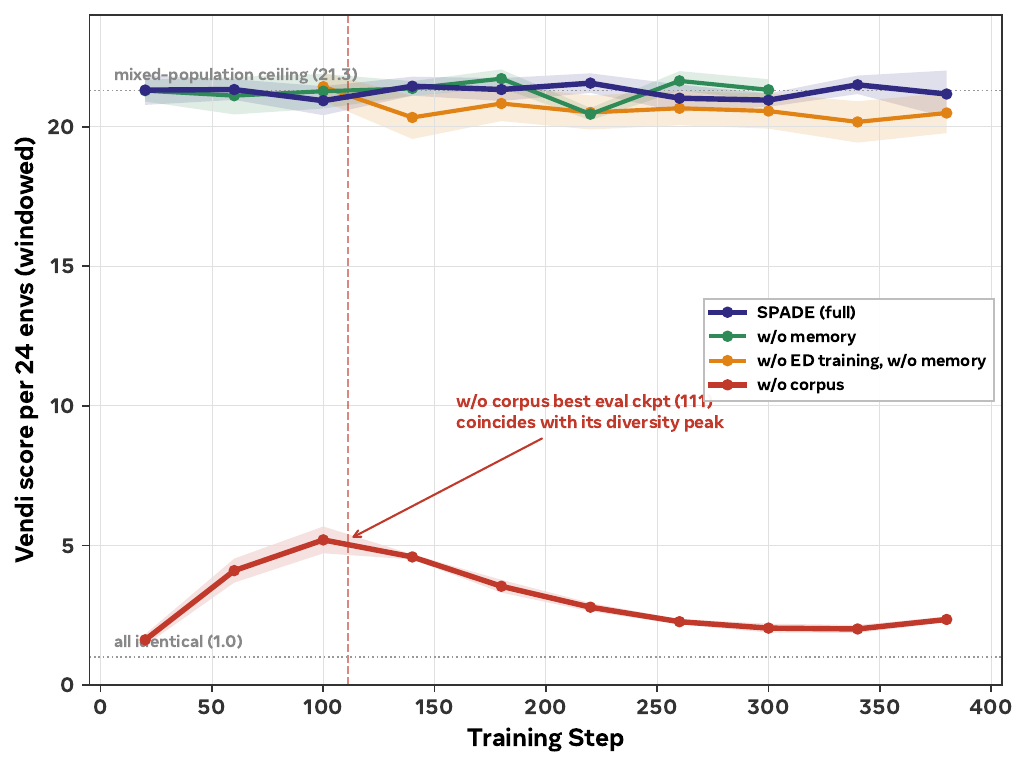}
    \caption{\textbf{Environment-diversity dynamics across the verified-matched runs.} Windowed Vendi score (40-step windows, subsampled to exactly 24 environments; mean $\pm$ s.d.\ over 12 draws). The corpus-grounded variants hold the mixed-population reference level for their full run lengths; the no-corpus run starts near-collapsed, peaks at Vendi ${\sim}5.2$ near step 100, where its best evaluation checkpoint (111) also falls, and re-collapses as optimization continues. Environment counts differ across runs with run length and acceptance rate; the windowed score uses fixed 24-environment subsamples, so counts do not bias the curves.}
    \label{fig:diversity-curves}
\end{figure}

\paragraph{Scope and limitations.} Per-environment scalar rewards are not retained in the public run log, so learnability is measured at the step level (win rate, learnable-band fraction) rather than per environment; the LLM rubric measures absolute difficulty and coherence rather than fit to the current agent; and executability is measured on raw generations, upstream of the pipeline's sanitizer. Diversity numbers use batch size 24 with larger boot-step batches subsampled for comparability.

\FloatBarrier

\clearpage
\section[Extended Qualitative Analysis]{Extended Qualitative Analysis\backtotoc}
\label{app:ext-qualitative}

\begin{findingbox}
\textbf{Summary.} Analyzing all 3{,}310 environments generated by the canonical \spade{}-30B run, each a distinct program, seeded by one of 1{,}513 distinct documents drawn from the 15k-document corpus, four results emerge. \textbf{(1)~Quality:} the share of environments in the learnable band nearly doubles ($0.16 \to 0.31$) while well-posedness, verifiability, and program richness hold constant. \textbf{(2)~Diversity:} semantic diversity stays at the mixed-population ceiling for all 400 training steps, and the stream keeps producing never-seen environments to the end. \textbf{(3)~Mechanism:} the diversity comes from corpus grounding; removing the corpus collapses generation to a single task family (Vendi/$n$: $0.68$ vs.\ $0.04$), whereas removing memory alone or jointly removing \ED{} training and memory leaves it intact. \textbf{(4)~Jointly:} difficulty targeting improves within an unchanged task distribution, so the curriculum sharpens in difficulty while holding its breadth.
\end{findingbox}

We analyze all 3{,}310 environments generated by the canonical \spade{}-30B run at the level of their content: the seed document each was generated from, the generated program, and the opening observation. Every generated environment is a distinct program (3{,}310 of 3{,}310 unique program hashes; 2{,}388 distinct initial states; 1{,}513 distinct seed documents), so all analyses operate on content, never on surface identifiers such as class names (585 distinct; 58\,\% keep the scaffold default despite the prompt's naming instruction). This appendix presents the qualitative evidence; the matching proxy metrics and training dynamics appear in Appendix~\ref{app:ext-quantitative}.

\subsection{Games}
\label{app:qual-games}

For the games setting we give additional privileged-hint examples and trace how a single environment evolves from early to late training.

\subsubsection{Additional Privileged-Hint Examples}
\label{app:hint-examples}

Figure~\ref{fig:hint-examples-appendix} expands the main-text selection to four positive-regret task--hint pairs from the same canonical 30B games run, spanning the regret range from large frontier gaps to a small mastery-regime gap. Each pair is drawn from one logged game record: the task description is a concise summary of its reset observation, while the hint is a verbatim excerpt of the substantive guidance appended to the with-hint arm.

\begin{figure*}[htbp]
\centering
\includegraphics[width=\textwidth]{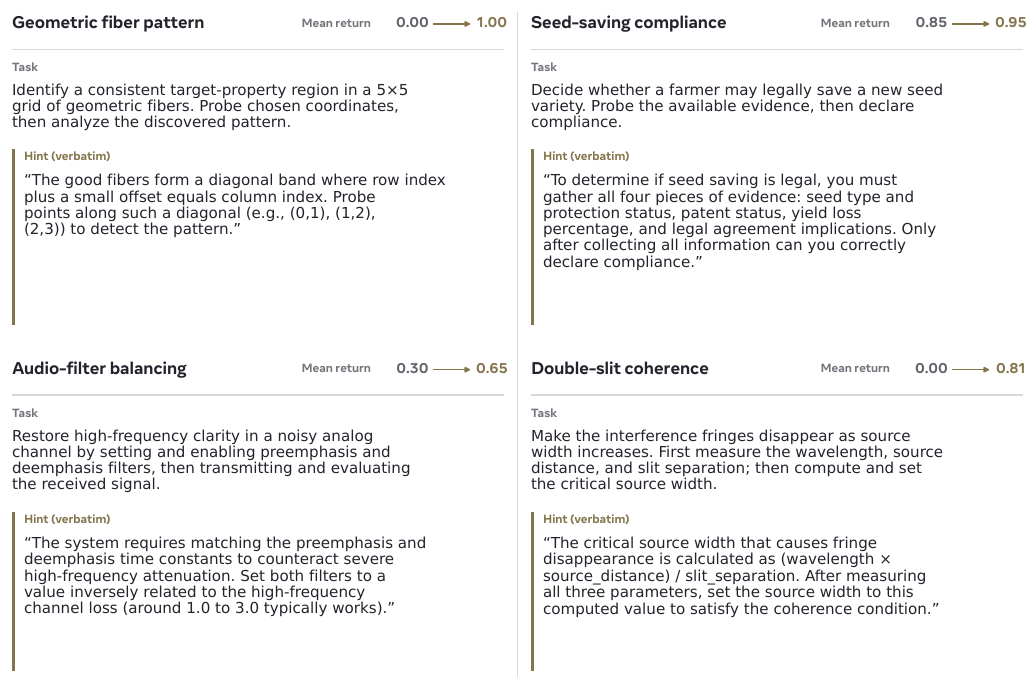}
\caption{\textbf{Task context makes hint utility legible.} Four same-record, positive-regret task--hint pairs from the canonical 30B games run. The task summaries retain the goal, hidden information, and usable interaction while removing generic runtime scaffolding. Hint excerpts omit only the standardized answer-format sentence and runtime wrapper. In each header, the two values report the mean return without hint / with hint.}
\label{fig:hint-examples-appendix}
\end{figure*}
\FloatBarrier

\subsubsection{Environment Evolution over Training}
\label{app:qa-quality-example}

The cards below show one early and one late environment in the exact format produced by the \ED{} (cf.\ the minimal example in Appendix~\ref{app:minimal-env}): a seed document grounds the task; the generated program implements a Gym-style interface with hidden state and a verifiable terminal answer; the opening observation states the goal and action space. The corpus's DCLM slice is filtered web text and includes off-topic documents; the early card's personal-finance seed is one such document. The visible shift over training is toward structured, instrumented ``laboratory'' tasks, while program size, functional hidden state, and verifiability hold steady (about half the early card's raw state variables are write-only narrative fields): the added difficulty comes from state-gated interaction and finer reward grading, not from longer programs.

\begin{envcard}{Training step 0 (early): \texttt{CarOwnershipDisputeEnv} \hfill 370 lines, 29 state variables}
\seedline{I have a car, both me and by grandpa are on the title. As I cannot get loan on my own, my grandpa is my cosigner. My grandpa is almost 90, and chances are he may not be around till the loan is paid}
\begin{minted}{python}
class CarOwnershipDisputeEnv:
    def __init__(self, max_turns=12, seed=None):
        self.max_turns = max_turns
        self.seed = seed
        self.reset(seed)

    def reset(self, seed=None) -> Tuple[str, dict]:
        self.turn_count = 0
        self.rng = random.Random(seed if seed is not None else 42)
        
        # Hidden state: the actual legal situation is not fully visible at start
        self.loan_balance = 12000
        self.car_value = 6000
        self.is_paid_off = False
        self.grandpa_alive = True
        self.grandpa_will_made = self.rng.choice([True, False])  # Hidden: whether will exists
\end{minted}
\vspace{-2pt}
{\footnotesize\sffamily\bfseries Opening observation (what the \RA{} sees):} \obsblock{You are in a tense family situation. Your grandfather is on the title and co-signer of your car loan. You've made consistent payments for 3 years, but still owe \$12,000 on a car worth only \$6,000. Your aunt has threatened to take the car when your grandfather passes. You're unsure if the will exists or what it says. You know your grandfather is still alive and mentally competent. You have not yet consulted a lawyer ...}
\end{envcard}

\begin{envcard}{Training step 384 (late): \texttt{ThermodynamicCycleManipulationLabEnv} \hfill 376 lines, 19 state variables}
\seedline{Subject: physics Problem: A gas has N atoms in volume V\_0 at temperature T\_0. The gas is heated at constant volume up to temperature 3T\_0, then allowed to expand isothermally up to volume}
\begin{minted}{python}
class ThermodynamicCycleManipulationLabEnv:
    def __init__(self, max_turns=12, seed=None):
        self.max_turns = max_turns
        self.seed = seed
        self.reset(seed)

    def reset(self, seed=None) -> Tuple[str, dict]:
        self.turn_count = 0
        self._seed = seed if seed is not None else random.randint(1, 10000)
        random.seed(self._seed)

        # Hidden: thermodynamic system parameters (not directly visible)
        # Represents a monatomic ideal gas undergoing a three-step cycle
        self.n_atoms = random.randint(100, 500)  # N: number of atoms
        self.initial_volume = random.uniform(1.0, 5.0)  # V0: initial volume in m3
        self.initial_temperature = random.uniform(100, 300)  # T0: initial temperature in K
\end{minted}
\vspace{-2pt}
{\footnotesize\sffamily\bfseries Opening observation (what the \RA{} sees):} \obsblock{You enter a high-precision thermodynamics lab. Before you is a transparent cylindrical chamber containing a cloud of gas atoms. The chamber is sealed and connected to external controls:

Available actions: - activate heating (starts constant-volume heating) - initiate expansion (starts isothermal expansion) - begin cooling (starts constant-pressure cooling) - measure current entropy (reveals cumulative ...}
\end{envcard}

The complete source of both environments (369 and 376 lines as generated) is reproduced verbatim in Appendix~\ref{app:env-full-source}.

\subsubsection{Matched-Step Generation Contrast}
\label{app:qa-diversity-example}

The clearest qualitative view of diversity is a matched-step contrast against the no-corpus ablation. At training steps 290 to 312, \spade{} generates environments spanning probability theory, quantum gate tomography, volcanology, hematology, radar signal processing, and hypoelliptic operators, while the no-corpus run generates the \emph{same} rotating-maze navigation task 41 times in a row, varying only the grid layout. Grounding each generation in a sampled corpus document supplies the stream of new task material.

\begin{contrastbox}{Matched-step generation contrast (training steps 290--312)}
{\footnotesize
\textbf{\spade{} (with corpus), one batch at step 296 (24 environments, 24 distinct programs):}\\
\texttt{ProbabilitySpaceExplorationEnv}, \texttt{QuantumGateSetTomographyLabEnv}, \texttt{OptimizationLabEnv}, \texttt{NearDoublesMathLabEnv}, \texttt{RadarDopplerAnalysisLabEnv}, \texttt{MinimaxStrategyLabEnv}, \texttt{HypoellipticOperatorLabEnv}, \texttt{TrigonometricApproximationLabEnv}, plus 16 default-named environments, each a distinct program on a distinct seed document (pattern recognition, genetics, circuit analysis, ...).\\[6pt]
\textbf{w/o corpus, steps 290--312 (41 environments):}\\
\texttt{RotatingMazeEnv}, \texttt{RotatingMazeEnv}, \texttt{RotatingMazeEnv}, \ldots (all 41 generations are the same rotating-obstacle grid-navigation family; only the maze layout and minor rule text vary).}
\par\smallskip\noindent{\footnotesize Counts are accepted generations; the two runs' acceptance rates differ over this window.}
\end{contrastbox}

\subsection{Tool Use}
\label{app:qual-tooluse}

The tool-use \ED{} emits environments in a different register from the games setting: each is a simulated API domain (banking, retail orders, support tickets, telecom accounts, smart-home control) in which a user issues one atomic instruction at a time, and the environment advances only when a programmatic criterion over the hidden state confirms the current instruction is complete. The cards below reproduce one early and one late environment exactly as generated: a code-corpus snippet seeds the generation, the \ED{} builds an unrelated everyday domain on top of it, and each user instruction is paired with a checkable criterion over the backend state (in the early card, coupling the tool call's provenance with its state effect). Criteria check the salient state effect; qualifiers in the instruction text, such as which account pays an invoice, are not always bound, an artifact the validation checks do not catch.

\begin{envcard}{Training step 8 (early): \texttt{BankingMultiTurnEnv} \hfill 286 lines, 8 tools}
\seedline{package com.mayo.client.mayoclientapi.persistence.repository; import com.fasterxml.jackson.databind.ObjectMapper; import com.google.api.core.ApiFuture; import com.google.cloud.Timestamp; import com.google.cloud.firestore.*;}
\begin{minted}[breaklines,fontsize=\scriptsize]{python}
self._user_messages = [
    "List all accounts I have.",
    "Find the checking account and check its current balance.",
    "Transfer $1000 from my savings account to my checking account.",
    "Now pay the unpaid invoice with ID INV2024-003 using my checking account.",
    "Finally, block the card ending in 2468 because it was lost."
]
# Define criteria that must be met for each step to advance
self._message_criteria = [
    lambda s: s['last_searched_account'] is not None and s['last_searched_account'] in s['accounts'],
    lambda s: s['last_query_result'] is not None and s['last_query_result']['account_id'] == 'ACC112233' and s['last_query_result']['balance'] > 0,
    lambda s: s['last_transfer_from'] == 'ACC445566' and s['last_transfer_to'] == 'ACC112233' and s['last_transfer_amount'] == 1000.0,
    lambda s: s['last_invoice_paid'] == 'INV2024-003',
    lambda s: s['last_card_status_change'] == 'blocked'
]
\end{minted}
\vspace{-2pt}
{\footnotesize\sffamily\bfseries Opening observation (what the \RA{} sees):} \obsblock{List all accounts I have.}
\end{envcard}

\begin{envcard}{Training step 392 (late): \texttt{CustomerSupportTicketWorkflowEnv} \hfill 200 lines, 6 tools}
\seedline{import json import os import sqlite3 import openai from pymongo import MongoClient from pymongo.errors import OperationFailure import tiktoken}
\begin{minted}[breaklines,fontsize=\scriptsize]{python}
self._user_messages = [
    "Find all high-priority tickets that are still in 'new' status.",
    "Assign the ticket with ID TICKET-001 to agent_01.",
    "Add a note to ticket TICKET-001 stating 'Customer confirmed issue is reproducible.'",
    "Update the status of ticket TICKET-001 to 'in_progress'.",
    "Finally, confirm that ticket TICKET-001 has been successfully resolved by marking it as 'resolved'."
]
self._message_criteria = [
    lambda s: len([t for t in s['tickets'] if t['priority'] == 'high' and t['status'] == 'new']) == 2,
    lambda s: any(t['id'] == 'TICKET-001' and t['assigned_to'] == 'agent_01' for t in s['tickets']),
    lambda s: any(t['id'] == 'TICKET-001' and any('Customer confirmed issue is reproducible' in n for n in t['notes']) for t in s['tickets']),
    lambda s: any(t['id'] == 'TICKET-001' and t['status'] == 'in_progress' for t in s['tickets']),
    lambda s: any(t['id'] == 'TICKET-001' and t['status'] == 'resolved' for t in s['tickets'])
]
\end{minted}
\vspace{-2pt}
{\footnotesize\sffamily\bfseries Opening observation (what the \RA{} sees):} \obsblock{Find all high-priority tickets that are still in 'new' status.}
\end{envcard}

Across the full 30B run the structural profile of generated environments is stationary, mirroring the distributional stationarity of the games setting (Appendix~\ref{app:qa-diversity-metrics}): mean program length moves from $247$ to $226$ lines between the first and last training band, tools per environment from $7.6$ to $5.9$, while instructions per environment ($4.7$) and the logical complexity of the per-step criteria ($1.9$ to $2.0$ conditions per criterion) stay flat, and roughly one environment in five ships a guarded failure path (locked cards, unavailable agents, denied requests) that the \RA{} must detect and recover from. What changes over training is not the scaffold but the instantiations rotated through it: the \ED{} holds a stable inventory of everyday API domains while regenerating fresh states, identifiers, and instruction sequences each round, exactly the regime the corpus-grounding and memory components are designed to sustain.

\subsection{Generated Environment Gallery}
\label{app:gallery}

This gallery collects representative environments in the exact format the \ED{} emits, starting from a minimal multi-turn example.

\subsubsection{A Minimal Generated Environment}
\label{app:minimal-env}

Listing~\ref{lst:minimal-env} (reproduced from Listing~\ref{lst:gym-example}) shows a minimal multi-turn environment in the format produced by \spade{}'s \ED{}: a single Python class exposing the Gym-style \texttt{reset()}/\texttt{step()} interface described in Section~\ref{sec:mdp-prelim}, with internal state, per-step feedback, and a verifiable terminal reward. This one worked example stands in for the gallery: the same interface and training pipeline also carry single-turn answer-grading tasks and multi-turn tool-use interactions.

\begin{listing}[!ht]
\begin{minted}{python}
import random

class WordleEnv:
    WORDS = ["spade", "trace", "lemon", "graph"]    # truncated

    def reset(self, seed=None):                     # initial state
        self.target = random.Random(seed).choice(self.WORDS)
        self.turns_left = 6
        return "Guess a 5-letter word in 6 tries.", {}

    def step(self, guess):                          # (s', r, term, trunc, info)
        self.turns_left -= 1
        left = [t for g, t in zip(guess, self.target) if g != t]
        fb = ""
        for g, t in zip(guess, self.target):
            if g == t: fb += "G"
            elif g in left: fb += "Y"; left.remove(g)
            else: fb += "-"
        if fb == "GGGGG":
            return fb, 1.0, True, False, {}
        return fb, 0.0, False, self.turns_left == 0, {}
\end{minted}
\caption{\textbf{A minimal \spade{}-generated environment.} A Wordle-flavored multi-turn deduction game implemented as a single Python class with Gym-style \texttt{reset()}/\texttt{step()} interface. The episode is stateful (\texttt{self.target}, \texttt{self.turns\_left}), ends either on a correct guess (terminated, reward $1$) or at the $6$-turn limit (truncated, reward $0$).}
\label{lst:minimal-env}
\end{listing}

\subsubsection{Full Source of Two Exemplar Generated Environments}
\label{app:env-full-source}

The two environments carded in Appendix~\ref{app:qa-quality-example} are reproduced here in full, exactly as emitted by the \ED{} (lightly transliterated to ASCII: typographic quotes, dashes, arrows, and sub-/superscripts). Each is a complete, self-contained program: hidden state initialized in \texttt{reset()}, a multi-turn \texttt{step()} loop with intermediate feedback, and a verifiable terminal reward. Episode return is the final-step reward, clipped to $[-1,1]$ and $0$ if the episode never terminates; the intermediate per-step rewards visible below are progress feedback, not summed.

\paragraph{Training step 0: \texttt{CarOwnershipDisputeEnv} (370 lines).}
\begin{minted}{python}
import random
import re
from typing import Tuple, Dict

class CarOwnershipDisputeEnv:
    def __init__(self, max_turns=12, seed=None):
        self.max_turns = max_turns
        self.seed = seed
        self.reset(seed)

    def reset(self, seed=None) -> Tuple[str, dict]:
        self.turn_count = 0
        self.rng = random.Random(seed if seed is not None else 42)
        
        # Hidden state: the actual legal situation is not fully visible at start
        self.loan_balance = 12000
        self.car_value = 6000
        self.is_paid_off = False
        self.grandpa_alive = True
        self.grandpa_will_made = self.rng.choice([True, False])  # Hidden: whether will exists
        self.grandpa_will_content = self.rng.choice(['gives car to OP', 'gives car to aunt', 'no mention']) if self.grandpa_will_made else None
        self.grandpa_owns_title = True  # Co-owner
        self.op_owns_title = True
        self.loan_coborrower = True
        self.loan_cosigner = True
        self.loan_lender_has_lien = True
        self.title_in_name = True  # Title exists in both names
        self.loan_docs_contain_death_clause = self.rng.choice([True, False])
        self.aunt_threatening = self.rng.choice([True, False])
        self.op_paying_payments = True
        self.op_has_good_payment_history = True
        self.op_credit_score = 580  # Poor
        self.op_has_refinancing_attempt = False
        self.op_has_consulted_lawyer = False
        self.grandpa_is_competent = True
        self.grandpa_will_likely_sign = self.rng.choice([True, False])
        self.grandpa_has_signed_transfer = False
        self.grandpa_has_refinanced = False
        self.grandpa_has_gifted_title = False
        self.legal_right_to_car = True  # OP has right due to ownership and payments

        # Goal: Ensure OP keeps the car after Grandpa's death
        # This requires multiple steps: verify legal rights, assess will, consider refinancing, or get title transfer.
        
        # Initial observation - hidden state, partial information
        obs = (
            "You are in a tense family situation. Your grandfather is on the title and co-signer of your car loan. "
            "You've made consistent payments for 3 years, but still owe $12,000 on a car worth only $6,000. "
            "Your aunt has threatened to take the car when your grandfather passes. "
            "You're unsure if the will exists or what it says. "
            "You know your grandfather is still alive and mentally competent. "
            "You have not yet consulted a lawyer or attempted refinancing. "
            "The loan documents may have a clause about death of a co-signer.\n\n"
            "Available actions:\n"
            "  - check loan documents\n"
            "  - talk to grandpa about the will\n"
            "  - talk to grandpa about transferring title\n"
            "  - research refinancing options\n"
            "  - consult a lawyer\n"
            "  - check car title status\n"
            "  - review state laws on joint ownership\n\n"
            "Reply with your next action as \\boxed{<action>}."
        )
        return obs, {}

    def step(self, action: str) -> Tuple[str, float, bool, bool, dict]:
        self.turn_count += 1
        truncated = self.turn_count >= self.max_turns
        m = re.search(r'\\boxed\{(.+?)\}', action)
        cmd = (m.group(1) if m else action).strip().lower()

        reward = 0.0
        terminated = False
        info = {}

        # Parse and execute command
        if cmd == "check loan documents":
            if self.loan_docs_contain_death_clause:
                obs = (
                    "You reviewed the loan documents. There is a clause: 'If either borrower dies, the remaining balance becomes due immediately.'\n"
                    "This means if your grandfather dies, you may have to pay off the full $12,000 immediately unless refinanced.\n\n"
                    "Next steps:\n"
                    "  - talk to grandpa about transferring title\n"
                    "  - research refinancing options\n"
                    "  - consult a lawyer\n"
                    "Reply with \\boxed{<action>}."
                )
                reward = 0.2
            else:
                obs = (
                    "You reviewed the loan documents. There is no death clause. The loan remains active and payable by you regardless of your grandfather's passing.\n"
                    "This is good news - you are not immediately liable for the full balance.\n\n"
                    "Next steps:\n"
                    "  - talk to grandpa about the will\n"
                    "  - check car title status\n"
                    "  - talk to grandpa about transferring title\n"
                    "Reply with \\boxed{<action>}."
                )
                reward = 0.2

        elif cmd == "talk to grandpa about the will":
            if not self.grandpa_alive:
                obs = (
                    "Your grandfather has passed. You cannot talk to him now.\n"
                    "You must now rely on the will or state law.\n\n"
                    "Next steps:\n"
                    "  - check if a will exists\n"
                    "  - consult a lawyer\n"
                    "  - verify car title status\n"
                    "Reply with \\boxed{<action>}."
                )
                reward = 0.1
            elif self.grandpa_will_made:
                obs = (
                    "You talked to your grandfather. He confirms he has a will. He says he intends to leave the car to you.\n"
                    "He will sign the will soon.\n\n"
                    "Next steps:\n"
                    "  - verify the will's content\n"
                    "  - talk to grandpa about transferring title\n"
                    "  - consult a lawyer\n"
                    "Reply with \\boxed{<action>}."
                )
                reward = 0.3
            else:
                obs = (
                    "You talked to your grandfather. He says he hasn't made a will yet. He is open to writing one.\n"
                    "He is willing to transfer the car to you before he dies.\n\n"
                    "Next steps:\n"
                    "  - talk to grandpa about transferring title\n"
                    "  - consult a lawyer\n"
                    "  - research title transfer process\n"
                    "Reply with \\boxed{<action>}."
                )
                reward = 0.2

        elif cmd == "talk to grandpa about transferring title":
            if not self.grandpa_alive:
                obs = (
                    "Your grandfather has passed. You cannot talk to him now.\n"
                    "You must now rely on the will or state law.\n\n"
                    "Next steps:\n"
                    "  - check if a will exists\n"
                    "  - consult a lawyer\n"
                    "  - verify car title status\n"
                    "Reply with \\boxed{<action>}."
                )
                reward = 0.1
            elif self.grandpa_will_made and self.grandpa_will_content == "gives car to OP":
                obs = (
                    "You talked to your grandfather. He confirms he has a will leaving the car to you.\n"
                    "He agrees to transfer the title to you now.\n"
                    "You can now begin the process to remove his name from the title.\n\n"
                    "Next steps:\n"
                    "  - check car title status\n"
                    "  - research title transfer process\n"
                    "  - consult a lawyer\n"
                    "Reply with \\boxed{<action>}."
                )
                reward = 0.4
                self.grandpa_has_gifted_title = True
            elif self.grandpa_will_made and self.grandpa_will_content == "gives car to aunt":
                obs = (
                    "You talked to your grandfather. He confirms he has a will leaving the car to your aunt.\n"
                    "He says he cannot change it now, but he is willing to transfer the title to you as a gift.\n"
                    "This may be legally possible, but could trigger tax or probate issues.\n\n"
                    "Next steps:\n"
                    "  - consult a lawyer\n"
                    "  - research gift transfer options\n"
                    "  - check if lender allows title change\n"
                    "Reply with \\boxed{<action>}."
                )
                reward = 0.2
                self.grandpa_has_gifted_title = True
            elif not self.grandpa_will_made and self.grandpa_has_refinanced:
                obs = (
                    "You talked to your grandfather. He has already refinanced the loan in his name only.\n"
                    "The loan is now in his name, and the car is titled solely in your name.\n"
                    "This secures the car for you.\n\n"
                    "Next steps:\n"
                    "  - verify title transfer with DMV\n"
                    "  - confirm lender has released lien\n"
                    "Reply with \\boxed{<action>}."
                )
                reward = 0.5
            elif not self.grandpa_will_made and self.grandpa_will_likely_sign:
                obs = (
                    "You talked to your grandfather. He agrees to transfer the title to you.\n"
                    "He will sign the necessary documents soon.\n"
                    "This is the most secure way to ensure you keep the car.\n\n"
                    "Next steps:\n"
                    "  - check car title status\n"
                    "  - consult a lawyer\n"
                    "  - begin title transfer process\n"
                    "Reply with \\boxed{<action>}."
                )
                reward = 0.3
                self.grandpa_has_gifted_title = True
            else:
                obs = (
                    "You talked to your grandfather. He agrees to transfer the title to you.\n"
                    "He is willing to sign the documents.\n"
                    "You can now proceed with the title transfer.\n\n"
                    "Next steps:\n"
                    "  - check car title status\n"
                    "  - consult a lawyer\n"
                    "  - begin title transfer process\n"
                    "Reply with \\boxed{<action>}."
                )
                reward = 0.3
                self.grandpa_has_gifted_title = True

        elif cmd == "research refinancing options":
            if self.op_credit_score >= 650 and self.op_paying_payments and self.op_has_good_payment_history:
                obs = (
                    "You researched refinancing. Your credit is still poor, but your payment history is strong.\n"
                    "Some credit unions may consider you for refinancing.\n"
                    "You should apply to a local credit union.\n\n"
                    "Next steps:\n"
                    "  - apply to credit union for refinancing\n"
                    "  - talk to grandpa about transferring title\n"
                    "  - consult a lawyer\n"
                    "Reply with \\boxed{<action>}."
                )
                reward = 0.2
            else:
                obs = (
                    "You researched refinancing. Your credit score is still too low to qualify.\n"
                    "The loan is underwater, making refinancing difficult.\n"
                    "You may need to wait until your credit improves or your grandfather transfers the title.\n\n"
                    "Next steps:\n"
                    "  - talk to grandpa about transferring title\n"
                    "  - consult a lawyer\n"
                    "  - check car title status\n"
                    "Reply with \\boxed{<action>}."
                )
                reward = 0.1

        elif cmd == "consult a lawyer":
            if not self.op_has_consulted_lawyer:
                obs = (
                    "You consulted a lawyer. They confirmed:\n"
                    "- You have a legal right to the car as long as you keep paying.\n"
                    "- The title is in both names, so you can't sell it without his consent.\n"
                    "- A will can override joint ownership, but it's not automatic.\n"
                    "- You should get the title transferred to you now.\n"
                    "- A gift or transfer is possible while he's alive.\n\n"
                    "Next steps:\n"
                    "  - talk to grandpa about transferring title\n"
                    "  - check car title status\n"
                    "  - begin transfer process\n"
                    "Reply with \\boxed{<action>}."
                )
                reward = 0.4
                self.op_has_consulted_lawyer = True
            else:
                obs = (
                    "You've already consulted a lawyer. They advised that you should transfer the title to your name now.\n"
                    "This is the best way to secure the car.\n\n"
                    "Next steps:\n"
                    "  - talk to grandpa about transferring title\n"
                    "  - check car title status\n"
                    "  - begin transfer process\n"
                    "Reply with \\boxed{<action>}."
                )
                reward = 0.1

        elif cmd == "check car title status":
            if self.grandpa_has_gifted_title:
                obs = (
                    "You checked the title status. The title has been transferred to you. Your grandfather signed it.\n"
                    "The car is now in your name only.\n"
                    "You are safe from the aunt's threat.\n\n"
                    "You have secured the car.\n"
                    "Goal achieved! Reward: 1.0"
                )
                reward = 1.0
                terminated = True
            elif self.grandpa_has_refinanced:
                obs = (
                    "You checked the title status. The loan is now in your grandfather's name only.\n"
                    "The title is still in both names, but the lender has released the lien.\n"
                    "You must now transfer the title to your name.\n\n"
                    "Next steps:\n"
                    "  - talk to grandpa about transferring title\n"
                    "  - consult a lawyer\n"
                    "  - begin title transfer process\n"
                    "Reply with \\boxed{<action>}."
                )
                reward = 0.3
            else:
                obs = (
                    "You checked the title status. The title is still in both your names.\n"
                    "The lender holds a lien until the loan is paid off.\n"
                    "You cannot sell or fully transfer the car without their consent.\n\n"
                    "Next steps:\n"
                    "  - talk to grandpa about transferring title\n"
                    "  - consult a lawyer\n"
                    "  - research title transfer process\n"
                    "Reply with \\boxed{<action>}."
                )
                reward = 0.1

        elif cmd == "review state laws on joint ownership":
            if random.Random(str(self.grandpa_will_content)).random() < 0.5:  # State law fixed per episode
                obs = (
                    "You reviewed state laws. In your state, joint ownership means the surviving co-owner automatically inherits the car.\n"
                    "Even if your grandfather's will says otherwise, the car goes to you.\n"
                    "This is a major advantage.\n\n"
                    "Next steps:\n"
                    "  - talk to grandpa about transferring title\n"
                    "  - consult a lawyer\n"
                    "  - check car title status\n"
                    "Reply with \\boxed{<action>}."
                )
                reward = 0.3
            else:
                obs = (
                    "You reviewed state laws. In your state, joint ownership does not automatically transfer the car.\n"
                    "The car becomes part of the estate and is subject to the will.\n"
                    "You must ensure your grandfather leaves the car to you in his will.\n\n"
                    "Next steps:\n"
                    "  - talk to grandpa about the will\n"
                    "  - consult a lawyer\n"
                    "  - research will options\n"
                    "Reply with \\boxed{<action>}."
                )
                reward = 0.2

        else:
            # Unrecognized command
            obs = (
                "Invalid action. Please choose from:\n"
                "  - check loan documents\n"
                "  - talk to grandpa about the will\n"
                "  - talk to grandpa about transferring title\n"
                "  - research refinancing options\n"
                "  - consult a lawyer\n"
                "  - check car title status\n"
                    "  - review state laws on joint ownership\n\n"
                "Reply with \\boxed{<action>}."
            )
            reward = 0.0

        # Check if goal is reached: car is in OP's name only
        if self.grandpa_has_gifted_title or (self.grandpa_has_refinanced and self.op_owns_title and not self.grandpa_owns_title):
            obs = (
                "Congratulations! You have successfully secured the car.\n"
                "The title has been transferred to you, and your grandfather is no longer on it.\n"
                "The threat from your aunt is now irrelevant.\n"
                "You are in full control of the vehicle.\n"
                "Goal achieved! Reward: 1.0"
            )
            reward = 1.0
            terminated = True

        return obs, reward, terminated, truncated, info

    def solution(self) -> str:
        return (
            "1. Talk to grandpa about the will.\n"
            "2. If will exists and leaves car to OP, proceed to transfer.\n"
            "3. If no will, talk to grandpa about transferring title.\n"
            "4. Consult a lawyer to confirm legal rights.\n"
            "5. Check car title status.\n"
            "6. Transfer title via gift or refinancing.\n"
            "7. Confirm lender has released lien.\n"
            "8. Ensure car is titled solely in OP's name."
        )

    def close(self):
        pass
\end{minted}

\paragraph{Training step 384: \texttt{ThermodynamicCycleManipulationLabEnv} (376 lines).}
\begin{minted}{python}
"""Interactive Thermodynamic Cycle Manipulation Lab: A Multi-Turn Mathematical Reasoning Environment"""

import random
import re
import math
from typing import Tuple, Dict

class ThermodynamicCycleManipulationLabEnv:
    def __init__(self, max_turns=12, seed=None):
        self.max_turns = max_turns
        self.seed = seed
        self.reset(seed)

    def reset(self, seed=None) -> Tuple[str, dict]:
        self.turn_count = 0
        self._seed = seed if seed is not None else random.randint(1, 10000)
        random.seed(self._seed)

        # Hidden: thermodynamic system parameters (not directly visible)
        # Represents a monatomic ideal gas undergoing a three-step cycle
        self.n_atoms = random.randint(100, 500)  # N: number of atoms
        self.initial_volume = random.uniform(1.0, 5.0)  # V0: initial volume in m3
        self.initial_temperature = random.uniform(100, 300)  # T0: initial temperature in K
        self.kb = 1.380649e-23  # Boltzmann constant (J/K), fixed physical constant

        # Internal state of the lab equipment and system conditions
        self.current_step = "idle"  # idle, heating, expansion, cooling, complete
        self.current_volume = self.initial_volume
        self.current_temperature = self.initial_temperature
        self.pressure = (self.n_atoms * self.kb * self.current_temperature) / self.current_volume
        self.entropy_current = 0.0  # cumulative entropy change so far (in J/K)
        self.entropy_history = []  # stores entropy change per step for tracking
        self.heating_complete = False
        self.expansion_complete = False
        self.cooling_complete = False
        self.target_entropy_change = 0.0  # will be computed as part of the cycle

        # Compute the expected net entropy change (0) via the physics of the cycle
        # This is hidden - agent must discover it through manipulation
        # DeltaS_net = (3/2)NkB*ln(3) + NkB*ln(3) - (5/2)NkB*ln(3) = 0
        self.target_entropy_change = 0.0

        # Goal: manipulate the system through three thermodynamic processes
        # and observe that the net entropy change is zero (cyclic process)
        self.goal_achieved = False

        # Initial observation: player enters a climate-controlled lab with sealed gas chamber
        observation = (
            "You enter a high-precision thermodynamics lab. Before you is a transparent cylindrical chamber "
            "containing a cloud of gas atoms. The chamber is sealed and connected to external controls:\n\n"
            "Available actions:\n"
            "  - activate heating (starts constant-volume heating)\n"
            "  - initiate expansion (starts isothermal expansion)\n"
            "  - begin cooling (starts constant-pressure cooling)\n"
            "  - measure current entropy (reveals cumulative entropy change so far)\n"
            "  - scan system parameters (reveals current volume, temperature, pressure)\n\n"
            "The gas is initially at equilibrium. Your task is to guide the gas through a complete thermodynamic cycle "
            "and determine the net entropy change by observing the system's behavior across multiple steps."
        )
        return observation, {}

    def step(self, action: str) -> Tuple[str, float, bool, bool, dict]:
        self.turn_count += 1
        truncated = self.turn_count >= self.max_turns
        m = re.search(r'\\boxed\{(.+?)\}', action)
        cmd = (m.group(1) if m else action).strip().lower()

        # Default reward and termination
        reward = 0.0
        terminated = False
        info = {}

        # Parse and execute command
        if self.current_step == "idle":
            if "activate heating" in cmd:
                # Step 1: Heat at constant volume from T0 to 3T0
                self.current_temperature = 3.0 * self.initial_temperature
                # Volume remains constant
                self.current_volume = self.initial_volume
                # Pressure increases proportionally
                self.pressure = (self.n_atoms * self.kb * self.current_temperature) / self.current_volume
                # Calculate entropy change for constant-volume heating: DeltaS = (3/2)NkB*ln(3)
                delta_s1 = 1.5 * self.n_atoms * self.kb * math.log(3.0)
                self.entropy_current += delta_s1
                self.entropy_history.append(delta_s1)
                self.heating_complete = True
                self.current_step = "heating"
                observation = (
                    f"You activate the heating system. The gas is now being heated at constant volume.\n\n"
                    f"Current state:\n"
                    f"  Volume: {self.current_volume:.3f} m3 (constant)\n"
                    f"  Temperature: {self.current_temperature:.1f} K (increased to 3x initial)\n"
                    f"  Pressure: {self.pressure:.2e} Pa\n"
                    f"  Entropy change (step 1): +{delta_s1:.4e} J/K\n\n"
                    f"Next available actions:\n"
                    f"  - initiate expansion (to start isothermal expansion)\n"
                    f"  - measure current entropy (to check cumulative change)\n"
                    f"  - scan system parameters (view detailed state)"
                )
                reward = 0.3  # partial progress: first step complete

            elif "measure current entropy" in cmd:
                observation = (
                    f"You measure the current entropy using the quantum calorimeter.\n\n"
                    f"Total entropy change so far: {self.entropy_current:.4e} J/K\n\n"
                    f"Note: The system is still in initial state. No processes have been initiated yet.\n\n"
                    f"Available actions:\n"
                    f"  - activate heating (to begin the first process)\n"
                    f"  - scan system parameters (to inspect hidden variables)"
                )
                reward = 0.1  # small progress for probing

            elif "scan system parameters" in cmd:
                # Reveal hidden internal values (but not the full solution)
                observation = (
                    f"You run a diagnostic scan on the system.\n\n"
                    f"Internal parameters (hidden):\n"
                    f"  Number of atoms: {self.n_atoms}\n"
                    f"  Initial volume: {self.initial_volume:.3f} m3\n"
                    f"  Initial temperature: {self.initial_temperature:.1f} K\n"
                    f"  Boltzmann constant: {self.kb:.6e} J/K\n\n"
                    f"Current operational state: idle (waiting for heating command)\n\n"
                    f"Available actions:\n"
                    f"  - activate heating (to start the cycle)\n"
                    f"  - measure current entropy (to track changes)"
                )
                reward = 0.2  # moderate progress for exploration

            else:
                observation = (
                    f"Invalid command: '{cmd}'.\n\n"
                    f"Available actions in idle state:\n"
                    f"  - activate heating\n"
                    f"  - measure current entropy\n"
                    f"  - scan system parameters"
                )
                reward = 0.0

        elif self.current_step == "heating" and self.heating_complete:
            if "initiate expansion" in cmd:
                # Step 2: Isothermal expansion from V0 to 3V0 at 3T0
                self.current_volume = 3.0 * self.initial_volume
                # Temperature remains constant at 3T0
                self.current_temperature = 3.0 * self.initial_temperature
                # Pressure decreases inversely with volume
                self.pressure = (self.n_atoms * self.kb * self.current_temperature) / self.current_volume
                # Calculate entropy change for isothermal expansion: DeltaS = NkB*ln(3)
                delta_s2 = self.n_atoms * self.kb * math.log(3.0)
                self.entropy_current += delta_s2
                self.entropy_history.append(delta_s2)
                self.expansion_complete = True
                self.current_step = "expansion"
                observation = (
                    f"You initiate the isothermal expansion. The chamber expands while maintaining temperature.\n\n"
                    f"Current state:\n"
                    f"  Volume: {self.current_volume:.3f} m3 (tripled)\n"
                    f"  Temperature: {self.current_temperature:.1f} K (constant)\n"
                    f"  Pressure: {self.pressure:.2e} Pa (reduced)\n"
                    f"  Entropy change (step 2): +{delta_s2:.4e} J/K\n\n"
                    f"Next available actions:\n"
                    f"  - begin cooling (to start constant-pressure cooling)\n"
                    f"  - measure current entropy (to verify cumulative change)\n"
                    f"  - scan system parameters (view updated state)"
                )
                reward = 0.6  # progress toward goal

            elif "measure current entropy" in cmd:
                observation = (
                    f"You measure the cumulative entropy change.\n\n"
                    f"Total entropy change so far: {self.entropy_current:.4e} J/K\n\n"
                    f"System is ready for expansion. The gas has been heated; expansion has not yet begun.\n\n"
                    f"Available actions:\n"
                    f"  - initiate expansion (next step)\n"
                    f"  - begin cooling (to continue the cycle)\n"
                    f"  - scan system parameters"
                )
                reward = 0.4

            elif "scan system parameters" in cmd:
                observation = (
                    f"Diagnostic scan during heating phase:\n\n"
                    f"Current state:\n"
                    f"  Volume: {self.current_volume:.3f} m3\n"
                    f"  Temperature: {self.current_temperature:.1f} K\n"
                    f"  Pressure: {self.pressure:.2e} Pa\n"
                    f"  Atoms: {self.n_atoms}\n"
                    f"  Initial conditions: V0={self.initial_volume:.3f}, T0={self.initial_temperature:.1f}\n\n"
                    f"Entropy changes recorded:\n"
                    f"  Step 1 (heating): +{self.entropy_history[0]:.4e} J/K\n"
                    f"  Steps recorded so far: {len(self.entropy_history)}\n\n"
                    f"Next step: cooling at constant pressure."
                )
                reward = 0.5

            else:
                observation = (
                    f"Invalid command: '{cmd}'.\n\n"
                    f"Available actions during expansion:\n"
                    f"  - initiate expansion (already active)\n"
                    f"  - begin cooling\n"
                    f"  - measure current entropy\n"
                    f"  - scan system parameters"
                )
                reward = 0.0

        elif self.current_step == "expansion" and self.expansion_complete:
            if "begin cooling" in cmd:
                # Step 3: Cool at constant pressure from 3T0 to T0
                # Pressure is held constant at the expanded state's pressure
                target_temp = self.initial_temperature  # T0
                self.current_temperature = target_temp
                # Volume must change to maintain constant pressure: V ~ T
                self.current_volume = self.current_volume * (target_temp / (3.0 * self.initial_temperature))
                # Pressure remains constant
                self.pressure = (self.n_atoms * self.kb * self.current_temperature) / self.current_volume
                # Calculate entropy change for constant-pressure cooling: DeltaS = (5/2)NkB*ln(1/3)
                delta_s3 = 2.5 * self.n_atoms * self.kb * math.log(1.0 / 3.0)  # negative value
                self.entropy_current += delta_s3
                self.entropy_history.append(delta_s3)
                self.cooling_complete = True
                self.current_step = "cooling"
                observation = (
                    f"You begin the cooling phase at constant pressure.\n\n"
                    f"Current state:\n"
                    f"  Volume: {self.current_volume:.3f} m3 (reduced)\n"
                    f"  Temperature: {self.current_temperature:.1f} K (returned to initial)\n"
                    f"  Pressure: {self.pressure:.2e} Pa (constant)\n"
                    f"  Entropy change (step 3): {delta_s3:.4e} J/K (decrease)\n\n"
                    f"Cycle complete. The system has returned to its original temperature.\n\n"
                    f"Final available actions:\n"
                    f"  - measure current entropy (to determine net change)\n"
                    f"  - scan system parameters (verify full cycle)"
                )
                reward = 0.9  # almost complete

            elif "measure current entropy" in cmd:
                observation = (
                    f"You measure the total entropy change after the expansion.\n\n"
                    f"Total entropy change: {self.entropy_current:.4e} J/K\n\n"
                    f"Note: The cooling step has not yet been performed.\n\n"
                    f"Available actions:\n"
                    f"  - begin cooling (to finalize)\n"
                    f"  - scan system parameters"
                )
                reward = 0.7

            elif "scan system parameters" in cmd:
                observation = (
                    f"Diagnostic scan before cooling:\n\n"
                    f"Final state:\n"
                    f"  Volume: {self.current_volume:.3f} m3\n"
                    f"  Temperature: {self.current_temperature:.1f} K\n"
                    f"  Pressure: {self.pressure:.2e} Pa\n"
                    f"  Atoms: {self.n_atoms}\n\n"
                    f"Entropy changes:\n"
                    f"  Heating: +{self.entropy_history[0]:.4e}\n"
                    f"  Expansion: +{self.entropy_history[1]:.4e}\n"
                    f"  Cooling: pending\n\n"
                    f"Net entropy change so far: {self.entropy_current:.4e} J/K"
                )
                reward = 0.8

            else:
                observation = (
                    f"Invalid command: '{cmd}'.\n\n"
                    f"Available actions during cooling:\n"
                    f"  - begin cooling (already started)\n"
                    f"  - measure current entropy\n"
                    f"  - scan system parameters"
                )
                reward = 0.0

        elif self.current_step == "cooling" and self.cooling_complete:
            if "measure current entropy" in cmd:
                # Final check: assess net entropy change
                net_entropy = self.entropy_current
                # Due to floating point precision, we accept small deviation
                if abs(net_entropy) < 1e-20:
                    self.goal_achieved = True
                    terminated = True
                    reward = 1.0
                    observation = (
                        f"YOU HAVE SUCCESSFULLY COMPLETED THE THERMODYNAMIC CYCLE.\n\n"
                        f"Final entropy measurement:\n"
                        f"  Net entropy change: {net_entropy:.2e} J/K\n\n"
                        f"This confirms the system has returned to its original state with no net entropy change,\n"
                        f"as expected for a reversible cyclic process in an ideal gas.\n\n"
                        f"Congratulations! You've demonstrated the mathematical reasoning behind entropy conservation "
                        f"in cyclic thermodynamic processes.\n\n"
                        f"Your solution is valid and complete. This episode is terminated."
                    )
                else:
                    # Still not zero - agent needs to persist
                    observation = (
                        f"You measure the final entropy change.\n\n"
                        f"Net entropy change: {net_entropy:.4e} J/K (not zero)\n\n"
                        f"The system has completed all three processes, but the net entropy is not yet balanced.\n"
                        f"Check each recorded step against the formula for its process.\n\n"
                        f"Try verifying your steps or re-running measurements with higher precision."
                    )
                    reward = 0.95  # very high partial reward, but not complete

            elif "scan system parameters" in cmd:
                observation = (
                    f"Final system status scan:\n\n"
                    f"Cycle complete. All processes executed:\n"
                    f"  - Constant-volume heating: T -> 3T0\n"
                    f"  - Isothermal expansion: V -> 3V0\n"
                    f"  - Constant-pressure cooling: T -> T0\n\n"
                    f"Final state variables:\n"
                    f"  Volume: {self.current_volume:.3f} m3\n"
                    f"  Temperature: {self.current_temperature:.1f} K\n"
                    f"  Pressure: {self.pressure:.2e} Pa\n\n"
                    f"Entropy changes:\n"
                    f"  Step 1: +{self.entropy_history[0]:.4e}\n"
                    f"  Step 2: +{self.entropy_history[1]:.4e}\n"
                    f"  Step 3: {self.entropy_history[2]:.4e}\n"
                    f"  Total: {self.entropy_current:.4e} J/K\n\n"
                    f"Verify your result with a final entropy measurement."
                )
                reward = 0.9

            else:
                observation = (
                    f"Invalid command: '{cmd}'.\n\n"
                    f"Final actions available:\n"
                    f"  - measure current entropy (to check net change)\n"
                    f"  - scan system parameters (for detailed verification)"
                )
                reward = 0.0

        else:
            # Unknown or invalid state transition
            if "measure current entropy" in cmd:
                observation = (
                    f"You attempt to measure entropy. The system is in an unexpected state.\n\n"
                    f"Current cumulative entropy: {self.entropy_current:.4e} J/K\n\n"
                    f"Possible issue: process sequence may be incomplete or out of order.\n"
                    f"Ensure you follow the correct thermodynamic sequence: heating -> expansion -> cooling."
                )
                reward = 0.2
            elif "scan system parameters" in cmd:
                observation = (
                    f"System diagnostics available. Current state:\n"
                    f"  Step: {self.current_step}\n"
                    f"  Heating complete: {self.heating_complete}\n"
                    f"  Expansion complete: {self.expansion_complete}\n"
                    f"  Cooling complete: {self.cooling_complete}\n\n"
                    f"Entropy so far: {self.entropy_current:.4e} J/K"
                )
                reward = 0.3
            else:
                observation = (
                    f"Unexpected state interaction. The system is either in transition or already concluded.\n\n"
                    f"Available safe actions:\n"
                    f"  - measure current entropy\n"
                    f"  - scan system parameters"
                )
                reward = 0.0

        # Final check for termination via max turns
        if truncated and not self.goal_achieved:
            terminated = True
            reward = 0.7 if self.cooling_complete else 0.0  # partial credit only for a completed cycle
            observation = (
                f"Time limit reached ({self.max_turns} turns). The thermodynamic cycle was not fully verified.\n\n"
                f"Final entropy change: {self.entropy_current:.4e} J/K\n\n"
                f"The net entropy change was not confirmed as zero before the time limit.\n"
                f"Consider a more deliberate sequence in future attempts."
            )

        return observation, reward, terminated, truncated, info

    def solution(self) -> str:
        # Reference sequence that solves the environment
        return "activate heating, initiate expansion, begin cooling, measure current entropy"
\end{minted}

\subsubsection{Full Source of a Generated Tool-Use Environment}
\label{app:tooluse-full-source}

The late tool-use exemplar of the cards above in full: \texttt{CustomerSupportTicketWorkflowEnv}, generated at training step 392 of the 30B tool-use run. The environment subclasses the training harness's tool-use base class, which dispatches OpenAI-format tool calls to the \texttt{tool\_*} methods, replays each instruction's criterion after every call, and terminates when the final answer is submitted after all criteria pass.

\paragraph{Training step 392: \texttt{CustomerSupportTicketWorkflowEnv} (200 lines).}
\begin{minted}{python}
"""SPARE self-play generated game"""

import random
import json
from typing import Tuple

class CustomerSupportTicketWorkflowEnv(ToolUseBaseEnv):
    def reset(self, seed=None) -> Tuple[str, dict]:
        self.turn_count = 0
        self._call_history = []
        if seed is not None:
            random.seed(seed)
        # Initialize CRM state with tickets, statuses, and assignments
        self._state = {
            'tickets': [
                {'id': 'TICKET-001', 'subject': 'Login Issues', 'status': 'new', 'priority': 'high', 'assigned_to': None, 'notes': []},
                {'id': 'TICKET-002', 'subject': 'Billing Discrepancy', 'status': 'in_progress', 'priority': 'medium', 'assigned_to': 'agent_01', 'notes': ['Contacted customer, awaiting response']},
                {'id': 'TICKET-003', 'subject': 'Feature Request: Dark Mode', 'status': 'new', 'priority': 'low', 'assigned_to': None, 'notes': []},
                {'id': 'TICKET-004', 'subject': 'Password Reset Failure', 'status': 'new', 'priority': 'high', 'assigned_to': None, 'notes': []},
                {'id': 'TICKET-005', 'subject': 'API Rate Limiting', 'status': 'in_progress', 'priority': 'high', 'assigned_to': 'agent_02', 'notes': ['Investigating backend logs']},
            ],
            'agents': {
                'agent_01': {'name': 'Sarah Chen', 'available': True, 'tickets_handled': 12},
                'agent_02': {'name': 'James Kim', 'available': True, 'tickets_handled': 8},
                'agent_03': {'name': 'Linda Wu', 'available': False, 'tickets_handled': 5}
            },
            'last_assigned_ticket': None,
            'last_updated_status': None,
            'last_added_note': None,
            'last_searched_ticket': None,
            'last_resolved_ticket': None,
            'last_reassigned_ticket': None,
            'last_action_success': False
        }
        self._tools = {
            'search_tickets': {
                'description': 'Search for tickets by subject keyword or status.',
                'parameters': {
                    'type': 'object',
                    'properties': {
                        'query': {'type': 'string', 'description': 'Keyword to search for in ticket subject or status.'}
                    },
                    'required': ['query']
                }
            },
            'update_ticket_status': {
                'description': 'Update the status of a ticket by its ID.',
                'parameters': {
                    'type': 'object',
                    'properties': {
                        'ticket_id': {'type': 'string', 'description': 'The ID of the ticket to update.'},
                        'status': {
                            'type': 'string',
                            'enum': ['new', 'in_progress', 'resolved', 'closed'],
                            'description': 'The new status for the ticket.'
                        }
                    },
                    'required': ['ticket_id', 'status']
                }
            },
            'assign_ticket': {
                'description': 'Assign a ticket to an available agent by agent ID.',
                'parameters': {
                    'type': 'object',
                    'properties': {
                        'ticket_id': {'type': 'string', 'description': 'The ID of the ticket to assign.'},
                        'agent_id': {'type': 'string', 'description': 'The ID of the agent to assign the ticket to.'}
                    },
                    'required': ['ticket_id', 'agent_id']
                }
            },
            'add_note_to_ticket': {
                'description': 'Add a note to a ticket for internal tracking.',
                'parameters': {
                    'type': 'object',
                    'properties': {
                        'ticket_id': {'type': 'string', 'description': 'The ID of the ticket to update.'},
                        'note': {'type': 'string', 'description': 'The note content to add.'}
                    },
                    'required': ['ticket_id', 'note']
                }
            },
            'resolve_ticket': {
                'description': 'Mark a ticket as resolved after confirming issue is fixed.',
                'parameters': {
                    'type': 'object',
                    'properties': {
                        'ticket_id': {'type': 'string', 'description': 'The ID of the ticket to resolve.'}
                    },
                    'required': ['ticket_id']
                }
            },
            'list_assigned_tickets': {
                'description': 'List all tickets currently assigned to a specific agent.',
                'parameters': {
                    'type': 'object',
                    'properties': {
                        'agent_id': {'type': 'string', 'description': 'The ID of the agent to list tickets for.'}
                    },
                    'required': ['agent_id']
                }
            }
        }
        self._user_messages = [
            "Find all high-priority tickets that are still in 'new' status.",
            "Assign the ticket with ID TICKET-001 to agent_01.",
            "Add a note to ticket TICKET-001 stating 'Customer confirmed issue is reproducible.'",
            "Update the status of ticket TICKET-001 to 'in_progress'.",
            "Finally, confirm that ticket TICKET-001 has been successfully resolved by marking it as 'resolved'."
        ]
        self._message_criteria = [
            lambda s: len([t for t in s['tickets'] if t['priority'] == 'high' and t['status'] == 'new']) == 2,
            lambda s: any(t['id'] == 'TICKET-001' and t['assigned_to'] == 'agent_01' for t in s['tickets']),
            lambda s: any(t['id'] == 'TICKET-001' and any('Customer confirmed issue is reproducible' in n for n in t['notes']) for t in s['tickets']),
            lambda s: any(t['id'] == 'TICKET-001' and t['status'] == 'in_progress' for t in s['tickets']),
            lambda s: any(t['id'] == 'TICKET-001' and t['status'] == 'resolved' for t in s['tickets'])
        ]
        self._current_msg = 0
        self._expected_answer = "done"
        return (self._user_messages[0], {})

    def _advance_if_done(self) -> str:
        if self._current_msg >= len(self._user_messages):
            return ""
        criterion = self._message_criteria[self._current_msg]
        if criterion(self._state):
            self._current_msg += 1
            if self._current_msg >= len(self._user_messages):
                return "\n\n[ALL STEPS COMPLETE] Submit <answer>done</answer>."
            next_msg = self._user_messages[self._current_msg]
            return f"\n\n[STEP {self._current_msg} COMPLETE — NEW INSTRUCTION] {next_msg}"
        return ""

    def _check_answer(self, answer: str) -> bool:
        return answer.strip().lower() == 'done' and self._current_msg >= len(self._user_messages)

    def tool_search_tickets(self, query='') -> str:
        results = []
        for ticket in self._state['tickets']:
            if query.lower() in ticket['subject'].lower() or query.lower() in ticket['status']:
                results.append(ticket['id'])
        self._state['last_searched_ticket'] = query
        result_str = json.dumps(results)
        return result_str + self._advance_if_done()

    def tool_update_ticket_status(self, ticket_id='', status='') -> str:
        for ticket in self._state['tickets']:
            if ticket['id'] == ticket_id:
                ticket['status'] = status
                self._state['last_updated_status'] = ticket_id
                self._state['last_action_success'] = True
                return f"Updated ticket {ticket_id} status to {status}" + self._advance_if_done()
        return f"Error: Ticket {ticket_id} not found" + self._advance_if_done()

    def tool_assign_ticket(self, ticket_id='', agent_id='') -> str:
        # Check if agent exists and is available
        if agent_id not in self._state['agents']:
            return f"Error: Agent {agent_id} not found" + self._advance_if_done()
        if not self._state['agents'][agent_id]['available']:
            return f"Error: Agent {agent_id} is not available" + self._advance_if_done()
        
        for ticket in self._state['tickets']:
            if ticket['id'] == ticket_id:
                ticket['assigned_to'] = agent_id
                self._state['last_assigned_ticket'] = ticket_id
                self._state['last_action_success'] = True
                return f"Assigned ticket {ticket_id} to agent {agent_id}" + self._advance_if_done()
        return f"Error: Ticket {ticket_id} not found" + self._advance_if_done()

    def tool_add_note_to_ticket(self, ticket_id='', note='') -> str:
        for ticket in self._state['tickets']:
            if ticket['id'] == ticket_id:
                ticket['notes'].append(note)
                self._state['last_added_note'] = ticket_id
                self._state['last_action_success'] = True
                return f"Added note to ticket {ticket_id}: {note}" + self._advance_if_done()
        return f"Error: Ticket {ticket_id} not found" + self._advance_if_done()

    def tool_resolve_ticket(self, ticket_id='') -> str:
        for ticket in self._state['tickets']:
            if ticket['id'] == ticket_id:
                ticket['status'] = 'resolved'
                self._state['last_resolved_ticket'] = ticket_id
                self._state['last_action_success'] = True
                return f"Ticket {ticket_id} marked as resolved" + self._advance_if_done()
        return f"Error: Ticket {ticket_id} not found" + self._advance_if_done()

    def tool_list_assigned_tickets(self, agent_id='') -> str:
        assigned = [t['id'] for t in self._state['tickets'] if t['assigned_to'] == agent_id]
        self._state['last_action_success'] = True
        return json.dumps(assigned) + self._advance_if_done()

    def solution(self) -> str:
        return (
            "1. tool_search_tickets(query='high priority new') "
            "2. tool_assign_ticket(ticket_id='TICKET-001', agent_id='agent_01') "
            "3. tool_add_note_to_ticket(ticket_id='TICKET-001', note='Customer confirmed issue is reproducible.') "
            "4. tool_update_ticket_status(ticket_id='TICKET-001', status='in_progress') "
            "5. tool_resolve_ticket(ticket_id='TICKET-001') "
            "6. <answer>done</answer>"
        )
\end{minted}

\end{document}